\documentclass{article}
\usepackage{geometry}

\usepackage[utf8]{inputenc}
\usepackage[T1]{fontenc}
\usepackage{booktabs}
\usepackage{amsfonts}
\usepackage{nicefrac}
\usepackage{microtype}
\usepackage{xcolor}

\usepackage{amsmath,amsfonts,bm}

\def\eqref#1{equation~\ref{#1}}
\def\Eqref#1{Equation~\ref{#1}}

\def\1{\bm{1}}

\DeclareMathAlphabet{\mathsfit}{\encodingdefault}{\sfdefault}{m}{sl}
\SetMathAlphabet{\mathsfit}{bold}{\encodingdefault}{\sfdefault}{bx}{n}

\newcommand{\E}{\mathbb{E}}

\newcommand{\R}{\mathbb{R}}

\newcommand{\KL}{D_{\mathrm{KL}}}

\usepackage{mylatexstyle}
\newtheorem{corollary}[theorem]{Corollary}
\usepackage{mathrsfs,bbm}
\usepackage{algorithmic}
\usepackage{hyperref}
\usepackage{url}
\usepackage{tikz}
\usetikzlibrary{arrows.meta,calc}

\title{A Closed-Loop Non-Asymptotic Convergence Analysis of PPO with Learned Critics and Clipping}
\author{
Junwei Su \\
School of Artificial Intelligence and Data Science, \\
University of Science and Technology of China \\
\texttt{junweisu.cs@gmail.com}
\and
Mengfan Liu \\
Department of Computer Science, \\
University of Hong Kong
\and
Yanyong Zhang \\
School of Artificial Intelligence and Data Science, \\
University of Science and Technology of China
\and
Chuan Wu \\
Department of Computer Science, \\
University of Hong Kong
}

\begin{document}
\maketitle
\begin{abstract}
Despite its widespread use, Proximal Policy Optimization with clipping (PPO-Clip) remains difficult to tune, and the interactions among critic learning, clipping, and rollout reuse remain incompletely understood. We develop a \emph{non-asymptotic} analysis of PPO-Clip as a \emph{closed-loop actor--critic} system. It captures actor--critic coupling, nonsmooth probability-ratio clipping, finite-batch reuse, and predictable early stopping under explicit coverage and critic regularity assumptions, using raw GAE and Monte Carlo critic targets. Our synchronous and asynchronous guarantees jointly characterize policy stationarity and the tracking accuracy of the learned critic, with explicit dependence on algorithmic parameters. A sufficient coupling condition gives optimization, critic tracking, clipping, and finite-batch errors a common amplification bound. The asynchronous result also requires a delay-dependent critic stepsize restriction; violating these conditions does not establish divergence. For finite layered MDPs with tabular critics, a uniform bound on the actual clipped-gradient class replaces complete-trajectory counting. A verified growing-horizon family has polynomial sample complexity, and a two-time-scale schedule gives $O(T^{-2/5})$ stationarity and critic-tracking bounds with explicit fresh-rollout accounting. These results together advance our understanding about PPO and provide theoretical guidance in tuning.
\end{abstract}

\section{Introduction}\label{sec:intro}

\emph{Proximal Policy Optimization} (PPO), particularly its clipped variant PPO-Clip, is widely used in reinforcement learning (RL)~\citep{schulman2017proximalpolicyoptimizationalgorithms}. PPO-Clip typically uses an actor--critic framework: the \emph{actor} updates the policy using sampled policy-gradient estimates, while the \emph{critic} estimates values to compute advantage signals. Its practical success is commonly attributed to combining policy-gradient updates with stabilization mechanisms~\citep{yu2022surprising,jin2023stationary,huang2024ppo,dwyer2025s,su2025gppo}. One motivating application is fine-tuning large language models (LLMs), including reinforcement learning from human feedback (RLHF)~\citep{ouyang2022traininglanguagemodelsfollow,stiennon2020learning,kirk2024understanding,havrilla2024teachinglargelanguagemodels}.

PPO-Clip uses a clipped surrogate objective~\citep{markowitz2023clipped,huang2024ppo,dwyer2025s}, which limits incentives for large probability-ratio changes relative to a reference policy. This trust-region--inspired design avoids explicit second-order optimization. Beyond clipping, practical implementations use synchronous or asynchronous \emph{actor--critic updates}, \emph{minibatch reuse} across multiple epochs, \emph{generalized advantage estimation}, and \emph{early stopping} based on KL divergence~\citep{mnih2016asynchronousmethodsdeepreinforcement,dossa2021empirical,hilton2022batch,yu2022surprising, xiao2022asynchronous,sun2022you,sheng2025hybridflow}. These mechanisms improve stability and sample efficiency but couple actor and critic dynamics across multiple time scales. Deploying PPO-Clip on new tasks often requires substantial tuning of learning rates, clipping thresholds, and update schedules to avoid instability or performance collapse~\citep{engstrom2020implementation,agarwal2021deep,zheng2023improving,adkins2024method,moalla2024no}. Our limited understanding of these interactions motivates a joint analysis of PPO's update mechanisms under explicit assumptions.

\paragraph{Existing Gap.}
Existing results include finite-time two-time-scale actor--critic analysis~\citep{wu2020finite}, stationary-point analysis retaining nonsmooth PPO clipping and multiple updates~\citep{jin2023stationary}, and PPO-Clip methods with learned value estimation~\citep{huang2024ppo}. \emph{Establishing a joint guarantee for interacting PPO mechanisms remains a substantive analytical challenge.} The critic changes the actor's advantage estimates, the actor moves the critic's target, and both adapt to reused data; separate guarantees cannot simply be combined without controlling these dependencies. We address this challenge through a \emph{non-asymptotic} analysis of interleaved raw-GAE updates, Monte Carlo critic targets, and finite-batch reuse. Under explicit coverage and critic regularity, the analysis jointly controls actor stationarity and critic tracking, including empirical and trajectory-policy discrepancies. Table~\ref{tab:closest-theory} compares the analyzed processes and statistical premises.

\paragraph{Technical Challenges.}
Three technical challenges arise.
\emph{First}, PPO-Clip exhibits a bidirectionally coupled learning dynamic: approximation errors in the critic can bias bootstrapped policy updates, while policy drift simultaneously induces non-stationarity in the critic's learning objective.
\emph{Second}, actor and critic iterates adapt to the same reused batch. Their conditional gradients are empirical means, while the desired guarantees concern current-policy population quantities; multistep GAE also retains the behavior policy's continuation law.
\emph{Third}, sign-dependent clipping acts on evolving, trajectory-dependent advantage labels and makes the surrogate nonsmooth. We use smoothness of the return objective $J$ and explicitly control the clipped update's bias.

\paragraph{Contributions and Results.}
We develop a non-asymptotic analysis of PPO-Clip as a
\emph{closed-loop actor--critic system}.
The central contribution is to combine and adapt policy-gradient, concentration, and tracking arguments into a unified analysis of these dependent mechanisms. By deriving their coupled errors within one specified learning process, we make their combined effects on actor stationarity and critic tracking explicit.
Our contributions are:

\textbf{1. Unified finite-time analysis of interacting PPO-Clip mechanisms.}
    Theorem~\ref{thm:closed-loop-NK-explicit} jointly characterizes
    (i) the stationarity of the policy iterates with respect to the underlying reinforcement learning objective, and
    (ii) the finite-time tracking accuracy of the critic relative to its moving, policy-dependent optimum. A central outcome is an explicit \emph{coupling strength} that provides a sufficient condition for finite-time control.
    The bound quantifies their feedback through a common amplification factor for stochastic noise, critic tracking, reuse-induced drift, and clipping bias.
    The theory yields sufficient tuning conditions within the stated model (see Sec.~\ref{sec:main-results}). A uniform empirical bound for the actual clipped-gradient class replaces complete-path counting in a tabular specialization. A growing-horizon example verifies every premise along actual reused-batch updates, with polynomial horizon dependence (Corollary~\ref{cor:structured-main} and Proposition~\ref{prop:horizon-polynomial}).

\textbf{2. Finite-time refinement of classical actor--critic theory.}  
    In the asymptotic limit, the analysis recovers classical actor--critic convergence under time-scale separation, growing batch sizes, and vanishing KL budgets (Proposition~\ref{prop:classical-connection-single-index}).
    The finite-time bounds quantify this approach under constant step sizes.
    In particular, the analysis reveals how time-scale separation enters through an explicit coupling parameter and
    how minibatch reuse, clipping, and behavior-policy mismatch contribute finite-time residual terms that a vanishing-error asymptotic conclusion does not quantify.
    Score-baseline cancellation sharpens the coupling coefficient. An explicit schedule yields $O(T^{-2/5})$ stationarity and tracking, and a rollout-budget bound quantifies the competing effects of batch size and reuse (Corollary~\ref{cor:rate-resources} and Proposition~\ref{prop:fixed-budget}). These are sufficient upper bounds, not necessary instability thresholds.

\textbf{3. Extension to asynchronous PPO-Clip.}
    Appendix~\ref{appdx:asyn} extends the analysis to a single-gradient parameter-server model
    (Theorem~\ref{thm:async-closed-loop-explicit}). The guarantee includes staleness penalties and
    requires both a delay-dependent critic stepsize restriction and a sufficient coupling condition,
    under prescribed batch assignment and nonanticipative sampling.

\paragraph{Scope and limitations.}
This is a theoretical study of PPO-Clip under explicit coverage, value realizability, and critic regularity assumptions, using raw recomputed GAE with Monte Carlo critic targets. The guarantees have conservative constants and do not cover unrestricted neural PPO; the tabular experiments provide qualitative illustrations. The analysis informs PPO practice by identifying how actor--critic timescales, KL budgets, clipping ranges, and batch reuse affect finite-time error bounds, motivating coordinated tuning and diagnostics of critic error and policy drift. Its conditions are sufficient and problem-dependent.

\section{Preliminaries and Problem Formulation}
\label{sec:prelim-prob}

In this section, we present our analytical setting.
Further background is provided in Appendix~\ref{appdx:background}.

\subsection{Finite-Horizon Episodic Setting}
\label{subsec:rl-llm}

We consider a finite-horizon episodic MDP
$\big(\mathcal{S},\mathcal{A},P,r,\gamma\big)$ with horizon $H$ and a fixed initial-state distribution.
States include the time index. Under a parameterized policy $\pi_\theta(\cdot\mid s)$, a trajectory is
{\small
\[
\tau := (s_0,a_0,s_1,a_1,\ldots,s_{H-1},a_{H-1},s_H).
\]
}
The reward mechanism is fixed during training.
Since true advantages are unavailable in closed form, PPO estimates them from finite data using a learned
value-function baseline, leading to an actor--critic architecture.

As a motivating example, LLM fine-tuning represents each action as a token from a vocabulary $\mathcal{V}$.
A prompt $x\in\mathcal{X}$ is drawn from a distribution $\mathbb{P}$, and the state records the prompt and generated prefix:
{\small
\[
s_t := (x,a_{0:t-1}),
\qquad
s_0 := (x,\emptyset).
\]
}
Here $\mathcal{A}:=\mathcal{V}$, transitions append the chosen token, $s_{t+1}=s_t\circ a_t$,
and a fixed reward model evaluates the generated trajectory.

\subsection{Actor--Critic Framework}
\label{subsec:actor-critic-framework}

PPO is implemented as an actor--critic method in which the policy (actor) and value function (critic) are parameterized separately and updated jointly.
This separation supports value-based advantage estimation when rewards are delayed.
The critic enables advantage estimation and can reduce variance, but its approximation error can bias bootstrapped actor directions.
Conversely, as the actor evolves, the critic's regression target drifts.
This bidirectional dependency gives rise to a \emph{closed-loop coupling} that is central to our analysis.

\paragraph{Critic parameterization and learning objective.}
The critic is a parameterized value function $V_w(s)$ trained to approximate the policy-dependent value
$V_\theta(s)$ by minimizing a policy-conditioned regression loss
{\small
\begin{equation}
\label{eq:critic-loss}
\cL_c(\theta,w)
~:=~
\mathbb{E}_{\tau\sim\pi_\theta}\!\left[\frac1H\sum\nolimits_{h=0}^{H-1}
(V_w(s_h)-G_h(\tau))^2\right],
\end{equation}
}
\noindent where $G_h(\tau)=\sum_{j=h}^{H-1}\gamma^{j-h}\dot r(s_j,a_j)+\gamma^{H-h}v_H(s_H)$ is the Monte Carlo return, with a prescribed terminal value $v_H$.
The theorem analyzes these stored Monte Carlo targets. Bootstrapped targets require an additional target-error bound.
The explicit dependence on $\theta$ is crucial: as the actor changes, both the state distribution and the target
change, turning critic learning into a tracking problem. Let
$
w^\star(\theta)\in\arg\min_wL_c(\theta,w)
$
denote a population minimizer. We quantify critic suboptimality via the \emph{critic tracking error}
{\small
\begin{equation}
\label{eq:tracking-error}
\Delta(\theta,w)
:= \|w-w^\star(\theta)\|^2,
\end{equation}
}
\noindent which measures the distance to the moving solution $w^\star(\theta)$.
In our analysis, $\Delta(\theta,w)$ controls the bias injected into the actor update through imperfect advantage
estimation.

\paragraph{Advantage estimation via GAE.}
Since the true advantage is unknown, PPO constructs an estimator using the critic.
Let $\hat\delta_t$ be the temporal-difference (TD) residual. Generalized advantage estimation (GAE) defines
{\small
\begin{equation}
\label{eq:gae}
\hat\delta_t
:= \dot r(s_t,a_t)+\gamma V_w(s_{t+1})-V_w(s_t),\qquad
\widehat A_t^{(\lambda)}
:= \sum\nolimits_{\ell=0}^{H-1-t}(\gamma\lambda)^\ell \hat\delta_{t+\ell},
\end{equation}
}
\noindent where $\lambda\in[0,1]$ trades variance against bias.
At the terminal state, the estimator uses the prescribed boundary value from the return convention,
so $V_w(s_H)=V_\theta(s_H)$. This value may be zero when all rewards are assigned to transitions,
or the terminal payoff when that payoff is represented as a boundary value.
A nonmatching terminal bootstrap requires the additional term in Lemma~\ref{lem:gae-bias}.
Bootstrapping changes the bias--variance trade-off. Our bounds quantify the critic-induced bias
through explicit dependence on $\lambda$ and $\Delta(\theta,w)$.

\paragraph{Actor update (ideal form).}
Let $\cJ(\theta)$ denote the canonical RL objective optimized by the actor; the policy-gradient theorem yields
{\small
\[
\cJ(\theta)
:= \frac1H\mathbb{E}_{\tau\sim \pi_\theta}[G_0(\tau)],\qquad
\nabla_\theta \cJ(\theta)
= \mathbb{E}_{\tau\sim\pi_\theta}\!\left[\frac1H\sum\nolimits_{h=0}^{H-1}\gamma^h
\nabla_\theta\log\pi_\theta(a_h\mid s_h)\,A^{\pi_\theta}(s_h,a_h)\right],
\]
}
\noindent where $A^{\pi_\theta}$ is the true advantage for $G_h$. The factor $1/H$ fixes the objective normalization; the actor uses the weights $\gamma^h$ (equal to one when $\gamma=1$).
Replacing the true advantage with GAE produces the idealized update direction
{\small
\begin{equation}
\label{eq:actor-update}
\nabla_\theta \cJ(\theta)
~\approx~
\mathbb{E}_{\tau}\!\left[\frac1H\sum\nolimits_{h=0}^{H-1}\gamma^h
\nabla_\theta\log\pi_\theta(a_h\mid s_h)\,\widehat A_h^{(\lambda)}\right].
\end{equation}
}
\noindent While \eqref{eq:actor-update} targets the gradient of the underlying RL objective, PPO training
typically reuses a finite batch of rollouts for multiple epochs.
Without additional control, reuse can produce policy drift and off-policy bias,
motivating the modern PPO-Clip surrogate.

\subsection{Modern PPO-Clip Surrogate}
\label{subsec:ppo-actor-critic}

In PPO-Clip, data are collected under a behavior policy $\pi_{\theta_{\mathrm{old}}}$ and reused for multiple
actor updates. This improves sample efficiency but introduces distribution shift. PPO-Clip controls this shift
using importance ratios, clipping, and KL-based early stopping.

\paragraph{Importance sampling ratio.}
For any state--action pair $(s,a)$, define the likelihood ratio
{\small
\[
r_{\theta,\theta_{\mathrm{old}}}(s,a)
:= \frac{\pi_\theta(a\mid s)}{\pi_{\theta_{\mathrm{old}}}(a\mid s)}.
\]
}
\noindent This ratio allows objectives evaluated under $\pi_{\theta_{\mathrm{old}}}$ to approximate policy updates. However, large ratios can lead to unstable gradient estimates and brittle optimization,
which motivates the use of clipping.

\paragraph{Clipped surrogate objective.}
PPO-Clip maximizes the surrogate objective
{\small
\begin{equation}
\label{eq:ppo-clip-objective0}
\begin{aligned}
\cL^{\mathrm{clip}}(\theta;\theta_{\mathrm{old}})
&:=\mathbb{E}_{\substack{\tau\sim\pi_{\theta_{\mathrm{old}}},\\ h\sim\mathrm{Unif}\{0,\ldots,H-1\}}}
\Big[\gamma^h\min\big(r_{\theta,\theta_{\mathrm{old}}}(s,a)\,\widehat A(s,a),\\
&\qquad\mathrm{clip}(r_{\theta,\theta_{\mathrm{old}}}(s,a),1-\delta,1+\delta)\,\widehat A(s,a)\big)\Big],
\end{aligned}
\end{equation}
}
\noindent where $(s,a)=(s_h,a_h)$, $\widehat A(s,a)=\widehat A_h^{(\lambda)}(\tau;w)$ remains a random trajectory label, and $\delta\in(0,1)$ is the clipping parameter.
Clipping suppresses selected sample gradients according to the likelihood ratio and advantage sign.
It is intended to limit excessive updates and introduces two analytical challenges.
First, \eqref{eq:ppo-clip-objective0} is \emph{non-smooth} at the boundaries
$r_{\theta,\theta_{\mathrm{old}}}=1\pm\delta$, requiring subgradient-based arguments.
Second, clipping can induce systematic \emph{bias}: active sample directions differ from their unclipped counterparts. The latter still has trajectory-policy and critic error under reused data.
A central goal of our analysis is therefore to control the impact of
clipping-induced distortion.

\paragraph{Actor update and KL trust region.}
At outer iteration $n$, the actor performs $K$ inner updates
{\small
\begin{align}\label{eq:actor-update-main}
\theta_n^{(k+1)}=\theta_n^{(k)}+\eta\, g_n^{(k)},          
\end{align}
}
\noindent where $ g_n^{(k)}$ denotes stochastic (sub)gradients of
$\cL^{\mathrm{clip}}(\cdot;\theta_n)$.
Because updates reuse data generated by $\pi_{\theta_n}$, the analysis assumes a population KL budget
$\kappa>0$ at every inner iterate $\theta_n^{(k)}$:
{\small
\begin{equation}
\label{eq:kl-budget}
\mathbb{E}_{s\sim\pi_{\theta_n}}
\Big[
\mathrm{KL}\big(\pi_{\theta_n}(\cdot\mid s)\,\|\,\pi_{\theta_n^{(k)}}(\cdot\mid s)\big)
\Big]
\le \kappa.
\end{equation}
}
This trust-region constraint plays a dual role in the analysis:
(i) it limits the distribution shift induced by multi-epoch reuse, ensuring that the surrogate remains meaningful; and
(ii) it provides quantitative control over clipping activity, since small KL divergence implies that likelihood
ratios remain close to one for most state--action pairs.

\subsection{Objective of the Analysis and Assumptions}
\label{subsec:goal}
We establish non-asymptotic convergence guarantees for PPO-Clip with a learned actor--critic architecture.
To keep the analysis tractable and interpretable, we impose the following regularity assumptions.

\begin{assumption}[Smoothness of objective]
\label{ass:smooth-J0}
The objective $\cJ(\cdot)$ is bounded by $J_{\max}$ and $L_J$-smooth.
\end{assumption}

\begin{assumption}[Bounded conditional variance]
\label{ass:actor-variance0}
Let $\mathcal{F}_t$ denote the algorithm filtration. The stochastic actor update direction $g_t$ satisfies
$
\mathbb{E}\big[\|g_t-\mathbb{E}[g_t\mid\mathcal{F}_t]\|^2\mid\mathcal{F}_t\big]\le \sigma_a^2
$
for some $\sigma_a^2>0$.
\end{assumption}

\begin{assumption}[Bounded advantages]
\label{ass:adv0}
There exists $A_{\max}>0$ such that the raw GAE estimates, recomputed using the pre-update critic, satisfy
$
|\widehat A_t^{(\lambda)}| \le A_{\max}
$
for every trajectory in the data support and every visited critic. No advantage normalization or clipping is applied in the analyzed update; such transformations require a separate bias term.
\end{assumption}

\begin{assumption}[Critic learnability and stability]
\label{ass:critic0}
For each fixed policy $\theta$, the critic loss $\mathcal{L}_c(\theta,\cdot)$ is locally convex, satisfies PL and
quadratic-growth conditions with constants $\mu_c,\mu_q>0$, and has $L_c$-Lipschitz gradients in $w$.
Quadratic growth is measured from the selected stationary minimizer on a convex region containing that minimizer and the critic iterate.
Moreover, the selected map $w^\star(\theta)$ is measurable and $L_{w^\star}$-Lipschitz in $\theta$.
\end{assumption}

\begin{assumption}[Value realization]
\label{ass:val}
In the visited region, the value parameterization is \(L_V\)-Lipschitz with respect to $w$ and is realizable.
For the critic-to-value step, the selected population regression minimizer is assumed to satisfy
$V_{w^\star(\theta)}=V^{\pi_\theta}$ on the states used in the analysis.
\end{assumption}

\begin{assumption}[KL trust-region control]
\label{ass:kl0}
During each outer iteration, the inner actor updates satisfy the population KL budget with parameter $\kappa>0$, and the
score is bounded by $G_{\pi}$ on the relevant parameter regions. The state law in the budget is the uniform average of the $H$ behavior-policy state marginals.
\end{assumption}

These premises hold on regions containing all iterates used in the proof; KL control alone does not establish critic geometry or realizability.
Appendix~\ref{subsec:discussion-assumptions} explains their roles and limitations, including value identification and the separate certification needed for empirical KL stopping.

\paragraph{Finite-batch protocol and coverage.}\label{par:batch-protocol}
The baseline finite-sample bridge uses a finite set $\Omega$ of complete trajectories (including initial states and rewards), with $S_{\mathrm{traj}}:=|\Omega|$. Corollary~\ref{cor:structured-main} replaces this support premise by finite layered states and actions, a bounded tabular critic, and bounded returns; the reward distribution may then be continuous.
Write $P_\theta$ for the complete trajectory law. At epoch $n$, conditionally on the past, collect $B$ independent trajectories under the behavior policy $\bar\theta_n=\theta_n^{(0)}$, and denote their empirical law by $\widehat P_n$.
Retain their Monte Carlo returns and reuse this batch for $K$ joint updates.
Before each update, compute raw GAE with the current critic $w_n^{(k)}$; sample trajectory--time pairs uniformly with replacement from the stored batch and $\{0,\ldots,H-1\}$, or sample trajectories uniformly with replacement and average all $H$ times. Both conventions have the same conditional mean.
Use the empirical clipped actor gradient (with the advantage detached) and empirical Monte Carlo critic gradient, both evaluated at the pre-update parameters, then apply both updates. At a clipping tie, fix a measurable samplewise derivative rule in advance (for example, select the unclipped branch); it must not depend on the other samples in the minibatch.
The two minibatches may coincide. Thus their noises may be correlated.
The pre-update filtration contains the entire stored batch, all past draws, and the current parameters; conditional means are empirical means, not population gradients.

Assume $\pi_\theta(a\mid s)\ge p_{\min}>0$ for every relevant action and policy, and put $R_{\max}:=1/p_{\min}$ (or any verified smaller ratio bound).
This coverage condition is imposed separately; it does not follow from the average KL constraint.
Assume $|G_h|\le G_{\max}$, $|V_w(s)|\le V_{\max}$, and $\|\nabla_w V_w(s)\|\le G_V$ on the visited region, with matched terminal values.
Consequently every single-sample critic gradient has norm at most
$Q_c:=2(V_{\max}+G_{\max})G_V$.
Appendix~\ref{appdx:batch-bridge} gives the full sampling specification and derives variance bounds and uniform population discrepancies under adaptive reuse.
Appendix~\ref{par:batch-scope-discussion} discusses the finite-support and structured cases and their verified examples.

\section{Main Results}
\label{sec:main-results}

This section presents our main finite-time convergence guarantees for PPO-Clip.
We analyze the specified \emph{synchronous} actor--critic protocol;
Appendix~\ref{appdx:asyn} extends the framework to bounded asynchronous staleness.
All proofs are deferred to the appendix.

\subsection{Convergence of PPO-Clip}
\label{subsec:closing-loop}

We begin with our main closed-loop guarantee in the synchronous setting.
The result jointly controls (i) the \emph{stationarity} of the actor iterates with respect to the underlying RL objective
$J(\theta)$ and (ii) the \emph{tracking error} of the learned critic.
Crucially, the bounds are stated in terms of explicit hyperparameters and intrinsic noise/regularity constants.

\begin{theorem}[Closed-loop non-asymptotic convergence of PPO-Clip]
\label{thm:closed-loop-NK-explicit}
Consider the finite-batch protocol in Section~\ref{sec:prelim-prob}, with fixed initial actor $\theta_0$, $N$ batches of $B$ trajectories, and $K$ actual joint updates per batch; set $T=NK$.
Suppose Assumptions~\ref{ass:smooth-J0}--\ref{ass:kl0} and the stated coverage and boundedness premises hold, with
$0<\eta\le1/(4L_J)$ and $0<\beta_c\le\min\{\mu_q/(2L_c^2),1/\mu_q\}$.
For matched terminal values, define
{\small
\[
 \widetilde C_{\lambda,H}=\gamma(1-\lambda)\sum\nolimits_{j=0}^{H-2}(\gamma\lambda)^j,
 \, a=4G_\pi^2L_V^2 \widetilde C_{\lambda,H}^2,
 \, \epsilon_B=\min\{1,S_{\mathrm{traj}}/(4B)\},
 \, E_c=8Q_c^2\epsilon_B+4Q_c^2H\kappa.
\]
}
An empty sum is zero. The state-only GAE error cancels under the on-policy score
identity (Corollary~\ref{cor:gae-score-cancellation}); thus $0\le\widetilde C_{\lambda,H}\le\gamma$
and $a=0$ when $\lambda=1$ or $H=1$. The empirical and clipping terms in $F$ remain.
The actor bias floor is
{\small
\begin{align}
F:= {}&\underbrace{16R_{\max}^2G_\pi^2A_{\max}^2\epsilon_B}_{\text{finite-batch reuse}}
+\underbrace{8G_\pi^2A_{\max}^2\kappa(1+\sqrt H)^2}_{\text{trajectory/policy drift}}
+\underbrace{8G_\pi^2A_{\max}^2\kappa(1+1/\delta)^2}_{\text{clipping}}.
\label{eq:finite-batch-F}
\end{align}
}
Set
{\small
\[
 A_0=\underbrace{\frac{4(J_{\max}-J(\theta_0))}{\eta T}}_{\text{optimization}}+\underbrace{2L_J\eta\sigma_a^2}_{\text{actor noise}},
 \quad d=\frac{40L_{w^\star}^2\eta^2}{\mu_q^2\beta_c^2},
 \quad U=\underbrace{\frac{4\mathbb E\Delta_0}{\mu_q\beta_c T}}_{\text{critic initialization}}
 +\underbrace{\frac{5\beta_c\sigma_c^2}{\mu_q}}_{\text{critic noise}}
 +\underbrace{\frac{40E_c}{\mu_q^2}}_{\text{batch/target mismatch}}.
\]
}
Define the explicit coupling coefficient $\rho:=10ad$ and suppose $\rho<1$.
Then the actor stationarity and critic tracking bounds are
{\fontsize{8}{10}\selectfont
\begin{equation}
 \frac1T\sum\nolimits_{t<T}\mathbb E\|\nabla J(\theta_t)\|^2
 \le\frac{A_0+4F+4aU+4ad\sigma_a^2}{1-\rho},
 \quad
 \frac1T\sum\nolimits_{t<T}\mathbb E\Delta_t
 \le\frac{U+d(\sigma_a^2+2A_0+10F)}{1-\rho}.
 \label{eq:closed-loop-actor-repaired}
\end{equation}
}
The flattened index is $t=nK+k$, and $\Delta_t=\|w_t-w^\star(\theta_t)\|^2$.
\end{theorem}

\begin{corollary}[A batch bound from the clipped-gradient class]
\label{cor:structured-main}
Retain the protocol and regularity premises of Theorem~\ref{thm:closed-loop-NK-explicit},
but replace finite complete-trajectory support by finite layered state sets
$\mathcal S_h$, finite actions, bounded rewards/returns, and a tabular critic
$w\in[0,W]^{d_c}$ along every update, where
$d_c=\sum_{h<H}|\mathcal S_h|$ and $M=\sum_{h<H}|\mathcal S_h||\mathcal A|$.
Assume a positive raw-GAE envelope $A_{\max}>0$ holds on this entire critic cube, and select the
unclipped derivative at clipping ties. Put $C=C_{\lambda,H}\le2$ and
{\small
\[
\begin{gathered}
 \ell_B=\log(6M)+d_c\log(2+CWB/A_{\max}),
 \qquad e_c(B)=\frac{8(W^2+G_{\max}^2)}{HB},\\
 e_a(B)=\min\left\{4R_{\max}^2G_\pi^2A_{\max}^2,
 \frac{4R_{\max}^2G_\pi^2A_{\max}^2M^2}{H^2B}(\ell_B+3)\right\}.
\end{gathered}
\]
}
Both conclusions of the theorem hold after replacing the residuals by
{\small
\begin{equation}
 F_{\rm str}=4e_a(B)+8G_\pi^2A_{\max}^2\kappa
       [(1+\sqrt H)^2+(1+\delta^{-1})^2],\quad
 E_{c,\rm str}=2e_c(B)+4Q_c^2H\kappa.
 \label{eq:structured-main-residuals}
\end{equation}
}
Here $Q_c$ can be any uniform envelope of the whole-trajectory critic gradient;
minibatch variances are bounded separately for the chosen sampling convention.
\end{corollary}

Lemma~\ref{lem:structured-batch-bridge} proves these uniform deviations for the
actual clipped vector gradient, including adaptive reuse. Within each time--state--action
cell the advantage multiplier is $\widehat A$, its positive part, or its negative
part. A critic-parameter cover controls these three classes without assuming
continuity of the clipped gradient in actor parameters. The actor bound is
$O(M^2d_c\log B/(H^2B))$ times its squared envelope; it contains no complete-path
cardinality. This is a structured tabular result: a prefix-tree state representation
can itself have exponentially many states.
Proposition~\ref{prop:horizon-polynomial} verifies all premises along actual
updates on a nonconstant-objective family with $2^H$ action paths but $2H-1$ critic
coordinates, and gives a full bound polynomial in $H$.

\begin{corollary}[Update and fresh-rollout complexity]
\label{cor:rate-resources}
Fix the problem constants, clipping range, minibatch size, and $K$.
Consider runs with $T=NK$, stepsizes constant within each run,
$\eta=c_\eta T^{-3/5}$ and $\beta_c=c_\beta T^{-2/5}$, where the positive
constants ensure the stepsize restrictions and $\rho\le\rho_0<1$.
If $\kappa_T=O(T^{-2/5})$, then both average squared stationarity and average
critic parameter error are $O(T^{-2/5})$, using either
$B_T=\lceil S_{\rm traj}T^{2/5}\rceil$ in the finite-support theorem or
$B_T=\lceil c_B T^{2/5}\log(eT)\rceil$, $c_B>0$, in
Corollary~\ref{cor:structured-main}.
The fresh-trajectory count is $Q=NB=TB/K$; thus the respective sufficient counts
for error at most $\varepsilon$ are $O(\varepsilon^{-7/2})$ and
$\widetilde O(\varepsilon^{-7/2})$, with problem-dependent constants.
Environment transitions number $HQ$.
\end{corollary}

Appendix~\ref{appdx:rate-resources} proves the rate, supplies an implementable KL
certificate, and derives a fixed-rollout-budget tradeoff between $B$ and $K$.
The rate balances optimization, critic noise, and target movement; it is a
sufficient rate for this analysis, not a minimax claim or a comparison with global
optimality results for different PPO variants. The explicitly verified
horizon family has polynomial, but conservative, sample complexity.

\paragraph{Reading the guarantee: decomposition and amplification.}
The bounds separate optimization and stochastic noise from finite-batch reuse,
trajectory drift, clipping, and critic-target mismatch. A defining structural
feature is the prefactor $(1-\rho)^{-1}$.
We interpret $(1-\rho)^{-1}$ as a \emph{coupling amplification factor}: it quantifies how actor--critic feedback
amplifies \emph{all} finite-time error sources (optimization error, noise, drift, clipping bias, and critic tracking error).
The condition $\rho<1$ suffices to close the coupled inequalities; numerical informativeness also depends on the numerator's constants and error floors (Appendix~\ref{appdx:further_discussion}). For nonzero coupling, it motivates time-scale separation but does not establish a universal 5--10$\times$ learning-rate ratio or a necessary instability threshold.

\paragraph{Practical implications and tuning guidance.}
The optimization contribution decays as $1/(\eta NK)$ when the remaining constants and error floors are controlled.
The effect of increasing $K$ (more mini-batch reuse) depends on the resulting policy drift.
Proposition~\ref{prop:fixed-budget} makes this trade-off explicit under a certificate
$\kappa\le D K^2\eta^2$ and a fixed fresh-rollout budget.
At fixed $\kappa$, the theorem does not imply that larger $K$ increases drift.
The coupling factor $\rho$ scales as $(\eta/\beta_c)^2$ up to problem constants and $\widetilde C_{\lambda,H}^2$.

Once residual terms dominate the bound, further reducing its optimization term gives little improvement.
The bound motivates coordinating three controls:
\emph{(i)} tightening the trust region (smaller $\kappa$, e.g., more aggressive KL-based early stopping),
\emph{(ii)} balancing reduced critic target lag against increased noise within its admissible stepsize range, and
\emph{(iii)} improving critic quality (reducing tracking error or variance).

Beyond these aggregate trade-offs, the bound surfaces several concrete PPO-Clip--specific insights.
First, the KL budget $\kappa$ controls two distinct failure modes simultaneously: within-epoch distribution shift
($\propto \kappa$) and clipping distortion ($\propto \kappa/\delta^2$), giving distinct reasons to control the behavior-policy mismatch. Certified predictable stopping is treated in Appendix~\ref{appdx:stopping}; an uncalibrated empirical check alone does not establish the population premise.
Second, a smaller clipping range $\delta$ is not automatically more conservative: for fixed $\kappa$, the distortion
term scales as $\kappa/\delta^2$, so this upper bound becomes less informative for overly small clipping ranges; it does not establish a monotone law for the actual bias.
Finally, bootstrapping enters the coupling through $\widetilde C_{\lambda,H}^2$, with
$\widetilde C_{\lambda,H}\le\gamma\le1$. Matched terminal values give $a=\rho=0$ when
$\lambda=1$ or $H=1$. Holding other constants fixed, smaller $\lambda$ can tighten the
sufficient stepsize-ratio bound. Other constants, including critic coverage and return
bounds, can still worsen with $H$.

\paragraph{Connection to classical two-time-scale theory.}
Classical actor--critic analyses typically rely on asymptotic time-scale separation, often expressed as
$\eta_t/\beta_{c,t}\to 0$.
Theorem~\ref{thm:closed-loop-NK-explicit} provides a finite-time counterpart tailored to PPO-Clip,
in which time-scale separation appears explicitly through the dimensionless coupling coefficient $
\rho \;\propto\; (\eta/\beta_c)^2
\quad (\text{up to problem constants and } \widetilde C_{\lambda,H}^2).$
Thus, $\rho\ll 1$ means weak amplification in the bound; critic tracking also depends on $U$ and $d$.
In particular, $\lambda=1$ or $H=1$ gives $\rho=0$ independently of critic speed.
As $\rho\rightarrow1$, the bound increasingly amplifies noise and bias terms.
In the asymptotic limit of diminishing step sizes and vanishing trust-region budgets, the bound recovers the
classical two-time-scale picture. More formally, we have the following result.

\begin{proposition}[Two-time-scale stationarity in a single-index $t$ and horizon $T=NK$]
\label{prop:classical-connection-single-index}
Fix $H,\gamma,\lambda$, $\delta\in(0,1)$, and the initial actor $\theta_0$.
For each $T\in\mathbb{N}$, consider a PPO-Clip run with $N=N(T)$ outer iterations and $K=K(T)$ actual inner actor--critic updates per iteration, where $T=N(T)K(T)$ is deterministic.
Assume the model and regularity premises of Theorem~\ref{thm:closed-loop-NK-explicit}, including the finite-batch sampling and Monte Carlo target protocol, hold with constants uniform in $T$, fixed finite $S_{\mathrm{traj}}$, batch size $B_T\to\infty$, and $\sup_T\mathbb E[\Delta_{0,0}]<\infty$.
Define the flattened time index $
t := nK+k, (n,k)\in\{0,\dots,N-1\}\times\{0,\dots,K-1\}.
$
\noindent Let $\theta_t := \theta_n^{(k)}$ and $\Delta_t := \Delta_{n,k}$ denote the corresponding actor iterate and critic tracking error, respectively. Assume the run uses positive (possibly $T$-dependent) hyperparameters $(\eta_T,\beta_{c,T},\kappa_T)$ satisfying
{\small
\begin{align*}
\eta_T \to 0,\quad \beta_{c,T}\to 0,\quad \kappa_T\to 0, \quad  \eta_T\,T \to \infty, \quad 
\beta_{c,T}\,N(T) \to \infty, \quad
\frac{\eta_T}{\beta_{c,T}} \to 0, \quad
\frac{\eta_T}{\beta_{c,T}^2\,T}\to 0.
\end{align*}
}
\noindent If $\hat t$ is sampled uniformly from $\{0,\dots,T-1\}$ independently of the run, then
{\small
\[
\mathbb{E}\big[\|\nabla J(\theta_{\hat t})\|^2\big]\to 0,
\qquad
\mathbb{E}\big[\Delta_{\hat t}\big]\to 0,
\]
}
\noindent recovering the classical two-time-scale actor--critic stationarity and critic tracking conclusions.
\end{proposition}

\paragraph{Takeaway.}
Proposition~\ref{prop:classical-connection-single-index} shows that the finite-time
closed-loop bound is consistent with the classical two-time-scale actor--critic limit. What our analysis adds beyond classical two-time-scale theory is an explicit \emph{finite-time} characterization of how this limit
is approached: time-scale separation enters quantitatively through $(\eta/\beta_c)^2$ (via $\rho$ and $(1-\rho)^{-1}$), while PPO-Clip
introduces additional bias floors controlled by the KL budget $\kappa$ and the clipping threshold $\delta$ through $\kappa/\delta^2$.
These terms describe how clipping, trust-region enforcement, and critic tracking enter the finite-time bound;
the displayed diminishing schedules are sufficient to recover the asymptotic two-time-scale limit.


\paragraph{Asynchronous extension.}
The closed-loop analysis extends to a single-gradient parameter-server model with
bounded actor and critic delays and a nonanticipative sampling schedule.
The resulting stationarity and critic-tracking guarantees include additional staleness
penalties and require a delay-dependent critic stepsize restriction alongside the
sufficient coupling condition $\rho_{\mathrm{async}}<1$.
The complete model, assumptions, Theorem~\ref{thm:async-closed-loop-explicit}, and its
proof appear in Appendix~\ref{appdx:asyn}.

\subsection{Empirical Study}

\begin{figure*}[!t]
\vspace{-6mm}
\centering
    \begin{subfigure}[t]{0.24\textwidth}
        \centering
        \includegraphics[width=\linewidth]{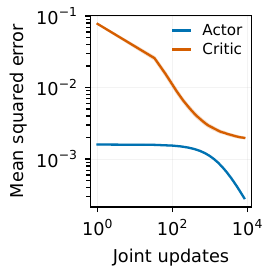}
        \caption{Joint errors}
        \label{fig:joint-errors}
    \end{subfigure}\hfill
    \begin{subfigure}[t]{0.24\textwidth}
        \centering
        \includegraphics[width=\linewidth]{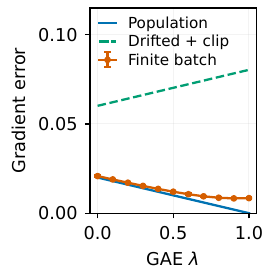}
        \caption{GAE and residuals}
        \label{fig:gae-cancellation}
    \end{subfigure}\hfill
    \begin{subfigure}[t]{0.24\textwidth}
        \centering
        \includegraphics[width=\linewidth]{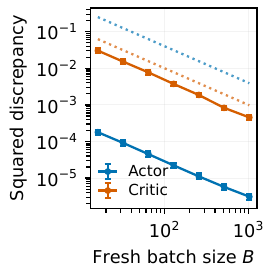}
        \caption{Batch discrepancy}
        \label{fig:adaptive-batch}
    \end{subfigure}\hfill
    \begin{subfigure}[t]{0.24\textwidth}
        \centering
        \includegraphics[width=\linewidth]{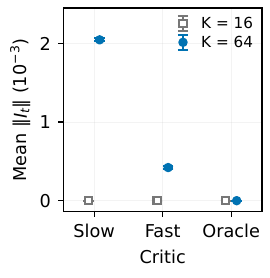}
        \caption{Clipping interaction}
        \label{fig:critic-clipping}
    \end{subfigure}
\caption{Two-step diagnostics: (a) joint average squared errors, (b) GAE bias and finite-batch
RMS error, (c) adaptive batch discrepancy with expectation bounds (dotted).
Eight-step diagnostic: (d) mean critic--clipping interaction with clipping enabled.
Bands/bars show 95\% bootstrap intervals.}
\label{fig:exp}
\vspace{-4mm}
\end{figure*}

Our experiments provide controlled illustrations of the main mechanisms
appearing in the theory. Figures~\ref{fig:joint-errors}--\ref{fig:adaptive-batch} use the
two-step example from Appendix~\ref{appdx:finite-batch-witness}.
Figure~\ref{fig:joint-errors} shows that actor stationarity error and critic
tracking error decrease together under joint updates, illustrating the two
quantities controlled simultaneously by Theorem~\ref{thm:closed-loop-NK-explicit}.
Figure~\ref{fig:gae-cancellation} isolates the actor-bias decomposition: the
critic-induced population GAE bias vanishes at $\lambda=1$, as predicted by
the score--baseline cancellation, while finite-batch and behavior-policy/clipping
residuals remain. Figure~\ref{fig:adaptive-batch} then examines the statistical
effect of reusing the same rollout batch: despite the actor and critic adapting
to that batch, their empirical-to-population discrepancies decrease as the
fresh batch size $B$ grows, consistent with our uniform finite-batch bounds.

Figure~\ref{fig:critic-clipping} isolates the critic--clipping interaction
during training on an eight-step chain. At $K=64$, critic error changes both
clipping decisions and distortion: mean $\|I_t\|$ is $2.04\times10^{-3}$
for the slow critic and $4.21\times10^{-4}$ for the fast critic.
This illustrates the closed-loop mechanism in our theory: critic tracking
affects clipped actor updates, while actor updates move the critic's value
target. At $K=16$, clipping is inactive; the exact-value oracle has
$I_t=0$ by construction.

These studies isolate and validate  three ingredients of the theory:
joint actor--critic tracking, finite-batch error under adaptive reuse, and
the interaction between critic error and clipping. They are mechanism
checks rather than validations of the theorem's rates or constants: the
eight-step setting does not satisfy a separately verified theorem
certificate, and neither bound tightness nor the predicted convergence rate
is tested here. Full protocols and additional results are given in
Appendix~\ref{appdx:exp}.

\section{Related Work}
\label{sec:related-work}

An extended review appears in Appendix~\ref{appdx:related-work}.

\paragraph{Policy gradients and actor--critic methods.}
Classical theory establishes stochastic policy-gradient convergence under regularity and sampling assumptions~\citep{sutton1999policy,williams1992simple}.
Learned critics can reduce gradient variance but can introduce bias when approximate~\citep{schulman2016gae,schulman2017proximalpolicyoptimizationalgorithms,wen2021characterizing,kumar2023sample}.
Asymptotic and non-asymptotic actor--critic analyses use stochastic approximation and non-convex optimization~\citep{liu2025non,zhong2023theoretical,holzleitner2021convergence},
including finite-time critic-tracking guarantees~\citep{wu2020finite}.
We analyze interleaved critic fitting with clipping, raw GAE, and finite-batch reuse under explicit assumptions.

\paragraph{Theory of PPO and clipping.}
Existing PPO analyses vary in policy, value-estimation, and sampling models~\citep{huang2024ppo,liu2023neuralproximaltrustregionpolicy,huang2021neural,jin2023stationary,liu2025non,chen2023finite}.
\citet{jin2023stationary} retain nonsmooth clipping and multiple updates; \citet{huang2024ppo} include learned value estimation.
These features alone do not distinguish our analysis.
Table~\ref{tab:closest-theory} compares the updates and statistical premises.
Our contribution derives advantage and sampling discrepancies for interleaved raw-GAE/Monte Carlo updates
and closes the resulting actor--critic recursions.
For finite layered MDPs with tabular critics, Corollary~\ref{cor:structured-main} controls adaptive reuse
through trajectory-function classes rather than complete-path counting;
Proposition~\ref{prop:horizon-polynomial} verifies a growing-horizon family with polynomial constants.
These conditional stationarity and tracking guarantees cannot be directly ranked against rates
under different objectives, sampling laws, or geometries.

\begin{table}[t]
\vspace{-8mm}
\centering
\small
\setlength{\tabcolsep}{3pt}
\renewcommand{\arraystretch}{1.12}
\caption{Closest theoretical comparisons under each work's assumptions.
Our row describes the structured tabular specialization.}
\label{tab:closest-theory}
\begin{tabular}{@{}>{\raggedright\arraybackslash}p{.13\linewidth}
>{\raggedright\arraybackslash}p{.20\linewidth}
>{\raggedright\arraybackslash}p{.16\linewidth}
>{\raggedright\arraybackslash}p{.24\linewidth}
>{\raggedright\arraybackslash}p{.19\linewidth}@{}}
\toprule
Work & Actor update & Critic treatment & Statistical control & Conclusion \\
\midrule
\citet{jin2023stationary}
& Clipped surrogate gradients
& Estimated advantages; no critic recursion
& Sampling/advantage bias retained as $\phi_n$
& Stationarity up to bias \\
\addlinespace
\citet{huang2024ppo}
& Entropic mirror descent; neural policy regression
& Neural action-value TD
& Evaluation/improvement errors controlled by width and inner iterations
& Global optimality for analyzed scheme \\
\addlinespace
\citet{wu2020finite}
& Online score update with TD error
& Simultaneous linear TD(0)
& Markovian data; uniform mixing; critic tracking
& Stationarity up to approximation error \\
\addlinespace
\citet{doering2026approximate}
& Symmetric clipped-gradient proxy
& Bounded advantage bias assumed
& Finite-buffer reuse with random reshuffling
& Epoch-start stationarity with error floors \\
\addlinespace
This work
& Interleaved clipping; recomputed raw GAE
& Tabular MC regression on stored returns
& Uniform adaptive actor/critic batch bounds; population KL control
& Joint stationarity and critic parameter tracking \\
\bottomrule
\end{tabular}
\vspace{-4mm}
\end{table}

\section{Concluding Discussion}
\label{sec:discussion}

Our non-asymptotic, closed-loop analysis of PPO-Clip jointly characterizes policy stationarity and learned-critic tracking in synchronous and asynchronous settings. Under an explicit sampling protocol, it captures actor--critic coupling, nonsmooth clipping, old-policy trajectories, and finite-sample effects from multi-epoch batch reuse.


\clearpage


\bibliography{reference}
\bibliographystyle{iclr2027_conference}

\clearpage
\appendix
\pdfbookmark[0]{Appendix contents and proof guide}{appendix-guide}
\section*{Appendix Contents and Proof Guide}

The appendices provide background, complete proofs, verified examples, and experimental details.
The contents and result references below are linked to their locations in the manuscript.

\begingroup
\small
\renewcommand{\arraystretch}{1.22}
\newcommand{\appendixentry}[1]{\ref{#1} & \hyperref[#1]{\nameref*{#1}} & \pageref{#1}\\}
\begin{tabular}{@{}p{.07\linewidth}p{.82\linewidth}r@{}}
\toprule
 & \textbf{Appendix} & \textbf{Page}\\
\midrule
\appendixentry{appdx:background}
\appendixentry{appdx:related-work}
\appendixentry{appdx:useful-lemmas}
\appendixentry{appdx:open-loop}
\appendixentry{appdx:closed-loop}
\appendixentry{appdx:rate-resources}
\appendixentry{appdx:classical-limit}
\appendixentry{appdx:stopping}
\appendixentry{appdx:finite-batch-witness}
\appendixentry{appdx:horizon-witness}
\appendixentry{appdx:asyn}
\appendixentry{appdx:exp}
\appendixentry{appdx:further_discussion}
\appendixentry{appdx:critic-extensions}
\bottomrule
\end{tabular}
\endgroup

\paragraph{Locations of the main proofs.}
The results stated in the main text are proved at the following locations. The asynchronous
theorem is stated and proved in Appendix~\ref{appdx:asyn}.

\begingroup
\small
\renewcommand{\arraystretch}{1.22}
\begin{tabular}{@{}p{.21\linewidth}p{.48\linewidth}p{.21\linewidth}@{}}
\toprule
\textbf{Result} & \textbf{Conclusion} & \textbf{Proof location}\\
\midrule
Theorem~\ref{thm:closed-loop-NK-explicit} & Synchronous stationarity and critic tracking & Appendix~\ref{appdx:closed-loop}, p.~\pageref{appdx:closed-loop}\\
Corollary~\ref{cor:structured-main} & Structured finite-batch bound & Appendix~\ref{appdx:rate-resources}, p.~\pageref{appdx:rate-resources}\\
Corollary~\ref{cor:rate-resources} & Update and fresh-rollout complexity & Appendix~\ref{appdx:rate-resources}, p.~\pageref{appdx:rate-resources}\\
Proposition~\ref{prop:classical-connection-single-index} & Classical two-time-scale limit & Appendix~\ref{appdx:classical-limit}, p.~\pageref{appdx:classical-limit}\\
Theorem~\ref{thm:async-closed-loop-explicit} & Asynchronous stationarity and tracking & Section~\ref{appdx:async-closed-loop-proof}, p.~\pageref{appdx:async-closed-loop-proof}\\
\bottomrule
\end{tabular}
\endgroup

\paragraph{Reading the dependency graphs.}
An arrow points from a proof ingredient to the result that uses it.
Solid arrows denote applications of results; dashed arrows denote reuse or adaptation of
their proof arguments. The graphs show the principal dependencies and omit routine
inequalities and some transitive edges. Each result retains its own stated assumptions.
Lemma~\ref{lem:gae-bias} and Corollary~\ref{lem:population-clip-bias} have the same numerical
identifier but are distinct results, so every node includes its result type.

\clearpage
\section*{Synchronous Result Dependencies}

\begingroup
\tikzset{
 result/.style={draw=black!65,rounded corners=1.5pt,fill=black!2,
   align=center,text width=3.8cm,minimum height=.72cm,inner sep=4pt,font=\small},
 keyresult/.style={result,fill=black!9},
 dep/.style={-{Stealth[length=1.8mm]},draw=black!75,thick},
 reuse/.style={dep,dashed}}

\paragraph{From GAE and clipping bounds to the coupled guarantee.}
The finite-batch bound supplies the actor bias and the critic's empirical-to-population
error. The two open-loop estimates are then combined to prove the synchronous result.

\begin{center}
\begin{tikzpicture}[x=1cm,y=1cm]
\node[result] (gae) at (0,0) {Lemma~\ref{lem:gae-bias}\\GAE bias};
\node[result] (clip) at (6,0) {Lemmas~\ref{lem:clip-localization-main}, \ref{lem:kl-bounds-Et-main}\\Clipping and KL control};
\node[result] (cancel) at (0,-1.25) {Corollary~\ref{cor:gae-score-cancellation}\\Score-baseline cancellation};
\node[result] (distort) at (6,-1.25) {Corollary~\ref{lem:population-clip-bias}\\Population clipping bias};
\node[result] (batch) at (3,-2.65) {Lemma~\ref{lem:finite-batch-bridge}\\Adaptive finite-batch bias};
\node[result] (actor) at (0,-4.05) {Theorem~\ref{thm:actor-open-loop-NK}\\Open-loop actor bound};
\node[result] (critic) at (6,-4.05) {Theorem~\ref{thm:critic-open-loop-NK}\\Open-loop critic bound};
\node[keyresult] (main) at (3,-5.45) {Theorem~\ref{thm:closed-loop-NK-explicit}\\Synchronous closed loop};
\draw[dep] (gae) -- (cancel);
\draw[dep] (clip) -- (distort);
\draw[dep] (cancel) -- (batch);
\draw[reuse] (distort) -- (batch);
\draw[dep] (batch) -- (actor);
\draw[dep] (batch) -- (critic);
\draw[dep] (actor) -- (main);
\draw[dep] (critic) -- (main);
\end{tikzpicture}
\end{center}
The clipping argument is reproduced within Lemma~\ref{lem:finite-batch-bridge}.
Theorem~\ref{thm:critic-open-loop-NK} proves a general biased tracking bound;
the arrow into it identifies the finite-batch specialization of that bias.

\paragraph{Structured statistics, rates, and the classical limit.}
The structured concentration lemma replaces complete-trajectory counting. Its residuals
enter the same synchronous proof, after which the rate and budget conclusions follow.

\begin{center}
\begin{tikzpicture}[x=1cm,y=1cm]
\node[result] (finite) at (0,0) {Lemma~\ref{lem:finite-batch-bridge}\\Population comparisons};
\node[result] (uniform) at (6,0) {Lemma~\ref{lem:structured-batch-bridge}\\Uniform tabular gradients};
\node[result] (replacement) at (6,-1.3) {Corollary~\ref{cor:structured-batch-replacement}\\Structured residuals};
\node[keyresult] (main) at (0,-1.3) {Theorem~\ref{thm:closed-loop-NK-explicit}\\Synchronous bound};
\node[keyresult] (structured) at (6,-2.6) {Corollary~\ref{cor:structured-main}\\Structured coupled bound};
\node[result] (other) at (0,-2.9) {Proposition~\ref{prop:classical-connection-single-index}: limit\\Proposition~\ref{prop:fixed-budget}: budget};
\node[result] (rate) at (3,-4.25) {Corollary~\ref{cor:rate-resources}\\Update and rollout rates};
\draw[reuse] (finite.east) -- (replacement.north west);
\draw[dep] (uniform) -- (replacement);
\draw[dep] (replacement) -- (structured);
\draw[reuse] (main.east) -- (structured.north west);
\draw[dep] (main) -- (other);
\draw[dep] (main.south east) -- (3,-2.05) -- (rate.north);
\draw[dep] (structured) -- (rate);
\end{tikzpicture}
\end{center}
The two incoming arrows to Corollary~\ref{cor:rate-resources} represent its finite-support
and structured alternatives. Proposition~\ref{prop:fixed-budget} also admits the structured
residual coefficients described in Appendix~\ref{appdx:rate-resources}.
\endgroup

\clearpage
\section*{Dependencies for Extensions and Examples}

\begingroup
\tikzset{
 result/.style={draw=black!65,rounded corners=1.5pt,fill=black!2,
   align=center,text width=3.8cm,minimum height=.72cm,inner sep=4pt,font=\small},
 keyresult/.style={result,fill=black!9},
 dep/.style={-{Stealth[length=1.8mm]},draw=black!75,thick},
 reuse/.style={dep,dashed}}

\paragraph{Asynchronous convergence.}
The asynchronous theorem combines finite-batch bias at stale parameters with a separately
proved delayed-critic bound. It reuses the actor smoothness argument, rather than applying
the synchronous convergence theorem directly.

\begin{center}
\begin{tikzpicture}[x=1cm,y=1cm]
\node[result] (actor) at (0,0) {Theorem~\ref{thm:actor-open-loop-NK}\\Actor smoothness argument};
\node[result] (batch) at (4.7,0) {Lemma~\ref{lem:finite-batch-bridge}\\Finite-batch bias};
\node[result] (drift) at (9.4,0) {Lemma~\ref{lem:ps-staleness-drift}\\Staleness drift};
\node[result] (critic) at (9.4,-1.5) {Lemma~\ref{lem:async-critic-tracking}\\Delayed critic tracking};
\node[keyresult] (async) at (4.7,-3.1) {Theorem~\ref{thm:async-closed-loop-explicit}\\Asynchronous closed loop};
\draw[dep] (drift) -- (critic);
\draw[dep] (batch) -- (async);
\draw[dep] (drift.south west) -- (7,-.85) -- (7,-2.1) -- (async.north east);
\draw[reuse] (actor) -- (async);
\draw[dep] (critic.south) -- (9.4,-3.1) -- (async.east);
\end{tikzpicture}
\end{center}
Theorem~\ref{thm:async-closed-loop-explicit} additionally requires its nonanticipative
schedule, mixed-parameter regularity, delay-dependent critic stepsize, and coupling condition.
The population clipping specialization is recorded separately:
\begin{center}
\begin{tikzpicture}[x=1cm,y=1cm]
\node[result] (clip) at (0,0) {Lemma~\ref{lem:kl-bounds-Et-main}\\Corollary~\ref{lem:population-clip-bias}};
\node[result] (asyncclip) at (6,0) {Lemma~\ref{lem:ps-async-kl-clip}\\Asynchronous clipping};
\draw[dep] (clip) -- (asyncclip);
\end{tikzpicture}
\end{center}
The final asynchronous proof uses this clipping control through the finite-batch argument;
Lemma~\ref{lem:ps-async-kl-clip} is not a separate prerequisite for that proof.

\paragraph{Predictable stopping and verified examples.}
The stopping result adapts the one-step inequalities to a predictable gate. The growing-horizon
example combines structured statistics with explicit bounds on every remaining problem constant.

\begin{center}
\begin{tikzpicture}[x=1cm,y=1cm]
\node[result] (proofs) at (0,0) {Lemma~\ref{lem:finite-batch-bridge}; Theorems\\\ref{thm:actor-open-loop-NK}, \ref{thm:critic-open-loop-NK}, \ref{thm:closed-loop-NK-explicit} (proofs)};
\node[result] (structured) at (6,0) {Lemma~\ref{lem:structured-batch-bridge}\\Corollary~\ref{cor:structured-batch-replacement}};
\node[result] (stop) at (0,-1.6) {Proposition~\ref{prop:predictable-stopping}\\Predictable stopping};
\node[result] (horizon) at (6,-1.6) {Proposition~\ref{prop:horizon-polynomial}\\Polynomial horizon bound};
\draw[reuse] (proofs) -- (stop);
\draw[dep] (structured) -- (horizon);
\draw[reuse] (proofs.east) -- node[below,sloped,font=\footnotesize] {closure} (horizon.north west);
\end{tikzpicture}
\end{center}
Appendix~\ref{appdx:finite-batch-witness} gives the two-step example, using the synchronous
closure with instance-specific constants and the schedules of Corollary~\ref{cor:rate-resources}.
Appendix~\ref{appdx:horizon-witness} extends that construction, but derives its horizon-dependent
resources explicitly. The loss-gap discussion in Section~\ref{subsec:loss-gap-tracking} is a
conditional route requiring additional drift and calibration estimates; it is not a proved
replacement theorem. Section~\ref{subsec:critic-approximation-error} instead derives the
approximation-error extension using the parameter-distance argument.
\endgroup

\clearpage

\section{Further Background}\label{appdx:background}

This appendix reviews the reinforcement-learning (RL) background for PPO-style fine-tuning of large language models
(LLMs), with emphasis on two features of the analyzed training loop:
\emph{(i) multi-epoch mini-batch reuse} (multiple gradient steps on a fixed rollout batch), and
\emph{(ii) a jointly trained critic} whose approximation error and tracking lag feed back into the actor update.
The \emph{closed-loop} actor--critic coupling arises because the actor update relies on critic-based advantage
estimates, while the
critic is trained to regress to policy-dependent targets that drift as the actor evolves.

\subsection{RL paradigms and where PPO-Clip sits}
\label{subsec:rl-paradigms}

At a high level, RL algorithms can be categorized along two orthogonal axes:
\emph{(i) on-policy vs.\ off-policy learning} (whether updates are computed using data from the current policy), and
\emph{(ii) value-based vs.\ policy-based optimization} (whether one learns a value function and derives a policy, or
directly optimizes a parameterized policy). PPO belongs to the family of \emph{on-policy policy-gradient methods}
augmented with a learned value baseline (actor--critic). However, modern PPO implementations used for LLM fine-tuning
depart from ``pure'' on-policy updates in an important way: they \emph{reuse} a fixed on-policy rollout batch for
multiple epochs. This reuse improves sample efficiency and amortizes the high cost of LLM rollouts, but it creates a
\emph{controlled distribution shift} within an outer iteration because the policy is updated while the data batch is
held fixed. This training structure consists of the following outer and inner loops:
\begin{itemize}[leftmargin=*]
    \item \textbf{Outer loop (on-policy data collection).} At iteration $n$, a fresh batch $\mathcal{D}_n$ is collected
    under $\pi_{\theta_{\rm old}}$.
    \item \textbf{Inner loop (multi-epoch reuse).} The batch $\mathcal{D}_n$ is reused for $K$ inner updates, during which
    the actor parameters drift away from $\theta_{\rm old}$.
\end{itemize}
PPO-Clip controls this drift by combining \emph{importance ratios} (to correct for mismatch),
\emph{clipping} (to prevent unstable corrections), and a \emph{KL trust region} (to keep the mismatch small on average).
Our analysis is tailored to this ``weakly off-policy'' regime: the updates are computed using the behavior policy
$\pi_{\theta_{\rm old}}$ but remain meaningful because the KL budget limits how far $\pi_\theta$ can move during reuse.

\subsection{RL for LLM fine-tuning and the token-level MDP view}
\label{subsec:rl-llm-app}

We model RL fine-tuning of an LLM as a finite-horizon episodic MDP at the token level.
At the beginning of each episode, a prompt $x\in\mathcal{X}$ is sampled from $\mathbb{P}$.
Conditioned on $x$, the LLM generates a length-$H$ sequence $(a_0,\ldots,a_{H-1})\in\mathcal{V}^H$.
The state encodes the prompt and prefix,
\[
s_t := (x,a_{0:t-1}),\qquad s_0 := (x,\emptyset),
\]
and actions correspond to next-token choices, $\mathcal{A}:=\mathcal{V}$.
Transitions are deterministic, $s_{t+1}=s_t\circ a_t$.
A trajectory is
\[
\tau := (s_0,a_0,s_1,a_1,\ldots,s_{H-1},a_{H-1},s_H).
\]
In RLHF-style settings, per-token rewards may be sparse, and a terminal reward $\bar r(\tau)$ (from a learned reward
model) often dominates. Given per-step rewards $\dot r(s_t,a_t)$ and an optional terminal reward $\bar r(\tau)$, the
return from time $t$ is
\begin{equation}
\label{eq:return-app}
G_t(\tau)
:=\sum_{k=t}^{H-1}\gamma^{k-t}\dot r(s_k,a_k)+\gamma^{H-t}\bar r(\tau),
\end{equation}
where $\gamma\in(0,1]$ (often $\gamma=1$ in finite-horizon alignment).
The corresponding value functions under $\pi_\theta$ are
\[
V^{\pi_\theta}(s):=\mathbb{E}[G_t(\tau)\mid s_t=s],\qquad
Q^{\pi_\theta}(s,a):=\mathbb{E}[G_t(\tau)\mid s_t=s,a_t=a],
\]
and the advantage is $A^{\pi_\theta}(s,a):=Q^{\pi_\theta}(s,a)-V^{\pi_\theta}(s)$.
Advantages provide the variance-reduced learning signal for policy gradients, but they are not available in closed
form for LLMs; they must be estimated from finite rollouts, motivating an actor--critic implementation.

\subsection{Actor--critic coupling as a closed-loop learning system}
\label{subsec:ac-coupling-app}

\paragraph{Critic learning as policy-conditioned regression.}
PPO trains a critic $V_w(s)$ to approximate the policy-dependent value $V^{\pi_\theta}(s)$.
At the population level, this can be written as a policy-conditioned regression problem
\begin{equation}
\label{eq:critic-loss-app}
\mathcal{L}_c(\theta,w)
:=\mathbb{E}_{\tau\sim P_\theta}\!\left[\frac1H\sum_{h=0}^{H-1}(V_w(s_h)-G_h(\tau))^2\right],
\end{equation}
where the state weighting is the uniform average over times and $G_h$ is the Monte Carlo target in \eqref{eq:return-app}. The analyzed update reuses stored Monte Carlo targets; bootstrapped critic targets are outside the stated theorem. Let $w^\star(\theta)\in\arg\min_w\mathcal{L}_c(\theta,w)$ denote a population minimizer. We measure
critic suboptimality by the tracking error
\begin{equation}
\label{eq:tracking-error-app}
\Delta(\theta,w):=\|w-w^\star(\theta)\|^2.
\end{equation}
This dependence on $\theta$ is the root of the coupling: as the actor changes, both the critic's sampling
distribution and its regression target drift.

\paragraph{Interleaved critic training.}
In the interleaved training pattern, the critic is updated \emph{inside the same
multi-epoch inner loop} as the actor. Each inner iteration $k$ draws a minibatch from the fixed rollout batch
$\mathcal{D}_n$ and performs:
\begin{enumerate}[leftmargin=*]
    \item compute both the Monte Carlo critic gradient and the clipped actor gradient at the pre-update parameter pair, recomputing raw GAE with that critic; and
    \item apply the critic descent and actor ascent steps, allowing their minibatches to coincide,
\end{enumerate}
with fixed update counts in the main theorem. Appendix~\ref{appdx:stopping} handles a predictable stopping gate with a population-KL certificate.
This joint inner-loop design captures the \emph{interleaving} between critic fitting and actor drift found in
large-scale PPO codebases: the critic tracks a moving target while shaping the actor's update direction through
advantages.

\paragraph{Advantage estimation and the bias--variance trade-off.}
Given a critic $V_w$, PPO forms temporal-difference residuals
\[
\hat\delta_t := \dot r(s_t,a_t)+\gamma V_w(s_{t+1})-V_w(s_t),
\]
and computes generalized advantage estimation (GAE)
\begin{equation}
\label{eq:gae-app}
\widehat A_t^{(\lambda)}:=\sum_{\ell=0}^{H-1-t}(\gamma\lambda)^\ell \hat\delta_{t+\ell},
\end{equation}
where $\lambda\in[0,1]$ trades variance against bias. We use the matched terminal value $V_w(s_H)=\bar r(\tau)$, recompute the raw estimates before each update, and detach them when differentiating the actor surrogate.
Bootstrapping changes the bias--variance trade-off. The bounds quantify the critic-induced population bias
through $\widetilde C_{\lambda,H}$ and $\Delta(\theta,w)$; they do not establish a monotone variance law.

\subsection{From policy gradients to PPO-Clip and KL trust regions}
\label{subsec:ppo-clip-app}

\paragraph{Policy gradients and importance ratios.}
Let $J(\theta):=H^{-1}\mathbb{E}_{\tau\sim\pi_\theta}[G_0(\tau)]$ be the normalized RL objective used throughout the analysis.
The policy-gradient theorem yields
\[
\nabla_\theta J(\theta)
=\mathbb{E}_{\tau\sim P_\theta}\!\left[\frac1H\sum_{h=0}^{H-1}\gamma^h\nabla_\theta\log\pi_\theta(a_h\mid s_h)\,A^{\pi_\theta}(s_h,a_h)\right].
\]
In multi-epoch reuse, updates are computed using data sampled from $\pi_{\theta_{\rm old}}$ with
importance ratios
\[
r_{\theta,\theta_{\rm old}}(s,a):=\frac{\pi_\theta(a\mid s)}{\pi_{\theta_{\rm old}}(a\mid s)}.
\]
Large ratios can yield high-variance updates; the clipped surrogate limits incentives for some large ratio changes.

\paragraph{Clipping as stability control (and as a source of bias).}
The PPO-Clip surrogate suppresses the sample gradient when the ratio exceeds $1+\delta$ for a positive advantage or falls below $1-\delta$ for a negative advantage. The opposite-sign branches remain active.
Clipping is intended to limit excessive surrogate updates but can introduce distortion relative to the
unclipped importance-weighted direction. A single action ratio does not correct the complete trajectory law; Lemma~\ref{lem:finite-batch-bridge} bounds that remaining discrepancy. In our analysis, the impact of clipping is
quantified through explicit ``clip-activation'' terms and their dependence on the KL budget and clipping threshold.

\paragraph{KL trust region and early stopping.}
Modern implementations often combine clipping with an empirical KL constraint (early stopping) to limit distribution
shift during batch reuse. If the empirical KL on a minibatch exceeds $\kappa$, the inner loop
breaks. Such an empirical statistic alone does not imply a population budget. The main theorem assumes that budget, and Appendix~\ref{appdx:stopping} certifies it for a specified predictable stopping rule. Analytically, the population budget plays two roles:
\begin{itemize}[leftmargin=*]
    \item it bounds the policy drift induced by multi-epoch reuse, keeping the surrogate meaningful; and
    \item it controls the prevalence of clipping by keeping likelihood ratios near one for most samples.
\end{itemize}
These mechanisms are central to obtaining interpretable finite-time bounds that reflect practical PPO behavior in
LLM training.

\subsection{The role of actor--critic dependency}
\label{subsec:why-coupled-matters}

A defining feature of PPO-Clip in practical LLM fine-tuning is that the actor and critic are \emph{trained jointly} and
\emph{interleaved}: each inner iteration performs (at least) one critic regression step and
one actor policy update on the same reused rollout batch.
This produces an intrinsically \emph{coupled} learning system rather than two independent optimizers.
The following observations explain how this dependency affects the update directions, critic targets, and
interaction with clipping.

\paragraph{(i) The actor update is computed through the critic.}
In PPO, the update direction is not the true policy gradient $\nabla J(\theta)$ but an estimate that depends on
critic-based advantages, e.g., $\widehat A^{(\lambda)}$ in \eqref{eq:gae-app}.
Consequently, the actor update admits a natural decomposition
\[
g(\theta,w)
~=~
\nabla J(\theta)
~+~
b(\theta,w)
~+~
\zeta,
\]
where $b(\theta,w)$ is the conditional-mean discrepancy from the true policy gradient and
$\zeta$ is centered sampling noise. Exact values remove the population critic-induced component;
empirical-batch, trajectory-policy, and clipping discrepancies can remain. Treating the critic as exact
also removes its dynamic tracking error from the analysis.

\paragraph{(ii) The critic target drifts with the actor (tracking, not static regression).}
Critic learning is not a stationary supervised-learning problem. The regression target is \emph{policy-dependent}:
both the visitation distribution and the conditional return law in \eqref{eq:critic-loss-app} change as $\theta$ changes.
Thus the critic must \emph{track} a moving minimizer $w^\star(\theta)$ rather than converge to a fixed optimum.
Neglecting this drift---for example, by treating $w^\star$ as fixed or assuming the critic ``converges'' between actor
updates---misses an essential mechanism in PPO: actor updates induce parameter drift, which creates critic tracking
error, which can feed back into subsequent actor updates.
This two-way dependency is particularly acute under multi-epoch reuse: within a single outer iteration, $\theta$ moves
while the batch and empirical regression targets remain fixed. The current-policy population minimizer moves, creating a mismatch between the empirical gradient and the gradient of the moving population loss.

\paragraph{(iii) Interleaving matters: critic and actor have comparable update frequencies.}
Many theoretical treatments simplify actor--critic by assuming a strict two-time-scale separation (e.g., the critic is
updated ``to convergence'' for each actor step, or the actor is effectively fixed during critic fitting).
While analytically convenient, such a separation is typically \emph{not} representative of modern PPO practice, where
both actor and critic are updated a small number of times per outer iteration and often with comparable compute
budgets. Interleaved inner updates capture this training pattern; equal update frequencies do not imply equal stepsizes.
If one assumes the critic is fully converged at each step, the analysis may predict negligible critic-induced bias and
may incorrectly attribute observed instability primarily to clipping or stochasticity, whereas in practice critic lag
and advantage mis-estimation can dominate early training and can remain non-negligible throughout.

\paragraph{(iv) PPO-specific mechanisms interact with critic error.}
PPO-Clip introduces additional nonlinearities---clipping and KL-based early stopping---whose effects depend on the
advantage estimates.
For instance, the sign and magnitude of $\widehat A(s,a)$ determine which branch of the clipped objective is active
and how much gradient distortion occurs when clipping activates.
If the critic underestimates or overestimates values, the resulting advantage errors can change:
(i) the empirical clip-activation frequency,
(ii) the effective step size implied by the surrogate, and
(iii) the likelihood of triggering the KL early-stopping condition.
Therefore, simplifying away critic error can also lead to an unfaithful characterization of clipping-induced bias and
trust-region behavior. In other words, clipping/KL control do not operate on an abstract ``true'' gradient signal;
they operate on a critic-dependent, data-dependent surrogate whose geometry is shaped by $w$.

\paragraph{(v) Scope of an open-loop surrogate.}
A common modeling shortcut is to analyze PPO as (stochastic) gradient ascent on a fixed surrogate objective
$\mathcal{L}^{\mathrm{clip}}(\theta;\theta_{\rm old})$ with bounded noise.
This viewpoint is useful for isolating clipping effects, but it becomes incomplete when used as a proxy for PPO
dynamics in practice because:
\begin{itemize}[leftmargin=*]
    \item critic error can bias the population actor direction, although its contribution cancels when $\lambda=1$ or $H=1$ for a fixed critic under on-policy sampling and matched terminal values; empirical-batch, policy-drift, and clipping errors can remain;
    \item the critic update rule influences the distribution of advantages, and hence the surrogate geometry itself;
    \item the critic error evolves endogenously with actor drift (a feedback loop), rather than acting as exogenous noise.
\end{itemize}
As a result, theoretical predictions derived under open-loop assumptions (e.g., monotone improvement, stable steps
under a fixed KL budget, or negligible bias floors) may fail to match observed behavior in LLM PPO runs, where the
critic is learned online, targets are bootstrapped, and updates are interleaved with multi-epoch reuse.

\paragraph{Takeaway.}
The coupled updates describe feedback pathways through which critic tracking and actor drift affect PPO training.
Our analysis treats PPO as a \emph{closed-loop learning system} in which these errors co-evolve, with clipping and
KL control modulating this coupling.

\subsection{Objective of the analysis}
\label{subsec:goal-app}

The analysis establishes \emph{non-asymptotic} convergence guarantees for PPO-Clip with a learned critic under a
restricted training regime that retains interleaved fitting, clipping, and multi-epoch reuse. The actor update uses
critic-dependent advantages, and the critic follows a policy-dependent regression objective whose minimizer moves
with the actor. The bounds quantify this \emph{closed-loop coupling} through the contribution of critic error to
actor stationarity and the effect of actor drift on critic tracking, alongside error terms from clipping and
multi-epoch reuse.

\section{Extended Related Work}
\label{appdx:related-work}

Our work is most closely related to research on (i) policy-gradient and actor--critic theory,
(ii) trust-region and proximal policy optimization,
(iii) theoretical analyses of PPO and clipping,
(iv) two-timescale stochastic approximation, and
(v) RLHF and large language model (LLM) alignment.

\paragraph{Policy gradients and actor--critic methods.}
Classical policy-gradient theory establishes convergence of stochastic gradient ascent on the expected return
under suitable regularity and sampling assumptions~\citep{sutton1999policy,williams1992simple}. Actor--critic methods extend this framework by introducing
a learned value function to reduce variance, which can introduce bias when approximate~\citep{schulman2016gae,schulman2017proximalpolicyoptimizationalgorithms,wen2021characterizing,kumar2023sample}.
A large body of work studies asymptotic and non-asymptotic convergence of actor--critic algorithms using tools
from stochastic approximation and non-convex optimization, often under i.i.d.\ trajectory abstractions or
mixing assumptions for Markovian sampling~\citep{liu2025non,zhong2023theoretical,holzleitner2021convergence}.
In particular, \citet{wu2020finite} analyze simultaneous actor and linear TD(0) updates for an
average-reward objective under Markovian sampling. Their Theorems~4.5 and~4.7 already couple
actor stationarity to critic tracking; Corollary~4.9 retains a function-approximation residual.
These results establish finite-time actor--critic coupling under a different protocol.
Our analysis also accounts for the clipped surrogate, recomputed raw GAE, and adaptive reuse of stored rollout batches.

\paragraph{Trust-region and proximal policy optimization.}
Trust-region methods such as TRPO were introduced to stabilize policy optimization by explicitly constraining
policy updates using KL divergence, yielding monotonic improvement guarantees under idealized assumptions~\citep{schulman2015trust,achiam2017constrained,garg2021proximal}.
PPO was proposed as a first-order alternative that approximates trust-region behavior using surrogate objectives,
most notably probability-ratio clipping and, in some variants, explicit KL penalties~\citep{schulman2017proximalpolicyoptimizationalgorithms,wang2020truly,guo2021cim,wu2021coordinated}. While PPO is now the
dominant algorithm in large-scale RL and RLHF, its theoretical understanding remains less complete than that of
earlier trust-region methods. Our analysis treats PPO as optimizing a clipped surrogate while targeting a smooth
return objective under a population KL constraint, and makes explicit how KL budgets simultaneously control within-epoch distribution
shift and clipping activation.

\paragraph{Theory of PPO and clipping.}
Several works study PPO and PPO-Clip under different policy, critic, and sampling models~\citep{huang2024ppo,liu2023neuralproximaltrustregionpolicy,huang2021neural,jin2023stationary,liu2025non,chen2023finite}.
In their ICLR 2024 paper, \citet{jin2023stationary} analyze repeated clipped-surrogate gradient
updates and retain an advantage/sampling discrepancy $\phi_n$ (their Eq.~(3)) in the
stationarity conclusions of Theorems~3.1--3.3. They do not derive this discrepancy from
an interleaved critic-fitting recursion. Our distinction is to bound the actual conditional
actor and critic discrepancies under stored-batch reuse and then eliminate the tracking term
through the coupled inequalities.
\citet{huang2024ppo} learn neural action values by TD, construct a target policy by entropic
mirror descent, and fit the neural actor to that target. Their Algorithm~1, Theorem~2, and Corollary~1
control policy evaluation and improvement errors through network widths and inner iterations.
This establishes global optimality for a different update and sampling model, including
uniformly sampled actions for policy improvement; it is not the joint raw-GAE/Monte Carlo
gradient process analyzed here.

The function-space approach of \citet{zou2025sampling} controls kernel-TD evaluation using
RKHS complexity and combines it with exponentiated natural-gradient policy improvement
(their Algorithm~1 and Theorem~12). Its statistical complexity bounds concern a different
policy-evaluation and improvement procedure. Our structured specialization instead controls
the actual clipped-gradient trajectory class and tabular critic-gradient class uniformly
over adaptively selected iterates (Corollary~\ref{cor:structured-main}).
Proposition~\ref{prop:horizon-polynomial} verifies a growing-horizon family in which these
bounds and the remaining regularity constants are polynomial despite exponentially many
complete trajectories. The different update and sampling models prevent a direct comparison of rates.

\paragraph{Two-timescale stochastic approximation and tracking.}
The coupled actor--critic recursion fits naturally within the theory of two-timescale stochastic approximation,
where one component tracks a moving target induced by another~\citep{holzleitner2021convergence,wu2020finite,zhang2020provably,khodadadian2022finite}.
Classical results establish convergence under timescale separation~\citep{holzleitner2021convergence,zhang2020provably}, while more recent work provides finite-time
tracking bounds that decompose error into transient, noise-floor, and drift-induced terms~\citep{wu2020finite,khodadadian2022finite}.
Our critic analysis builds on this perspective by treating the critic as a stochastic optimizer tracking the
policy-dependent minimizer of a regression objective, yielding an explicit finite-time bound in which actor
stepsizes directly control critic drift. This characterization is central to closing the loop in our PPO analysis.

\paragraph{RLHF and LLM alignment.}
Reinforcement learning from human feedback has become a standard paradigm for aligning large language models,
typically combining learned reward models, KL regularization against a reference policy, and PPO-style
optimization with minibatch reuse and multiple epochs~\citep{ouyang2022traininglanguagemodelsfollow,schulman2016gae,stiennon2020learning,hilton2022batch,yu2022surprising,sun2022you}.
While extensive empirical work studies stability, reward hacking, and hyperparameter heuristics in RLHF,
comparatively little theory addresses the finite-time behavior of PPO in this regime~\citep{dossa2021empirical,miao2025energy,yu2025dapoopensourcellmreinforcement,fu2025reward,chen2024odin}.
RLHF motivates the present model, but the proved objective is normalized expected return with a fixed reward model and a population KL constraint. The explicit coverage, raw-GAE, Monte Carlo-target, and critic regularity premises delimit its applicability.

\paragraph{Relation to our analysis.}
Building on these lines of work, we derive and close actor and critic bias bounds for the
specified interleaved, reused-batch PPO process. The structured specialization replaces trajectory-count
dependence by bounds on the relevant function classes while retaining the coverage,
realizability, critic geometry, and population-KL premises. The resulting guarantees concern
stationarity and critic tracking; they do not assert global optimality or cover arbitrary
neural PPO or LLM fine-tuning implementations.

\section{Useful Lemmas}
\label{appdx:useful-lemmas}

The lemmas in this appendix quantify how critic error, clipping, and batch reuse affect
the actor update. They provide two components of the convergence proof.

\begin{itemize}[leftmargin=*]
    \item \textbf{Critic-induced bias control.}
    These results bound the effect of value-estimation error on the policy update through
    GAE, with explicit dependence on critic tracking error and the bootstrapping parameter.

    \item \textbf{Clipping- and trust-region--induced distortion control.}
    These results localize the non-smoothness of the PPO-Clip surrogate to clipping events
    and bound the resulting deviation from the policy gradient under KL control.
\end{itemize}

The bounds depend explicitly on critic tracking error, the GAE parameter $\lambda$,
the clipping threshold $\delta$, and the KL budget $\kappa$. They apply at intermediate
iterates under batch reuse, without requiring either component to have converged.

\subsection{Bounds for critic-based advantage estimation}

We begin with lemmas controlling the accuracy of critic-based advantage estimates.
In PPO, the actor update is driven by $\widehat A^{(\lambda)}$, not the true advantage $A^{\pi_\theta}$.
Discrepancies between $V_w$ and $V^{\pi_\theta}$ can induce actor bias through GAE, although some value errors cancel in the population policy gradient.

The following lemma bounds the conditional mean error of GAE in terms of value-estimation error, with a coefficient depending on $\lambda$, the horizon, and the terminal convention. This is a bias bound; the variance of GAE is outside its scope.

\begin{lemma}[Finite-horizon GAE bias bound (critic-induced bias)]
\label{lem:gae-bias}
Fix a policy $\pi_\theta$, a finite horizon $H\ge1$, and $\gamma,\lambda\in[0,1]$.
The state includes the timestep, and rewards are integrable. Let $V_\theta$ denote the true value function under $\pi_\theta$
for the reward and terminal-value convention used by the algorithm, and let $V_w$ be an arbitrary fixed critic.
The value function satisfies the finite-horizon Bellman equation with that terminal convention.
Define $e(s):=V_w(s)-V_\theta(s)$ and assume $\|e\|_\infty<\infty$, including terminal states.
The policy and critic are fixed before sampling the on-policy trajectory; the critic is held fixed within it.
In particular, $e(s_H)$ is not assumed to vanish.

For a trajectory generated under $\pi_\theta$, define the TD residual based on $V_w$:
\[
\hat\delta_t^{(w)} := r_t + \gamma V_w(s_{t+1}) - V_w(s_t),
\]
and the finite-horizon GAE estimator
\[
\widehat A_t^{(\lambda,w)}
~:=~
\sum_{\ell=0}^{H-1-t}(\gamma\lambda)^\ell \hat\delta_{t+\ell}^{(w)}.
\]
Let
\[
\bar A_t^{(\lambda,w)}(s,a)
~:=~
\mathbb{E}\!\left[\widehat A_t^{(\lambda,w)} ~\middle|~ s_t=s,\,a_t=a\right]
\]
denote its conditional expectation. Then for all supported $(s,a)$,
\begin{equation}
\label{eq:gae-bias-bound-lemma}
\big|\bar A_t^{(\lambda,w)}(s,a)-A_\theta(s,a)\big|
~\le~
C_{\lambda,H,t}^{\mathrm{term}}\,\|e\|_\infty
~\le~
C_{\lambda,H}^{\mathrm{term}}\,\|e\|_\infty,
\end{equation}
where one may take
\[
C_{\lambda,H,t}^{\mathrm{term}}
:= 1 + (1-\lambda)\sum_{\ell=1}^{H-1-t}(\gamma\lambda)^{\ell-1}\gamma
+\gamma(\gamma\lambda)^{H-1-t},
\qquad
C_{\lambda,H}^{\mathrm{term}}:=\sup_{t\in\{0,\dots,H-1\}} C_{\lambda,H,t}^{\mathrm{term}}.
\]
In particular, the uniform bound
\begin{equation}
\label{eq:gae-constant}
C_{\lambda,H}^{\mathrm{term}}=1+\gamma
\end{equation}
holds for all $\gamma,\lambda\in[0,1]$, with empty sums equal to zero and zeroth powers equal to one.
If the critic and true value agree at the prescribed terminal boundary,
$V_w(s_H)=V_\theta(s_H)$, the terminal term can be omitted. In this case
$|\bar A_t^{(\lambda,w)}(s,a)-A_\theta(s,a)|\le C_{\lambda,H}\|e\|_\infty$, where
\[
C_{\lambda,H}:=1+\gamma(1-\lambda)\sum_{j=0}^{H-2}(\gamma\lambda)^j.
\]
This matching condition permits a zero terminal value or a known terminal payoff.
\end{lemma}

\begin{proof}
Fix $t$ and condition on $(s_t,a_t)=(s,a)$. We compare the conditional mean of GAE computed with $V_w$
with the true advantage $A_\theta(s,a)$.

\paragraph{Step 1: Express the $V_w$-TD residual as a perturbation of the $V_\theta$-TD residual.}
Define the TD residual based on the true value $V_\theta$:
\[
\delta_t^{(\theta)} := r_t + \gamma V_\theta(s_{t+1}) - V_\theta(s_t).
\]
Using $V_w = V_\theta + e$, we can write
\begin{align}
\hat\delta_t^{(w)}
&= r_t + \gamma (V_\theta(s_{t+1})+e(s_{t+1})) - (V_\theta(s_t)+e(s_t)) \nonumber\\
&= \delta_t^{(\theta)} + \gamma e(s_{t+1}) - e(s_t).
\label{eq:delta-perturb}
\end{align}

\paragraph{Step 2: Decompose $\widehat A_t^{(\lambda,w)}$ into the true GAE term plus an error term.}
By definition,
\begin{align}
\widehat A_t^{(\lambda,w)}
&= \sum_{\ell=0}^{H-1-t}(\gamma\lambda)^\ell \hat\delta_{t+\ell}^{(w)} \nonumber\\
&= \sum_{\ell=0}^{H-1-t}(\gamma\lambda)^\ell \delta_{t+\ell}^{(\theta)}
\;+\;
\sum_{\ell=0}^{H-1-t}(\gamma\lambda)^\ell \big(\gamma e(s_{t+\ell+1})-e(s_{t+\ell})\big).
\label{eq:gae-split}
\end{align}
The first sum is precisely the finite-horizon GAE computed with the true value function:
\[
\widehat A_t^{(\lambda,\theta)}
:=\sum_{\ell=0}^{H-1-t}(\gamma\lambda)^\ell \delta_{t+\ell}^{(\theta)}.
\]

\paragraph{Step 3: The true-value GAE has the correct conditional mean.}
For finite-horizon on-policy GAE/TD($\lambda$),
\begin{equation}
\label{eq:true-gae-unbiased}
\mathbb{E}\!\left[\widehat A_t^{(\lambda,\theta)} ~\middle|~ s_t=s,\,a_t=a\right]
~=~
A_\theta(s,a).
\end{equation}
To see this, the Bellman identity gives
$\mathbb E[\delta_t^{(\theta)}\mid s_t=s,a_t=a]=A_\theta(s,a)$.
For $j>t$, let $\mathcal H_j$ contain the trajectory history through $s_j$, before drawing $a_j$.
On-policy sampling and the Bellman identity imply
$\mathbb E[\delta_j^{(\theta)}\mid\mathcal H_j]=0$.
Since $(s_t,a_t)$ is measurable with respect to $\mathcal H_j$, the tower property gives
$\mathbb E[\delta_j^{(\theta)}\mid s_t,a_t]=0$.
Summing these conditional means with the finite GAE weights proves \eqref{eq:true-gae-unbiased}.

Thus, taking conditional expectation in \eqref{eq:gae-split} and using \eqref{eq:true-gae-unbiased},
\begin{equation}
\label{eq:mean-diff}
\bar A_t^{(\lambda,w)}(s,a) - A_\theta(s,a)
~=~
\mathbb{E}\!\left[
\sum_{\ell=0}^{H-1-t}(\gamma\lambda)^\ell \big(\gamma e(s_{t+\ell+1})-e(s_{t+\ell})\big)
~\middle|~ s_t=s,\,a_t=a
\right].
\end{equation}

\paragraph{Step 4: Telescope the error terms with $\lambda$.}
Let $L:=H-1-t$ for brevity.
Rewrite the right-hand side of \eqref{eq:mean-diff} as
\begin{align}
S
&:= \sum_{\ell=0}^{L}(\gamma\lambda)^\ell \big(\gamma e(s_{t+\ell+1})-e(s_{t+\ell})\big) \nonumber\\
&= -e(s_t) + \sum_{\ell=1}^{L}\gamma^\ell\lambda^{\ell-1}(1-\lambda)\,e(s_{t+\ell})
\;+\;
(\gamma\lambda)^{L}\gamma\,e(s_{t+L+1}).
\label{eq:telescoped}
\end{align}
To verify \eqref{eq:telescoped}, collect coefficients of each $e(s_{t+j})$:
\begin{itemize}
\item For $j=0$, only the $-\ell=0$ term contributes: coefficient $-1$.
\item For $1\le j \le L$, $e(s_{t+j})$ appears twice: from $+\gamma(\gamma\lambda)^{j-1}e(s_{t+j})$
and from $-(\gamma\lambda)^j e(s_{t+j})$. The net coefficient is
\[
\gamma(\gamma\lambda)^{j-1}-(\gamma\lambda)^j
=\gamma^j\lambda^{j-1}(1-\lambda).
\]
\item The last state $e(s_{t+L+1})$ appears only from the final $+\gamma(\gamma\lambda)^L e(s_{t+L+1})$ term.
\end{itemize}

Now take conditional expectation of \eqref{eq:telescoped} given $(s_t,a_t)=(s,a)$ and apply the triangle inequality:
\begin{align}
\big|\bar A_t^{(\lambda,w)}(s,a)-A_\theta(s,a)\big|
&= \big|\mathbb{E}[S\mid s_t=s,a_t=a]\big| \nonumber\\
&\le
\mathbb{E}\big[|e(s_t)|\mid s_t=s,a_t=a\big]
\nonumber\\
&\quad
+ (1-\lambda)\sum_{\ell=1}^{L} \gamma^\ell\lambda^{\ell-1}
\mathbb{E}\big[|e(s_{t+\ell})|\mid s_t=s,a_t=a\big]
\nonumber\\
&\quad
+ (\gamma\lambda)^{L}\gamma\,\mathbb{E}\big[|e(s_{t+L+1})|\mid s_t=s,a_t=a\big].
\label{eq:triangle}
\end{align}

\paragraph{Step 5: Sup-norm bound and the constant $C_{\lambda,H,t}^{\mathrm{term}}$.}
Since $|e(s)|\le \|e\|_\infty$ for all states $s$, \eqref{eq:triangle} yields
\[
\big|\bar A_t^{(\lambda,w)}(s,a)-A_\theta(s,a)\big|
\le
\Big(1 + (1-\lambda)\sum_{\ell=1}^{L}\gamma^\ell\lambda^{\ell-1} + (\gamma\lambda)^{L}\gamma\Big)\|e\|_\infty.
\]
This is exactly \eqref{eq:gae-bias-bound-lemma} with
\[
C_{\lambda,H,t}^{\mathrm{term}}=1 + (1-\lambda)\sum_{\ell=1}^{H-1-t}(\gamma\lambda)^{\ell-1}\gamma + \gamma(\gamma\lambda)^{H-1-t}.
\]
Finally, write $n=H-t$ and $q=\gamma\lambda$. Since $q\le\lambda$,
\[
\begin{aligned}
(1-\lambda)\sum_{j=0}^{n-2}q^j+q^{n-1}
&\le (1-\lambda)\sum_{j=0}^{n-2}\lambda^j+\lambda^{n-1}\\
&=1.
\end{aligned}
\]
Thus $C_{\lambda,H,t}^{\mathrm{term}}\le1+\gamma$, with equality at $t=H-1$.
This proves \eqref{eq:gae-constant}. If $e(s_H)=0$, deleting the terminal term in
\eqref{eq:triangle} gives the stated sharper coefficient.
The argument is on-policy and does not control old-policy sampling, a critic fitted
adaptively to the same trajectory, or further normalization or clipping of GAE.
\end{proof}

\begin{corollary}[Baseline cancellation in the population actor bias]
\label{cor:gae-score-cancellation}
Under Lemma~\ref{lem:gae-bias}, suppose the terminal values match and the policy
has a differentiable finite-action distribution with bounded score
$z_h=\nabla_\theta\log\pi_\theta(a_h\mid s_h)$, $\|z_h\|\le G_\pi$.
For the normalized objective $J=\mathbb E G_0/H$, define
\[
 \widetilde C_{\lambda,H}:=C_{\lambda,H}-1
 =\gamma(1-\lambda)\sum_{j=0}^{H-2}(\gamma\lambda)^j\le\gamma.
\]
Then
\begin{equation}\label{eq:gae-score-cancellation}
 \left\|\frac1H\mathbb E_{P_\theta}\sum_{h=0}^{H-1}
       \gamma^h z_h\widehat A_h^{(\lambda,w)}-\nabla J(\theta)\right\|
 \le G_\pi\widetilde C_{\lambda,H}\|V_w-V_\theta\|_\infty.
\end{equation}
In particular, this population critic-induced bias is zero if $\lambda=1$ or
$H=1$. This conclusion concerns the on-policy expectation with a fixed critic;
it does not remove the empirical, clipping, or old-policy discrepancies.
\end{corollary}
\begin{proof}
For each fixed state,
$\mathbb E[z_h\mid s_h]=\sum_a\nabla_\theta\pi_\theta(a\mid s_h)=0$.
Consequently the leading term $-e(s_h)$ in \eqref{eq:telescoped} cancels after
multiplication by $z_h$ and expectation. The matched terminal term is zero.
The remaining nonnegative coefficients sum to at most
$\gamma(1-\lambda)\sum_{j=0}^{H-2}(\gamma\lambda)^j$ at every time.
Use $\|z_h\|\le G_\pi$, $\gamma^h\le1$, and the policy-gradient identity.
Finally, $(1-\lambda)\sum_{j=0}^{H-2}(\gamma\lambda)^j\le1$,
including the endpoints by the finite-sum convention.
\end{proof}

\subsection{Handling PPO-Clip non-smoothness via event localization}

A major analytical difficulty in PPO arises from the non-smooth nature of the clipped surrogate objective.
The objective is not differentiable at likelihood ratios $r_{\theta,\theta_{\mathrm{old}}}=1\pm\delta$, which
precludes a direct application of standard smooth optimization arguments.

The key observation underlying our analysis is that this non-smoothness is \emph{localized}:
it only affects samples at or beyond the relevant sign-dependent clipping boundary.
We enclose the sign-dependent clip-active region and its boundary in the ratio event
\[
E_t := \{\, |r_{\theta_t,\theta_{\mathrm{old}}}(s,a)-1|\ge\delta \,\}.
\]
On the complement of this event, the sample derivative coincides with the unclipped importance-weighted surrogate derivative. This does not by itself make the reused-batch update unbiased for the policy gradient.

The next lemma makes this localization explicit by showing that the clipped gradient agrees with the unclipped
gradient whenever $E_t$ does not occur.
This allows us to decompose PPO updates into a smooth component and a clipping-induced distortion term that is
weighted by $\mathbb{P}(E_t)$.

\begin{lemma}[PPO-Clip gradient localization]
\label{lem:clip-localization-main}
Let $L^{\mathrm{clip}}(\theta;\bar\theta_t)$ be the PPO-Clip objective
computed with a fixed behavior policy $\pi_{\bar\theta_t}$.
Policies are probabilities or densities with respect to a common action measure, with a fixed common support and positive values there.
On a parameter ball $U$ of radius $R>0$ centered at $\theta$, assume the current log-policy is continuously differentiable and
$\|\nabla_\vartheta\log\pi_\vartheta(a\mid s)\|\le G_\pi$ for almost every relevant $(s,a)$, uniformly over $\vartheta\in U$.
Let $P$ be a fixed joint law of $(s,a,\widehat A)$ with state--action marginal given by
$s\sim\nu$, $a\sim\pi_{\bar\theta_t}(\cdot\mid s)$, and $|\widehat A|\le A_{\max}$.
The advantage label and its law are held fixed when differentiating, while its sign remains inside the expectations.
Fix $\delta\in(0,1)$ and write $r_\theta=\pi_\theta/\pi_{\bar\theta_t}$.
Define the active boundary
$B_\theta:=\{\widehat A>0,r_\theta=1+\delta\}\cup\{\widehat A<0,r_\theta=1-\delta\}$.
For any $\theta$ with $P(B_\theta)=0$, the population objective is differentiable and
\begin{equation}
\label{eq:grad-localization}
\nabla_\theta L^{\mathrm{clip}}(\theta;\bar\theta_t)
=
\mathbb{E}_P
\Big[\nabla_\theta r_{\theta,\bar\theta_t}(s,a)\,\widehat A\;\mathbf{1}\{\neg E(\theta;s,a,\widehat A)\}\Big].
\end{equation}
On $B_\theta$, a sample Clarke derivative is selected by a measurable coefficient
$\alpha\in[0,1]$ multiplying $\nabla_\theta r_\theta\widehat A$.
Off $B_\theta$, take $\alpha=0$ on the strict clip-active event $E$ and $\alpha=1$ otherwise.
The expected selected sample direction is then
$\mathbb E_P[\alpha\nabla_\theta r_\theta\widehat A]$; the indicator formula need not hold
when $P(B_\theta)>0$.
\end{lemma}

\begin{proof}
Fix an update index $t$ and suppress the dependence on $t$ when it is clear.
Let the behavior policy be $\pi_{\bar\theta_t}$ and define the likelihood ratio
\[
r_\theta(s,a) ~:=~ r_{\theta,\bar\theta_t}(s,a)
~=~ \frac{\pi_\theta(a\mid s)}{\pi_{\bar\theta_t}(a\mid s)}.
\]
Condition on a sample $(s,a,\widehat A)$ and hold its advantage label fixed for differentiation.
The label remains random under the joint expectation $P$.
The PPO-Clip objective can be written as
\begin{equation}
\label{eq:Lclip-def-proof}
\begin{aligned}
L^{\mathrm{clip}}(\theta;\bar\theta_t)
&~=~
\mathbb{E}_P
\big[\ell^{\mathrm{clip}}(\theta;s,a)\big],\\
\ell^{\mathrm{clip}}(\theta;s,a)
&~:=~
\min\Big\{ r_\theta(s,a)\widehat A,\; \mathrm{clip}(r_\theta(s,a),1-\delta,1+\delta)\widehat A\Big\}.
\end{aligned}
\end{equation}

\paragraph{Step 1: Express the surrogate as a scalar piecewise function.}
For fixed $(s,a)$, $\ell^{\mathrm{clip}}(\theta;s,a)$ depends on $\theta$ only through the scalar $r_\theta(s,a)$.
Define the scalar function (parameterized by $\widehat A$)
\[
\psi_{\widehat A}(u) ~:=~ \min\Big\{u\,\widehat A,\; \mathrm{clip}(u,1-\delta,1+\delta)\,\widehat A\Big\},
\qquad u\ge 0.
\]
Then for each sample $(s,a)$,
\begin{equation}
\label{eq:psi-comp}
\ell^{\mathrm{clip}}(\theta;s,a) ~=~ \psi_{\widehat A}\!\big(r_\theta(s,a)\big).
\end{equation}

\paragraph{Step 2: Compute $\psi_{\widehat A}$ explicitly and identify the clip-active region.}
Recall
\[
\mathrm{clip}(u,1-\delta,1+\delta)
=
\begin{cases}
1-\delta, & u< 1-\delta,\\
u, & 1-\delta\le u\le 1+\delta,\\
1+\delta, & u> 1+\delta.
\end{cases}
\]

For $\widehat A=0$, the sample function and all its derivatives are zero.

\smallskip
\noindent\textbf{Case 1: $\widehat A>0$.}
Since $\widehat A>0$, the map $u\mapsto u\widehat A$ is increasing in $u$ and
\[
\psi_{\widehat A}(u) = \widehat A\cdot \min\Big\{u,\;\mathrm{clip}(u,1-\delta,1+\delta)\Big\}
= \widehat A\cdot \min\{u,\,1+\delta\}.
\]
Indeed, if $u>1+\delta$, then $\mathrm{clip}(u,\cdot)=1+\delta<u$ so the minimum is the clipped value;
otherwise the minimum is $u$.
Thus,
\begin{equation}
\label{eq:psi-Apos}
\psi_{\widehat A}(u)
=
\begin{cases}
u\widehat A, & u<1+\delta,\\
(1+\delta)\widehat A, & u>1+\delta,
\end{cases}
\quad\text{and it is non-differentiable at }u=1+\delta.
\end{equation}
Hence, for $\widehat A>0$, clipping is active exactly on $\{u>1+\delta\}$.

\smallskip
\noindent\textbf{Case 2: $\widehat A<0$.}
Since $\widehat A<0$, multiplying by $\widehat A$ reverses inequalities, and
\[
\psi_{\widehat A}(u)
= \widehat A\cdot \max\Big\{u,\;\mathrm{clip}(u,1-\delta,1+\delta)\Big\}
= \widehat A\cdot \max\{u,\,1-\delta\}.
\]
Indeed, if $u<1-\delta$, then $\mathrm{clip}(u,\cdot)=1-\delta>u$ so the maximum is the clipped value;
otherwise the maximum is $u$.
Thus,
\begin{equation}
\label{eq:psi-Aneg}
\psi_{\widehat A}(u)
=
\begin{cases}
(1-\delta)\widehat A, & u<1-\delta,\\
u\widehat A, & u>1-\delta,
\end{cases}
\quad\text{and it is non-differentiable at }u=1-\delta.
\end{equation}
Hence, for $\widehat A<0$, clipping is active exactly on $\{u<1-\delta\}$.

\smallskip
Combining the two cases gives the clip-active event
\begin{equation}
\label{eq:Et-proof}
E(\theta;s,a,\widehat A)
:=
\Big\{\widehat A>0,\; r_\theta(s,a) > 1+\delta\Big\}
\;\cup\;
\Big\{\widehat A<0,\; r_\theta(s,a) < 1-\delta\Big\}.
\end{equation}
On $E(\theta;s,a,\widehat A)$, $\psi_{\widehat A}(r_\theta)$ equals a \emph{constant} (either $(1+\delta)\widehat A$
or $(1-\delta)\widehat A$), while on its complement it equals $r_\theta\widehat A$.

\paragraph{Step 3: Pointwise gradient for differentiable points.}
Assume $\theta$ is such that $\ell^{\mathrm{clip}}(\theta;s,a)$ is differentiable for $P$-almost every sample $(s,a,\widehat A)$.
From \eqref{eq:psi-Apos}--\eqref{eq:psi-Aneg}, whenever $\psi_{\widehat A}$ is differentiable at $u=r_\theta(s,a)$ we have
\[
\psi'_{\widehat A}(u)
=
\begin{cases}
\widehat A, & \neg E(\theta;s,a,\widehat A),\\
0, & E(\theta;s,a,\widehat A),
\end{cases}
\qquad\text{(ignoring the boundary points }u\in\{1\pm\delta\}\text{)}.
\]
By the chain rule and \eqref{eq:psi-comp},
\begin{equation}
\label{eq:pointwise-grad}
\nabla_\theta \ell^{\mathrm{clip}}(\theta;s,a)
=
\psi'_{\widehat A}\!\big(r_\theta(s,a)\big)\,\nabla_\theta r_\theta(s,a)
=
\nabla_\theta r_\theta(s,a)\,\widehat A\,\mathbf{1}\{\neg E(\theta;s,a,\widehat A)\}.
\end{equation}

\paragraph{Step 4: Interchange gradient and expectation.}
The joint law $P$ does not depend on $\theta$. Integrating the uniformly bounded
score along the segment from $\theta$ to $\vartheta\in U$ gives
\[
r_\vartheta(s,a)\le e^{G_\pi R}r_\theta(s,a),
\qquad
\|\nabla_\vartheta r_\vartheta(s,a)\|
\le G_\pi e^{G_\pi R}r_\theta(s,a).
\]
Since $u\mapsto\psi_{\widehat A}(u)$ is $|\widehat A|$-Lipschitz, sample difference
quotients and selected derivative norms are bounded by
$M(s,a):=A_{\max}G_\pi e^{G_\pi R}r_\theta(s,a)$.
Common support and normalization imply $\mathbb E_P[r_\theta]=1$, so
$\mathbb E_P[M]=A_{\max}G_\pi e^{G_\pi R}<\infty$.
This provides an integrable local envelope without finite state/action sets or a
pointwise likelihood-ratio cap.
Hence, by dominated convergence (or standard differentiation-under-the-integral rules),
\[
\nabla_\theta L^{\mathrm{clip}}(\theta;\bar\theta_t)
=
\mathbb{E}_P
\big[\nabla_\theta \ell^{\mathrm{clip}}(\theta;s,a)\big].
\]
Substituting the pointwise identity \eqref{eq:pointwise-grad} yields
\[
\nabla_\theta L^{\mathrm{clip}}(\theta;\bar\theta_t)
=
\mathbb{E}_P
\Big[
\nabla_\theta r_{\theta,\bar\theta_t}(s,a)\,\widehat A\,\mathbf{1}\{\neg E(\theta;s,a,\widehat A)\}
\Big],
\]
which is exactly \eqref{eq:grad-localization}.

\paragraph{Step 5: Non-differentiable boundary and Clarke subgradients.}
The only non-differentiable points of $\psi_{\widehat A}$ occur at $u=1+\delta$ when $\widehat A>0$ and at
$u=1-\delta$ when $\widehat A<0$.
At those points the (scalar) Clarke subdifferential is
\[
\partial^{\circ}\psi_{\widehat A}(u)
=
\begin{cases}
\mathrm{conv}\{0,\widehat A\}, & \widehat A>0,\; u=1+\delta,\\
\mathrm{conv}\{0,\widehat A\}, & \widehat A<0,\; u=1-\delta,\\
\{\widehat A\}, & \neg E(\theta;s,a,\widehat A)\text{ and not on boundary},\\
\{0\}, & E(\theta;s,a,\widehat A)\text{ and not on boundary}.
\end{cases}
\]
By the chain rule for Clarke subgradients of compositions with smooth maps,
\[
\partial^{\circ}_\theta \ell^{\mathrm{clip}}(\theta;s,a)
=
\partial^{\circ}\psi_{\widehat A}\!\big(r_\theta(s,a)\big)\cdot \nabla_\theta r_\theta(s,a),
\]
so any measurable selection $\tilde g(\theta;s,a)\in \partial^{\circ}_\theta \ell^{\mathrm{clip}}(\theta;s,a)$ can be written as
\[
\begin{aligned}
\tilde g(\theta;s,a)
&=
\alpha(\theta;s,a,\widehat A)\,\nabla_\theta r_\theta(s,a)\,\widehat A,\\
\text{with }\alpha(\theta;s,a,\widehat A)&=
\begin{cases}
1, & \neg E(\theta;s,a,\widehat A)\text{ (away from boundary)},\\
0, & E(\theta;s,a,\widehat A)\text{ (away from boundary)},\\
\in[0,1], & \text{on the boundary}.
\end{cases}
\end{aligned}
\]
At a zero advantage label one may take $\alpha=1$ without changing the derivative.
If $P(B_\theta)=0$, the choice of boundary selector does not affect the expectation,
and \eqref{eq:grad-localization} follows from Step 4.
When the boundary has positive mass, the joint expectation includes the chosen selector:
\[
g_{\mathrm{sel}}(\theta)
=\mathbb E_P[\alpha(\theta;s,a,\widehat A)
\nabla_\theta r_\theta(s,a)\,\widehat A].
\]
The clipping difference is supported on $E\cup B_\theta\subseteq\{|r_\theta-1|\ge\delta\}$.
No equality between this expression and an arbitrary population Clarke subgradient is assumed.

This completes the proof.
\end{proof}

\subsection{Bounding clipping-induced bias}

Clipping changes active sample directions and can introduce population bias. The following results
control this distortion through the clipping event and the advantage and score envelopes.
They contribute a KL-dependent term to the main theorem; an active clipping event need not
produce nonzero aggregate bias.

\begin{lemma}[KL control bounds clip activation]
\label{lem:kl-bounds-Et-main}
For any $\delta\in(0,1)$ and any state $s$, let $r(a)=\pi_{\theta_t}(a\mid s)/\pi_{\bar\theta_t}(a\mid s)$.
Both policies have probability mass functions or densities with respect to a common action measure,
share a fixed support, and are positive there. The state law may be arbitrary.
For densities, the sums below are interpreted as integrals.
Then
\begin{equation}
\label{eq:ratio-tail-main}
\Pr_{a\sim\pi_{\bar\theta_t}(\cdot\mid s)}\big(|r(a)-1|\ge\delta\big)
~\le~
\frac{\sqrt{2\,\mathrm{KL}(\pi_{\bar\theta_t}(\cdot\mid s)\,\|\,\pi_{\theta_t}(\cdot\mid s))}}{\delta}.
\end{equation}
Consequently, for any fixed state law $\nu$ satisfying the population condition
$\mathbb E_{s\sim\nu}[\mathrm{KL}(\pi_{\bar\theta_t}(\cdot\mid s)\|\pi_{\theta_t}(\cdot\mid s))]\le\kappa$,
the boundary-inclusive clip-active event $E_t$ defined below satisfies
\begin{equation}
\label{eq:Et-prob-main}
\Pr(E_t) \le \frac{\sqrt{2\kappa}}{\delta}.
\end{equation}
\end{lemma}

\begin{proof}
Fix an arbitrary time index $t$ and a state $s$.  For brevity, define the conditional action distributions
\[
p(a) := \pi_{\bar\theta_t}(a\mid s),
\qquad
q(a) := \pi_{\theta_t}(a\mid s),
\qquad
r(a) := \frac{q(a)}{p(a)}.
\]
Actions with $p(a)=q(a)=0$ may be omitted; a current-policy action outside the behavior support is excluded.

\paragraph{Step 1: A tail bound via Markov's inequality and total variation.}
Let
\[
A_\delta := \{a:\ |r(a)-1|\ge\delta\}.
\]
By Markov's inequality applied to the nonnegative random variable $|r(a)-1|$ under $a\sim p$,
\begin{equation}
\label{eq:markov}
\Pr_{a\sim p}(A_\delta)
~=~
\Pr_{a\sim p}\big(|r(a)-1|\ge\delta\big)
~\le~
\frac{1}{\delta}\,\E_{a\sim p}\big[|r(a)-1|\big].
\end{equation}
We now compute the expectation explicitly:
\begin{align}
\E_{a\sim p}\big[|r(a)-1|\big]
&=
\sum_{a} p(a)\left|\frac{q(a)}{p(a)}-1\right|
=
\sum_{a} \big|q(a)-p(a)\big|
~=~ 2\,\mathrm{TV}(p,q).
\label{eq:ratio-tv}
\end{align}
Substituting \eqref{eq:ratio-tv} into \eqref{eq:markov} yields
\begin{equation}
\label{eq:tail-tv}
\Pr_{a\sim p}\big(|r(a)-1|\ge\delta\big)
~\le~
\frac{2\,\mathrm{TV}(p,q)}{\delta}.
\end{equation}

\paragraph{Step 2: Convert the total-variation bound to a KL bound.}
By Pinsker's inequality,
\begin{equation}
\label{eq:pinsker}
\mathrm{TV}(p,q) \le \sqrt{\frac{1}{2}\KL(p\|q)}.
\end{equation}
Combining \eqref{eq:tail-tv} and \eqref{eq:pinsker} gives
\[
\Pr_{a\sim p}\big(|r(a)-1|\ge\delta\big)
~\le~
\frac{2}{\delta}\sqrt{\frac{1}{2}\KL(p\|q)}
~=~
\frac{\sqrt{2\,\KL(p\|q)}}{\delta},
\]
which is exactly \eqref{eq:ratio-tail-main} after substituting back
$p(\cdot)=\pi_{\bar\theta_t}(\cdot\mid s)$ and $q(\cdot)=\pi_{\theta_t}(\cdot\mid s)$.

\paragraph{Step 3: Average over the state distribution.}
Let $d_{\bar\theta_t}:=\nu$ denote the specified state distribution under which the population KL budget is enforced, and write
\[
K(s) := \KL\big(\pi_{\bar\theta_t}(\cdot\mid s)\,\|\,\pi_{\theta_t}(\cdot\mid s)\big).
\]
Taking expectation of \eqref{eq:ratio-tail-main} over $s\sim d_{\bar\theta_t}$ yields
\begin{align}
\label{eq:avg-over-s}
\Pr_{s\sim d_{\bar\theta_t},\,a\sim\pi_{\bar\theta_t}(\cdot\mid s)}\big(|r(a)-1|\ge\delta\big)
&=
\E_{s\sim d_{\bar\theta_t}}
\Big[
\Pr_{a\sim\pi_{\bar\theta_t}(\cdot\mid s)}\big(|r(a)-1|\ge\delta\big)
\Big]\nonumber\\
&\le
\frac{\sqrt{2}}{\delta}\,
\E_{s\sim d_{\bar\theta_t}}\big[\sqrt{K(s)}\big].
\end{align}
Since $\sqrt{\cdot}$ is concave on $\R_+$, Jensen's inequality gives
\begin{equation}
\label{eq:jensen-sqrt}
\E_{s\sim d_{\bar\theta_t}}\big[\sqrt{K(s)}\big]
~\le~
\sqrt{\E_{s\sim d_{\bar\theta_t}}[K(s)]}.
\end{equation}
Using the stated population KL budget,
\[
\E_{s\sim d_{\bar\theta_t}}[K(s)]
=
\E_{s\sim d_{\bar\theta_t}}
\KL\big(\pi_{\bar\theta_t}(\cdot\mid s)\,\|\,\pi_{\theta_t}(\cdot\mid s)\big)
~\le~ \kappa,
\]
and combining with \eqref{eq:avg-over-s}--\eqref{eq:jensen-sqrt} yields
\begin{equation}
\label{eq:ratio-tail-global}
\Pr_{s\sim d_{\bar\theta_t},\,a\sim\pi_{\bar\theta_t}(\cdot\mid s)}\big(|r(a)-1|\ge\delta\big)
~\le~
\frac{\sqrt{2\kappa}}{\delta}.
\end{equation}

\paragraph{Step 4: Relating clip activation to likelihood-ratio deviation.}
Include the active boundaries in the clip-active event (for a sample $(s,a,\widehat A)$):
\[
E_t
~:=~
\Big\{\widehat A>0,\; r(a) \ge 1+\delta\Big\}
\;\cup\;
\Big\{\widehat A<0,\; r(a) \le 1-\delta\Big\}.
\]
In either case, $E_t$ implies $|r(a)-1|\ge\delta$.
Hence, regardless of the distribution of $\widehat A$,
\[
\Pr(E_t)
~\le~
\Pr\big(|r(a)-1|\ge\delta\big).
\]
Combining this with \eqref{eq:ratio-tail-global} proves \eqref{eq:Et-prob-main}:
\[
\Pr(E_t) \le \frac{\sqrt{2\kappa}}{\delta}.
\]
\end{proof}

\begin{coro}[Population clipping distortion]
\label{lem:population-clip-bias}
Under the preceding lemma, let $P$ be a fixed joint law of $(s,a,\widehat A)$
with state--action marginal $\nu(s)\pi_{\bar\theta_t}(a\mid s)$.
Suppose $|\widehat A|\le A_{\max}$ and $\|\nabla\log\pi_{\theta_t}(a\mid s)\|\le G_\pi$.
For a sample selector $\alpha$ as in Lemma~\ref{lem:clip-localization-main}, the
clipping difference from the unclipped direction under the same law is
\[
c_t:=\mathbb E_P[(1-\alpha)r\nabla\log\pi_{\theta_t}(a\mid s)\widehat A].
\]
Then, without a pointwise likelihood-ratio cap,
\begin{equation}\label{eq:population-clip-bound}
\|c_t\|^2\le 2G_\pi^2A_{\max}^2\kappa(1+\delta^{-1})^2
\le 8G_\pi^2A_{\max}^2\kappa/\delta^2.
\end{equation}
\end{coro}
\begin{proof}
\Eqref{eq:ratio-tv}, Pinsker's inequality, and Jensen's inequality give
$\mathbb E_P|r-1|\le\sqrt{2\kappa}$.
Since $0\le1-\alpha\le\mathbf1_{E_t}$, we have
\[
\begin{aligned}
\|c_t\|&\le G_\pi A_{\max}\mathbb E_P[r\mathbf1_{E_t}]\\
&\le G_\pi A_{\max}\big(P(E_t)+\mathbb E_P|r-1|\big)\\
&\le G_\pi A_{\max}\sqrt{2\kappa}(1+\delta^{-1}).
\end{aligned}
\]
Squaring and using $\delta<1$ proves the claim. The random advantage and its sign
remain inside the joint expectation throughout.
\end{proof}
The corollary compares two population surrogate directions. Relating the unclipped
direction to the current-policy return gradient, or relating empirical KL and
adaptive batch averages to population quantities, requires additional arguments.

\paragraph{Role in the convergence proof.}
Smoothness of the return objective $J$ relates each actor update to progress in the objective.
The localization and KL lemmas bound clipping distortion, while the GAE lemma controls the
current-policy advantage error. The finite-batch bridge below adds the empirical-law and
complete-trajectory discrepancies and combines all four terms. These bounds then enter the
coupled analysis of actor stationarity and critic tracking.

\subsection{Finite-batch bias under adaptive rollout reuse}
\label{appdx:batch-bridge}

The following lemma bounds the discrepancy between empirical actor and critic
directions and their current-policy population counterparts. Resampling a stored
batch leaves a statistical error because its empirical gradients can differ from
the population gradients.

\begin{lemma}[Adaptive finite-batch actor and critic bias]
\label{lem:finite-batch-bridge}
\label{lem:adaptive-batch-bridge}
Let $P_\theta$ be the law of a length-$H$ trajectory, including its prompt and rewards,
on a common finite support $\Omega$ of cardinality $S_{\rm traj}$.
The environment, reward, and prompt kernels are policy independent, and policies have
common positive action support. At outer iteration $n$, conditional on the history
$\mathcal G_n$ before collection, draw $B$ independent trajectories under the
$\mathcal G_n$-measurable behavior parameter $\bar\theta=\theta_n$.
Write $\widehat P_n$ for their empirical law and
\[
 \epsilon_B:=\min\left\{1,\frac{S_{\rm traj}}{4B}\right\}.
\]
For each deterministic inner index $t=nK+k$, the pre-update history $\mathcal F_t$
contains this batch and $(\theta_t,w_t)$, which may depend on all preceding uses of
the batch. Uniform minibatch sampling with replacement is the only additional
randomness in the two update directions.

Define the return, normalized objective, and population critic loss by
\begin{align}
 G_h(\tau)&:=\sum_{j=h}^{H-1}\gamma^{j-h}r_j+
                  \gamma^{H-h}v_{\rm term}(s_H),\notag\\
 J(\theta)&:=\frac1H\mathbb E_{P_\theta}G_0,\qquad
 \mathcal L_c(\theta,w):=\mathbb E_{P_\theta}
       \frac1H\sum_{h=0}^{H-1}(V_w(s_h)-G_h(\tau))^2.
 \label{eq:batch-population-objectives}
\end{align}
Here $\gamma\in(0,1]$, the state includes time, and the critic matches the prescribed
terminal boundary. At every update, raw GAE is recomputed on the stored paths using
the current pre-update critic, held fixed along each trajectory:
\[
 \widehat A_h(w,\tau)=\sum_{\ell=0}^{H-1-h}(\gamma\lambda)^\ell
 [r_{h+\ell}+\gamma V_w(s_{h+\ell+1})-V_w(s_{h+\ell})],
 \qquad\lambda\in[0,1].
\]
The actor holds these labels fixed when differentiating the surrogate and weights
the token-$h$ direction by $\gamma^h$. The critic uses the stored Monte Carlo targets
$G_h$, which do not depend on $w$. Both directions use pre-update parameters;
the actor and critic minibatches may be shared. Fix a measurable samplewise clipping-tie derivative rule in advance, independently of the other minibatch samples.

Assume uniformly over supported trajectories and visited parameters that
\begin{gather*}
 \|\nabla_\theta\log\pi_\theta(a\mid s)\|\le G_\pi,\qquad
 |\widehat A_h(w,\tau)|\le A_{\max},\qquad
 \frac{\pi_\theta(a\mid s)}{\pi_{\bar\theta}(a\mid s)}\le R_{\max},\\
 |V_w(s)|\le V_{\max},\qquad |G_h(\tau)|\le G_{\max},\qquad
 \|\nabla_wV_w(s)\|\le G_V.
\end{gather*}
The ratio envelope is a separate coverage premise: a uniform behavior-action
probability floor $p_{\min}>0$ permits $R_{\max}=1/p_{\min}$.
For the selected critic minimizer assume
\[
 V_{w^\star(\theta)}=V^{\pi_\theta},\qquad
 \|V_w-V_{w^\star(\theta)}\|_\infty
       \le L_V\|w-w^\star(\theta)\|.
\]
At each such inner iterate, suppose the population KL budget holds almost surely:
\begin{equation}\label{eq:batch-averaged-kl}
 \frac1H\sum_{h=0}^{H-1}\mathbb E_{s_h\sim P_{\bar\theta}}
 \mathrm{KL}\big(\pi_{\bar\theta}(\cdot\mid s_h)\,\Vert\,
                 \pi_{\theta_t}(\cdot\mid s_h)\big)\le\kappa.
\end{equation}
Let $\Delta_t=\|w_t-w^\star(\theta_t)\|^2$, and define
\begin{align}
 a&:=4G_\pi^2L_V^2\widetilde C_{\lambda,H}^2,\qquad
 Q_c:=2(V_{\max}+G_{\max})G_V,\notag\\
 F_{B,\kappa,\delta}
 &:=16R_{\max}^2G_\pi^2A_{\max}^2\epsilon_B
 +8G_\pi^2A_{\max}^2\kappa
       \big[(1+\sqrt H)^2+(1+\delta^{-1})^2\big],\notag\\
 E_c&:=8Q_c^2\epsilon_B+4Q_c^2H\kappa,
 \qquad
 \widetilde C_{\lambda,H}:=\gamma(1-\lambda)\sum_{j=0}^{H-2}(\gamma\lambda)^j.
 \label{eq:batch-bias-constants}
\end{align}
The sum is empty for $H=1$. For the actor and critic conditional biases
\[
 b_t:=\mathbb E[g_t\mid\mathcal F_t]-\nabla J(\theta_t),\qquad
 u_t:=\mathbb E[h_t\mid\mathcal F_t]
                   -\nabla_w\mathcal L_c(\theta_t,w_t),
\]
one has
\begin{align}
 \mathbb E\|b_t\|^2&\le a\,\mathbb E[\Delta_t]+F_{B,\kappa,\delta},
       \label{eq:actor-batch-bias}\\
 \mathbb E\|u_t\|^2&\le E_c.
       \label{eq:critic-batch-bias}
\end{align}
If each direction averages $m_a$ or $m_c$ independent uniform-with-replacement
samples, respectively, valid conditional variance bounds are
\begin{align}
 \mathbb E[\|g_t-\mathbb E[g_t\mid\mathcal F_t]\|^2\mid\mathcal F_t]
 &\le\frac{R_{\max}^2G_\pi^2A_{\max}^2}{m_a},\notag\\
 \mathbb E[\|h_t-\mathbb E[h_t\mid\mathcal F_t]\|^2\mid\mathcal F_t]
 &\le\frac{Q_c^2}{m_c}.
 \label{eq:batch-variance-bounds}
\end{align}
A sample may be a whole trajectory with its token-averaged contribution, or a
uniformly chosen trajectory--token pair; both give the same conditional means.
\end{lemma}

\begin{proof}
We first bound the discrepancy between the empirical and population trajectory laws, then apply this bound to the actor and critic integrands.

\paragraph{Step 1: a uniform empirical-law comparison.}
Condition on $\mathcal G_n$, and let $p_i$ and $\widehat p_i$ be the population and
empirical probabilities of the $S_{\rm traj}$ paths. With
$D_n:=\mathrm{TV}(\widehat P_n,P_{\bar\theta})$, Cauchy--Schwarz and the variance
of each empirical frequency imply
\[
 \mathbb E[D_n^2\mid\mathcal G_n]
 \le\frac{S_{\rm traj}}4\sum_i
          \mathbb E[(\widehat p_i-p_i)^2\mid\mathcal G_n]
 =\frac{S_{\rm traj}}{4B}\left(1-\sum_i p_i^2\right).
\]
Since $D_n\le1$, this proves $\mathbb E[D_n^2]\le\epsilon_B$.
For every vector-valued integrand $f$ satisfying $\sup_\tau\|f(\tau)\|\le L$,
\begin{equation}\label{eq:adaptive-tv-comparison}
 \|(\widehat P_n-P_{\bar\theta})f\|\le2L D_n
 \quad\text{pathwise}.
\end{equation}
This inequality holds simultaneously for all bounded integrands. In particular,
$f$ may depend on the realized batch through the current actor and critic. We use
the expectation bound on $D_n^2$ only after this pathwise comparison; conditioning
on the entire batch does not make its empirical error vanish.

\paragraph{Step 2: compare whole trajectory laws.}
The shared prompt and environment kernels cancel in the likelihood ratio. The KL
chain rule and \eqref{eq:batch-averaged-kl} give
\begin{equation}\label{eq:batch-path-tv}
 \mathrm{KL}(P_{\bar\theta}\Vert P_{\theta_t})
 =\sum_h\mathbb E_{P_{\bar\theta}}
       \mathrm{KL}(\pi_{\bar\theta}(\cdot\mid s_h)\Vert
                    \pi_{\theta_t}(\cdot\mid s_h))
 \le H\kappa,\qquad
 \mathrm{TV}(P_{\bar\theta},P_{\theta_t})^2\le H\kappa/2.
\end{equation}
These are pointwise comparisons at the realized parameters, so their validity is
unaffected by how those parameters were selected from the batch.

\paragraph{Step 3: identify the empirical actor direction.}
Suppress $t$ and write
$z_h:=\nabla_\theta\log\pi_\theta(a_h\mid s_h)$ and
$r_h:=\pi_\theta(a_h\mid s_h)/\pi_{\bar\theta}(a_h\mid s_h)$.
Let $\alpha_h$ be the selected sample derivative multiplier from
Lemma~\ref{lem:clip-localization-main}; in particular,
$0\le1-\alpha_h\le\mathbf1_{\{|r_h-1|\ge\delta\}}$.
Define
\begin{align*}
 \Phi^{\rm clip}(\tau)&:=\frac1H\sum_h\gamma^h\alpha_hr_hz_h\widehat A_h(w,\tau),\\
 \Phi(\tau)&:=\frac1H\sum_h\gamma^hr_hz_h\widehat A_h(w,\tau),\qquad
 \Psi(\tau):=\frac1H\sum_h\gamma^hz_h\widehat A_h(w,\tau).
\end{align*}
The random advantage and its sign stay inside $\alpha_h$. Uniform sampling from
the stored batch gives exactly $\mathbb E[g_t\mid\mathcal F_t]=\widehat P_n\Phi^{\rm clip}$.
Introduce four differences:
\begin{align}
 b_t={}&\underbrace{(\widehat P_n-P_{\bar\theta})\Phi^{\rm clip}}_{b^{\rm batch}}
 +\underbrace{P_{\bar\theta}(\Phi^{\rm clip}-\Phi)}_{b^{\rm clip}}\notag\\
 &+\underbrace{(P_{\bar\theta}\Phi-P_\theta\Psi)}_{b^{\rm traj}}
 +\underbrace{(P_\theta\Psi-\nabla J(\theta))}_{b^{\rm adv}}.
 \label{eq:batch-four-biases}
\end{align}
Since $\|\Phi^{\rm clip}\|\le R_{\max}G_\pi A_{\max}$, Step 1 gives
\begin{equation}\label{eq:batch-actor-empirical}
 \|b^{\rm batch}\|^2\le4R_{\max}^2G_\pi^2A_{\max}^2D_n^2.
\end{equation}

\paragraph{Step 4: population clipping distortion.}
For any fixed state and timestep, put $p=\pi_{\bar\theta}(\cdot\mid s)$ and
$q=\pi_\theta(\cdot\mid s)$. The boundary-inclusive clip-active event is contained
in $\{|r-1|\ge\delta\}$. Thus, even when the label is random,
\[
 \mathbb E_p[r(1-\alpha)]
 \le\mathbb E_p[r\mathbf1_{\{|r-1|\ge\delta\}}]
 \le\Pr_p(|r-1|\ge\delta)+\mathbb E_p|r-1|
 \le\sqrt{2\mathrm{KL}(p\Vert q)}(1+\delta^{-1}),
\]
where the last inequality follows from Markov and Pinsker as in
Lemma~\ref{lem:kl-bounds-Et-main}. Averaging over states and timesteps, applying
Jensen, and using $\gamma^h\le1$ yields
\begin{equation}\label{eq:batch-actor-clipping}
 \|b^{\rm clip}\|^2
 \le2G_\pi^2A_{\max}^2\kappa(1+\delta^{-1})^2.
\end{equation}
No likelihood-ratio envelope is needed for this population comparison.

\paragraph{Step 5: trajectory and action-ratio drift.}
Decompose
\[
 b^{\rm traj}=P_{\bar\theta}(\Phi-\Psi)
                   +(P_{\bar\theta}-P_\theta)\Psi.
\]
The first term is bounded by
\[
 \|P_{\bar\theta}(\Phi-\Psi)\|
 \le\frac{G_\pi A_{\max}}H\sum_h\mathbb E_{P_{\bar\theta}}|r_h-1|
 \le G_\pi A_{\max}\sqrt{2\kappa}.
\]
Here $\mathbb E_{a\sim p}|q/p-1|=2\mathrm{TV}(p,q)$ and Jensen give the final
inequality. Since $\|\Psi\|\le G_\pi A_{\max}$, Step 2 bounds the second term by
$G_\pi A_{\max}\sqrt{2H\kappa}$. Therefore
\begin{equation}\label{eq:batch-actor-trajectory}
 \|b^{\rm traj}\|^2
 \le2G_\pi^2A_{\max}^2\kappa(1+\sqrt H)^2.
\end{equation}
The full-trajectory comparison changes both the state visitation and the future
continuation law used in GAE. It does not identify action-only importance weighting
with on-policy trajectories.

\paragraph{Step 6: current-policy GAE error and the combined actor bound.}
Under $P_\theta$, the realized current critic is fixed before this population
expectation is evaluated. Corollary~\ref{cor:gae-score-cancellation} therefore applies. The
policy-gradient identity with the stated return normalization is
\[
 \nabla J(\theta)=\frac1H\mathbb E_{P_\theta}
       \sum_h\gamma^h z_h A^{\pi_\theta}(s_h,a_h).
\]
The score identity cancels the state-only baseline error before applying the
norm bound. With value realization, this gives
\begin{equation}\label{eq:batch-actor-advantage}
 \|b^{\rm adv}\|^2\le G_\pi^2L_V^2\widetilde C_{\lambda,H}^2\Delta_t.
\end{equation}
Apply $\|\sum_{i=1}^4v_i\|^2\le4\sum_{i=1}^4\|v_i\|^2$ to
\eqref{eq:batch-four-biases}, use
\eqref{eq:batch-actor-empirical}--\eqref{eq:batch-actor-advantage}, and take total
expectation. The resulting constants are exactly
\eqref{eq:batch-bias-constants}, proving \eqref{eq:actor-batch-bias}.

\paragraph{Step 7: the stored-target critic bias.}
The trajectory gradient integrand is
\[
 \Xi_w(\tau):=\frac1H\sum_h2(V_w(s_h)-G_h(\tau))\nabla_wV_w(s_h),
 \qquad \|\Xi_w(\tau)\|\le Q_c.
\]
Differentiation of the finite expectation in
\eqref{eq:batch-population-objectives} gives
$P_\theta\Xi_w=\nabla_w\mathcal L_c(\theta,w)$, while uniform minibatch sampling
gives $\mathbb E[h_t\mid\mathcal F_t]=\widehat P_n\Xi_{w_t}$.
Consequently
\[
 u_t=(\widehat P_n-P_{\bar\theta})\Xi_{w_t}
                   +(P_{\bar\theta}-P_{\theta_t})\Xi_{w_t}.
\]
Steps 1--2 and the two-term squared triangle inequality imply
\[
 \|u_t\|^2
 \le8Q_c^2D_n^2+8Q_c^2\mathrm{TV}(P_{\bar\theta},P_{\theta_t})^2
 \le8Q_c^2D_n^2+4Q_c^2H\kappa.
\]
Taking total expectation proves \eqref{eq:critic-batch-bias}. This argument permits
adaptive fitting of $w_t$ to the same stored trajectories because the comparison
in \eqref{eq:adaptive-tv-comparison} is uniform over bounded integrands.

\paragraph{Step 8: minibatch variances.}
Conditional on $\mathcal F_t$, a single actor contribution has norm at most
$R_{\max}G_\pi A_{\max}$, and a single critic contribution has norm at most $Q_c$.
Independence of the draws within each with-replacement minibatch and the variance
identity give \eqref{eq:batch-variance-bounds}. These are bounds on centered
conditional variances; total second moments also contain the squared conditional
means. The same bounds hold when the actor and critic share samples.
\end{proof}

\paragraph{Scope of the comparison.}
The finite-support constant may be large, and increasing the number of inner
updates does not eliminate the $S_{\rm traj}/B$ statistical floor. The lemma covers
raw GAE recomputed with the current pre-update critic and stored Monte Carlo critic
targets. Further advantage normalization, critic-dependent bootstrapped targets,
or an empirical KL acceptance rule require additional arguments. The expectation
bounds above hold at deterministic outer--inner indices; selecting among batches
based on their realized data requires control of that selection as well.

\subsection{A structured bridge without trajectory counting}
\label{appdx:structured-batch}

The finite-support comparison above controls every bounded function of a complete
trajectory. For a tabular critic, the actual clipped actor gradient has a smaller
class of trajectory-dependent scalar factors. Controlling that class gives the
following alternative, including when rewards have continuous distributions.

\begin{lemma}[Uniform empirical gradients for a bounded tabular critic]
\label{lem:structured-batch-bridge}
Consider a finite layered Markov decision process with disjoint nonterminal state
sets $\mathcal S_0,\ldots,\mathcal S_{H-1}$ and a finite action set $\mathcal A$.
Write
\[
 d_V:=\sum_{h=0}^{H-1}|\mathcal S_h|,
 \qquad M:=\sum_{h=0}^{H-1}|\mathcal S_h||\mathcal A|.
\]
The critic is tabular: $V_w(s)=w_s$ at nonterminal states, with
$w\in\mathcal W:=[0,W]^{d_V}$, and it has a fixed prescribed terminal boundary.
Conditionally on the history $\mathcal G_n$ before collection, collect $B\ge1$
independent complete trajectories under a $\mathcal G_n$-measurable behavior
policy $\pi_{\bar\theta}$. Let $P=P_{\bar\theta}$ and $\widehat P=\widehat P_n$.
The state, reward, and initial-state kernels are policy independent; the reward
distribution need not have finite support.
Policies have common positive action support and are continuously differentiable
in a finite-dimensional parameter. All trajectory functions and parameter
selections are measurable.

Use the stored targets $G_h$, raw GAE $\widehat A_h(w,\tau)$, and $\gamma^h/H$
actor normalization from Lemma~\ref{lem:finite-batch-bridge}. Suppose
$|G_h|\le G_{\max}$ and $|\widehat A_h(w,\tau)|\le A_{\max}$ for every
$w\in\mathcal W$ and supported trajectory, where $A_{\max}>0$.
All admissible actor parameters satisfy the score and likelihood-ratio envelopes
$G_\pi$ and $R_{\max}$ of that lemma. Fix $\delta\in(0,1)$ and select the
\emph{unclipped derivative at both clipping boundaries}. Define
\[
 C:=C_{\lambda,H}
 =1+\gamma(1-\lambda)\sum_{j=0}^{H-2}(\gamma\lambda)^j\le2,
 \qquad
 \ell_B:=\log(6M)+d_V\log\!\left(2+\frac{CWB}{A_{\max}}\right),
\]
where the sum is empty for $H=1$. For the actual clipped trajectory gradient
$\Phi^{\rm clip}_{\theta,w}$ from the preceding proof and the stored-target
critic gradient
\[
 \Xi_w(\tau):=\frac2H\sum_{h=0}^{H-1}
                    (w_{s_h}-G_h(\tau))e_{s_h},
\]
where $e_s$ is the tabular coordinate vector, there exist nonnegative,
random variables $D_{a,n},D_{c,n}$ measurable with respect to $\mathcal G_n$
and the collected batch such that, simultaneously for
every admissible actor parameter $\theta$ and $w\in\mathcal W$,
\[
 \| (\widehat P-P)\Phi^{\rm clip}_{\theta,w}\|\le D_{a,n},
 \qquad
 \| (\widehat P-P)\Xi_w\|\le D_{c,n}.
\]
Their conditional second moments obey
\begin{align}
 \mathbb E[D_{a,n}^2\mid\mathcal G_n]
 &\le e_a(B),\label{eq:structured-actor-empirical}\\
 \mathbb E[D_{c,n}^2\mid\mathcal G_n]
 &\le e_c(B),\label{eq:structured-critic-empirical}
\end{align}
where the following explicit constants are valid:
\begin{align}
 e_a(B)&:=\min\!\left\{
  4R_{\max}^2G_\pi^2A_{\max}^2,
  \frac{4R_{\max}^2G_\pi^2A_{\max}^2M^2}{H^2B}(\ell_B+3)
                         \right\},\notag\\
 e_c(B)&:=\frac{8(W^2+G_{\max}^2)}{HB}.
 \label{eq:structured-empirical-constants}
\end{align}
In particular these bounds hold when $(\theta,w)$ is selected adaptively after
arbitrarily many preceding uses of the batch, at each deterministic outer--inner
index. No covering number for the actor parameters is required.
\end{lemma}

\begin{proof}
All probabilities in this proof are conditional on $\mathcal G_n$, so the
behavior law is fixed before the independent trajectories are collected.

\paragraph{Step 1: the critic dependence of raw GAE.}
For two critics $w,v\in\mathcal W$, the reward and fixed terminal terms cancel.
Writing $q=\gamma\lambda$, direct collection of the coefficients gives
\begin{equation}\label{eq:structured-gae-lipschitz}
 \widehat A_h(w,\tau)-\widehat A_h(v,\tau)
 =-(w_{s_h}-v_{s_h})
  +\gamma(1-\lambda)\sum_{j=1}^{H-1-h}
                q^{j-1}(w_{s_{h+j}}-v_{s_{h+j}}).
\end{equation}
Thus raw GAE is affine in $w$ and
$|\widehat A_h(w,\tau)-\widehat A_h(v,\tau)|
 \le C\|w-v\|_\infty$ uniformly over trajectories and timesteps.
The maps
\[
 T_0(x)=x,\qquad T_+(x)=\max\{x,0\},\qquad
 T_-(x)=\min\{x,0\}
\]
are all $1$-Lipschitz and have magnitude at most $A_{\max}$ on
$[-A_{\max},A_{\max}]$.

\paragraph{Step 2: the exact clipped gradient in each cell.}
For a cell $j=(h,s,a)$, let
$I_j(\tau)=\mathbf1_{\{s_h=s,a_h=a\}}$,
$r_j=\pi_\theta(a\mid s)/\pi_{\bar\theta}(a\mid s)$, and
$z_j=\nabla_\theta\log\pi_\theta(a\mid s)$.
With the stated boundary rule, differentiating the PPO surrogate while holding
the advantage label fixed gives the scalar multiplier
\begin{equation}\label{eq:structured-clipping-parts}
 \alpha_j\widehat A_h(w,\tau)=
 \begin{cases}
 T_+(\widehat A_h(w,\tau)),&r_j<1-\delta,\\
 T_0(\widehat A_h(w,\tau)),&1-\delta\le r_j\le1+\delta,\\
 T_-(\widehat A_h(w,\tau)),&r_j>1+\delta.
 \end{cases}
\end{equation}
Here $T_-$ is the signed negative part, so no additional minus sign is needed.
For fixed $(\theta,\bar\theta)$, the choice $u_j\in\{0,+,-\}$ depends only
on the cell. In particular, although the label sign depends on the rest of the
trajectory and on $w$, all that dependence stays inside $T_{u_j}$.
Consequently
\[
 \Phi^{\rm clip}_{\theta,w}(\tau)
 =\frac1H\sum_{j=(h,s,a)}\gamma^h r_jz_j
                  I_j(\tau)T_{u_j}(\widehat A_h(w,\tau)).
\]
Let $f_{j,u,w}(\tau):=I_j(\tau)T_u(\widehat A_h(w,\tau))$.
The vector coefficient $\gamma^hr_jz_j$ is constant on its cell, with norm at
most $R_{\max}G_\pi$. Therefore, pathwise and simultaneously for all
admissible actor and critic parameters,
\begin{equation}\label{eq:structured-cell-reduction}
 \|(\widehat P-P)\Phi^{\rm clip}_{\theta,w}\|
 \le\frac{R_{\max}G_\pi M}{H}
          \max_{j,u}\sup_{w\in\mathcal W}
                         |(\widehat P-P)f_{j,u,w}|.
\end{equation}
The bound requires no continuity of the clipped gradient in $\theta$.
Because it holds uniformly, it also covers parameters selected using the collected data.

\paragraph{Step 3: a finite critic grid and its expected squared maximum.}
Set $\varepsilon=A_{\max}/(CB)$. A product grid in $[0,W]^{d_V}$ with mesh
at most $\varepsilon$ is an $\ell_\infty$ cover with cardinality
\[
 N\le(2+W/\varepsilon)^{d_V}
       =\left(2+\frac{CWB}{A_{\max}}\right)^{d_V}.
\]
Indeed one may use the multiples of $\varepsilon$ in each coordinate and add
the endpoint $W$; if $W=0$, the grid is the singleton cube.
For each of the at most $3MN$ grid functions, $f\in[-A_{\max},A_{\max}]$.
Hoeffding's inequality for the $B$ independent trajectories and a union bound
give, for their maximal absolute empirical deviation $Z$,
\[
 \Pr(Z\ge x\mid\mathcal G_n)
 \le\min\!\left\{1,6MN\exp\!\left(-\frac{Bx^2}{2A_{\max}^2}\right)\right\}.
\]
Integrating the tail of $Z^2$, split at
$2A_{\max}^2\log(6MN)/B$, yields
\begin{equation}\label{eq:structured-hoeffding-integral}
 \mathbb E[Z^2\mid\mathcal G_n]
 \le\frac{2A_{\max}^2}{B}\big(\log(6MN)+1\big)
 \le\frac{2A_{\max}^2}{B}(\ell_B+1).
\end{equation}
Step 1 implies that replacing any $w$ by a grid point within $\varepsilon$
changes $f_{j,u,w}$ by at most $C\varepsilon$ on each trajectory. Hence it
changes its empirical-minus-population mean by at most
$2C\varepsilon=2A_{\max}/B$. It follows that
\begin{align*}
 \mathbb E\!\left[
  \left(\max_{j,u}\sup_{w\in\mathcal W}
                 |(\widehat P-P)f_{j,u,w}|\right)^2
          \,\middle|\,\mathcal G_n\right]
 &\le 2\mathbb E[Z^2\mid\mathcal G_n]+
                                  \frac{8A_{\max}^2}{B^2}\\
 &\le\frac{4A_{\max}^2}{B}(\ell_B+3),
\end{align*}
where $B\ge1$ is used in the last inequality.
Together with \eqref{eq:structured-cell-reduction} this proves the second entry
in the minimum defining $e_a$. The first entry follows from
$\|\Phi^{\rm clip}_{\theta,w}\|\le R_{\max}G_\pi A_{\max}$ and the
triangle inequality. Explicitly, a batch-measurable upper bound is
\[
 D_{a,n}:=\min\!\left\{2R_{\max}G_\pi A_{\max},
       \frac{R_{\max}G_\pi M}{H}
                         \left(Z+\frac{2A_{\max}}B\right)\right\}.
\]
Thus no measurability assumption on an uncountable actor supremum is needed.

\paragraph{Step 4: the sparse tabular critic gradient.}
Index nonterminal states by $s\in\bigcup_h\mathcal S_h$, and let $h(s)$ be
the unique layer of $s$. Define
\[
 I_s(\tau)=\mathbf1_{\{s_{h(s)}=s\}},\quad
 X_s=(\widehat P-P)I_s,\quad
 Y_s=(\widehat P-P)(I_sG_{h(s)}).
\]
Every path visits exactly one state per layer, so
$\sum_s PI_s=H$. The $s$th coordinate of $(\widehat P-P)\Xi_w$ is
$2(w_sX_s-Y_s)/H$. Thus
\[
 \sup_{w\in\mathcal W}\|(\widehat P-P)\Xi_w\|^2
 \le\frac8{H^2}\sum_s(W^2X_s^2+Y_s^2).
\]
Define $D_{c,n}$ as the square root of this last right-hand side.
The trajectories are independent, and hence
\begin{align*}
 \sum_s\mathbb E[X_s^2\mid\mathcal G_n]
 &=\frac1B\sum_s\operatorname{Var}_P(I_s)\le\frac HB,\\
 \sum_s\mathbb E[Y_s^2\mid\mathcal G_n]
 &=\frac1B\sum_s\operatorname{Var}_P(I_sG_{h(s)})
 \le\frac1B\sum_sP(I_sG_{h(s)}^2)
 \le\frac{HG_{\max}^2}{B}.
\end{align*}
These inequalities prove \eqref{eq:structured-critic-empirical}. Dependence
between states or rewards along a trajectory is allowed; only independence
between the $B$ complete trajectories was used.
\end{proof}

\begin{corollary}[Structured replacement in the coupled bounds]
\label{cor:structured-batch-replacement}
Use the structured setting and clipping-boundary rule of
Lemma~\ref{lem:structured-batch-bridge}, and retain the population-KL,
realizability, and other population comparison assumptions of
Lemma~\ref{lem:finite-batch-bridge}. A finite trajectory support is not required.
Then the actor and critic bias conclusions of that lemma hold with the same
coefficient $a$, and with
\begin{align}
 F^{\rm str}_{B,\kappa,\delta}
 &:=4e_a(B)+8G_\pi^2A_{\max}^2\kappa
          \big[(1+\sqrt H)^2+(1+\delta^{-1})^2\big],\notag\\
 E_c^{\rm str}&:=2e_c(B)+4Q_c^2H\kappa,
 \label{eq:structured-bias-replacement}
\end{align}
where $Q_c$ may be any uniform bound on $\|\Xi_w(\tau)\|$.
For example, layered tabular structure permits
\[
 Q_c=\frac{2(W+G_{\max})}{\sqrt H};
 \qquad
 Q_c=\frac{2\max\{W,G_{\max}\}}{\sqrt H}
       \quad\text{if }0\le G_h\le G_{\max}.
\]
The synchronous coupled theorem therefore applies with $F=F^{\rm str}$ and
$E_c=E_c^{\rm str}$, whenever its remaining smoothness, critic contraction,
and stepsize premises hold. For complete-trajectory minibatches the same $Q_c$
also gives the critic variance bound $Q_c^2/m_c$; for token minibatches that
variance instead uses a bound on the individual token gradient.
\end{corollary}

\begin{proof}
In \eqref{eq:batch-four-biases}, replace only the expected squared empirical
actor term by $e_a(B)$ from \eqref{eq:structured-actor-empirical}. The other
three population comparisons are unchanged and use bounded integrands and the
KL chain rule, not the cardinality of the trajectory support. Their four-term
squared triangle inequality contributes $4e_a(B)$.
Continuous rewards cause no differentiation problem: conditional on each
finite state--action sequence, their law is policy independent. The expectations
defining $J$ and $\mathcal L_c$ are therefore finite sums of differentiable policy
path probabilities times bounded reward integrals. Their derivatives are the
same policy-gradient and critic-gradient expressions used above.
For the critic, the same decomposition gives
\[
 u_t=(\widehat P-P_{\bar\theta})\Xi_{w_t}
                      +(P_{\bar\theta}-P_{\theta_t})\Xi_{w_t}.
\]
The expected square of the first term is at most $e_c(B)$, and the square of the
second is at most $4Q_c^2\mathrm{TV}(P_{\bar\theta},P_{\theta_t})^2
\le2Q_c^2H\kappa$. The two-term squared triangle inequality proves the stated
$E_c^{\rm str}$. Each path visits $H$ distinct tabular coordinates, so
\[
 \|\Xi_w(\tau)\|^2
 =\frac4{H^2}\sum_{h=0}^{H-1}(w_{s_h}-G_h)^2.
\]
Bounding each residual by $W+G_{\max}$, or by $\max\{W,G_{\max}\}$ for
nonnegative targets, proves the displayed choices of $Q_c$ and the corresponding
complete-trajectory minibatch variance. This proof applies at deterministic
outer--inner indices, as does the batch-collection conditioning above.
\end{proof}

For example, if $M,d_V=O(H)$, $W,G_{\max},A_{\max}=O(H)$, and the score
and ratio envelopes are bounded independently of $H$, then the uncapped actor
bound is $e_a(B)=O(H^3\log(B+H)/B)$ and $e_c(B)=O(H/B)$.
These bounds depend on the dimensions of the specified gradient classes rather
than the number of complete action sequences, with all complexity terms explicit.

\newpage
\section{Open-Loop Convergence Results}
\label{appdx:open-loop}

This appendix establishes \emph{open-loop} finite-time guarantees for the actor updates in PPO-Clip.
An explicit \emph{bias decomposition} separates four effects that enter the stationarity analysis:
(i) critic-induced advantage error (controlled by critic tracking),
(ii) within-epoch trajectory drift (controlled by the population KL budget),
(iii) clipping-induced distortion, and
(iv) the finite-batch error that persists under adaptive reuse.
These bounds are ``open-loop'' because they treat critic tracking error as an input.
The closed-loop analysis combines the actor and critic inequalities to control this error.

\subsection{Open-Loop Convergence Results for the Actor}
\label{subsec:conv-results-actor}

We first present a stationarity bound for the underlying smooth RL objective $J(\theta)$ that makes the
$N$--$K$ update structure explicit. The theorem is stated at the level of inner iterates
$\{\theta_n^{(k)}\}$ and decomposes the deviation between the expected PPO update direction and the true policy
gradient into critic, drift, clipping, and finite-batch terms.

\begin{theorem}[Non-asymptotic stationarity of the actor (open-loop, explicit $N$--$K$ structure)]
\label{thm:actor-open-loop-NK}
Consider PPO-Clip with a fixed initial actor $\theta_0$, deterministic outer iterations $n=0,\ldots,N-1$, and $K$ actual inner actor updates per outer iteration, with $\theta_{n+1}=\theta_n^{(K)}$ and $\theta_n^{(0)}=\theta_n$.
Let the inner iterates satisfy
\[
\theta_n^{(k+1)}=\theta_n^{(k)}+\eta\,g_n^{(k)},\qquad k=0,\ldots,K-1,
\]
with stepsize $0<\eta\le 1/(4L_J)$.
Assume Assumptions~\ref{ass:smooth-J0} and~\ref{ass:actor-variance0} and the finite-batch sampling, coverage, raw-GAE, and population-KL premises of Lemma~\ref{lem:finite-batch-bridge}. The objective is $J(\theta)=H^{-1}\mathbb E_{P_\theta}G_0$, and the actor uses the $\gamma^h$-weighted, token-averaged sample direction defined in that lemma.
Define the (conditional) actor bias at inner step $(n,k)$ by
\[
b_{n,k}
:=\mathbb{E}[g_n^{(k)}\mid \mathcal{F}_{n,k}] - \nabla J(\theta_n^{(k)}).
\]
Then
\begin{align}
\label{eq:actor-open-loop-NK}
\frac{1}{NK}\sum_{n=0}^{N-1}\sum_{k=0}^{K-1}
\mathbb{E}\big[\|\nabla J(\theta_n^{(k)})\|^2\big]
~\le~
&\underbrace{\frac{4\big(J_{\max}-J(\theta_0)\big)}{\eta\, NK}}_{\textnormal{optimization}}
~+~
\underbrace{2L_J\,\eta\,\sigma_a^2}_{\textnormal{stochastic noise}} \notag\\
&\quad
+~
\underbrace{\frac{4}{NK}\sum_{n=0}^{N-1}\sum_{k=0}^{K-1}\mathbb{E}\|b_{n,k}\|^2}_{\textnormal{systematic bias}} .
\end{align}
Moreover, for each inner step $(n,k)$, the bias satisfies
\begin{equation}
\label{eq:actor-bias-open-loop-NK}
\mathbb{E}\|b_{n,k}\|^2
\le a\,\mathbb E[\Delta_{n,k}]+F_{B,\kappa,\delta},
\end{equation}
where $\Delta_{n,k}:=\|w_n^{(k)}-w^\star(\theta_n^{(k)})\|^2$ and
$a,F_{B,\kappa,\delta}$ are the explicit constants in
\eqref{eq:batch-bias-constants}. Advantages are recomputed from the stored trajectories
using the current pre-update critic $w_n^{(k)}$, as required by that lemma.
\end{theorem}

\paragraph{Interpretation.}
Theorem~\ref{thm:actor-open-loop-NK} averages stationarity over all \emph{inner} iterates $\{\theta_n^{(k)}\}$, i.e.,
over the total number of actor gradient steps $NK$. The bound has the standard nonconvex SGD form: an
$O(1/(\eta NK))$ optimization term, an $O(\eta)$ noise term, and an additive bias term.
The PPO-specific effects appear in \eqref{eq:actor-bias-open-loop-NK}:
critic-induced advantage error enters through the expected tracking error and the GAE
coefficient $\widetilde C_{\lambda,H}$; trajectory drift and clipping are controlled by $\kappa$
and $\delta$; and adaptive reuse contributes an explicit finite-batch term through
$\epsilon_B=\min\{1,S_{\rm traj}/(4B)\}$.
The bound is \emph{open-loop} because critic tracking remains to be controlled by its
own update recursion. The finite-batch term is not removed by taking more inner updates.

\begin{proof}[Proof of Theorem~\ref{thm:actor-open-loop-NK}]
We prove \eqref{eq:actor-open-loop-NK} first (smoothness $\Rightarrow$ stationarity), and then
\eqref{eq:actor-bias-open-loop-NK} (PPO-specific bias decomposition).

\paragraph{Part I: from smoothness to the open-loop stationarity bound.}
Let $T := NK$ and flatten the double index $(n,k)$ into a single time index
\[
t := nK + k \in \{0,1,\dots,T-1\}.
\]
Define the single-index iterates and quantities
\[
\theta_t := \theta_n^{(k)},\qquad g_t := g_n^{(k)},\qquad b_t := b_{n,k},\qquad \mathcal{F}_t := \mathcal{F}_{n,k}.
\]
Then the actor recursion becomes
\[
\theta_{t+1} = \theta_t + \eta g_t,\qquad t=0,\dots,T-1,
\]
and the averaged stationarity measure satisfies
\[
\frac{1}{T}\sum_{t=0}^{T-1}\mathbb{E}\big[\|\nabla J(\theta_t)\|^2\big]
=
\frac{1}{NK}\sum_{n=0}^{N-1}\sum_{k=0}^{K-1}
\mathbb{E}\big[\|\nabla J(\theta_n^{(k)})\|^2\big].
\]

By $L_J$-smoothness of $J$ (Assumption~\ref{ass:smooth-J0}), for any $x,y$,
\begin{equation}
\label{eq:smooth-lower-actor-open}
J(y)\;\ge\;J(x) + \langle \nabla J(x), y-x\rangle - \frac{L_J}{2}\|y-x\|^2.
\end{equation}
Apply \eqref{eq:smooth-lower-actor-open} with $x=\theta_t$ and $y=\theta_{t+1}=\theta_t+\eta g_t$:
\begin{equation}
\label{eq:one-step}
J(\theta_{t+1})
\;\ge\;
J(\theta_t)
+\eta\langle \nabla J(\theta_t), g_t\rangle
-\frac{L_J\eta^2}{2}\|g_t\|^2.
\end{equation}
Taking conditional expectation given $\mathcal{F}_t$ and using that $\theta_t$ is $\mathcal{F}_t$-measurable yields
\begin{equation}
\label{eq:cond-exp-actor-open}
\mathbb{E}\!\left[J(\theta_{t+1})\mid \mathcal{F}_t\right]
\;\ge\;
J(\theta_t)
+\eta\left\langle \nabla J(\theta_t), \mathbb{E}[g_t\mid \mathcal{F}_t]\right\rangle
-\frac{L_J\eta^2}{2}\,\mathbb{E}\!\left[\|g_t\|^2\mid \mathcal{F}_t\right].
\end{equation}
By definition of the (conditional) bias,
\[
b_t := \mathbb{E}[g_t\mid \mathcal{F}_t] - \nabla J(\theta_t)
\qquad\Longleftrightarrow\qquad
\mathbb{E}[g_t\mid \mathcal{F}_t] = \nabla J(\theta_t) + b_t,
\]
so
\[
\left\langle \nabla J(\theta_t), \mathbb{E}[g_t\mid \mathcal{F}_t]\right\rangle
=
\|\nabla J(\theta_t)\|^2 + \langle \nabla J(\theta_t), b_t\rangle.
\]
Using $2\langle u,v\rangle \ge -\|u\|^2-\|v\|^2$ with $u=\nabla J(\theta_t)$ and $v=b_t$, we get
\[
\langle \nabla J(\theta_t), b_t\rangle
\ge
-\frac12\|\nabla J(\theta_t)\|^2 - \frac12\|b_t\|^2.
\]
The conditional variance assumption gives
\[
\mathbb{E}\!\left[\|g_t\|^2\mid \mathcal{F}_t\right]
=
\mathbb{E}\!\left[\|g_t-\mathbb{E}[g_t\mid\mathcal F_t]\|^2\mid \mathcal{F}_t\right]
+\|\nabla J(\theta_t)+b_t\|^2
\le
\sigma_a^2+2\|\nabla J(\theta_t)\|^2+2\|b_t\|^2.
\]
Substituting this bound into \eqref{eq:cond-exp-actor-open} yields
\begin{equation}
\label{eq:progress}
\mathbb{E}\!\left[J(\theta_{t+1})\mid \mathcal{F}_t\right]
\;\ge\;
J(\theta_t)
+
\left(\frac{\eta}{2}-L_J\eta^2\right)\|\nabla J(\theta_t)\|^2
-\left(\frac{\eta}{2}+L_J\eta^2\right)\|b_t\|^2
-\frac{L_J\eta^2}{2}\sigma_a^2.
\end{equation}
Since $\eta\le 1/(4L_J)$, we have
$\frac{\eta}{2}-L_J\eta^2\ge \frac{\eta}{4}$ and
$\frac{\eta}{2}+L_J\eta^2\le \eta$. Rearranging \eqref{eq:progress} gives
\begin{equation}
\label{eq:per-step-grad}
\frac{\eta}{4}\|\nabla J(\theta_t)\|^2
\;\le\;
\mathbb{E}\!\left[J(\theta_{t+1})\mid \mathcal{F}_t\right] - J(\theta_t)
+\eta\|b_t\|^2
+\frac{L_J\eta^2}{2}\sigma_a^2.
\end{equation}
Taking total expectation in \eqref{eq:per-step-grad} and summing over $t=0,\dots,T-1$ yields the telescoping bound
\begin{align}
\frac{\eta}{4}\sum_{t=0}^{T-1}\mathbb{E}\|\nabla J(\theta_t)\|^2
~\le~
\mathbb{E}[J(\theta_T)] - J(\theta_0)
+\eta\sum_{t=0}^{T-1}\mathbb{E}\|b_t\|^2
+\frac{L_J\eta^2}{2}T\sigma_a^2.
\label{eq:telescope}
\end{align}
By the bounded-above assumption (Assumption~\ref{ass:smooth-J0}), $\mathbb{E}[J(\theta_T)]\le J_{\max}$.
Plugging this into \eqref{eq:telescope} gives
\[
\frac{\eta}{4}\sum_{t=0}^{T-1}\mathbb{E}\|\nabla J(\theta_t)\|^2
\le
\big(J_{\max}-J(\theta_0)\big)
+\eta\sum_{t=0}^{T-1}\mathbb{E}\|b_t\|^2
+\frac{L_J\eta^2}{2}\,T\sigma_a^2.
\]
Divide both sides by $\eta T/4$:
\[
\frac{1}{T}\sum_{t=0}^{T-1}\mathbb{E}\|\nabla J(\theta_t)\|^2
\le
\frac{4(J_{\max}-J(\theta_0))}{\eta T}
+2L_J\eta\sigma_a^2
+\frac{4}{T}\sum_{t=0}^{T-1}\mathbb{E}\|b_t\|^2.
\]
Finally, substitute $T=NK$ and restore the double index $t\leftrightarrow(n,k)$ to obtain \eqref{eq:actor-open-loop-NK}.

\paragraph{Part II: decomposition and control of the actor bias.}
Fix an inner step $(n,k)$ and condition on its pre-update history $\mathcal F_{n,k}$,
which contains the stored rollout batch and the current actor and critic.
Lemma~\ref{lem:finite-batch-bridge} identifies the conditional minibatch mean with
the empirical clipped direction, not with a population gradient. It decomposes its
bias into four terms: empirical-to-population error, population clipping distortion,
whole-trajectory policy drift, and current-policy GAE error.
The empirical total-variation bound holds uniformly over bounded integrands, so it
continues to hold when both parameters have been fitted adaptively to the same batch.
The trajectory comparison controls old-policy continuation errors in GAE in addition
to state-visitation drift. The random advantage label and its sign remain inside the
clipping selector.
Applying \eqref{eq:actor-batch-bias} at this deterministic index gives
\[
\mathbb E\|b_{n,k}\|^2
\le 4G_\pi^2L_V^2\widetilde C_{\lambda,H}^2\,\mathbb E[\Delta_{n,k}]
   +F_{B,\kappa,\delta},
\]
which gives \eqref{eq:actor-bias-open-loop-NK}.
This completes the proof.
\end{proof}

\subsection{Open-Loop Convergence for the Critic}

\begin{theorem}[Non-asymptotic tracking of the critic (open-loop, explicit $N$--$K$ structure)]
\label{thm:critic-open-loop-NK}
Consider $T=NK$ actual joint actor--critic updates, with
$\theta_{t+1}=\theta_t+\eta g_t$ and $w_{t+1}=w_t-\beta_c h_t$.
For $f_t(w)=\mathcal L_c(\theta_t,w)$, assume the convexity, $L_c$-smoothness and selected-minimizer quadratic-growth premises of Assumption~\ref{ass:critic0}, with modulus $\mu_q>0$.
Let $w_t^\star=w^\star(\theta_t)$ be measurable and $L_{w^\star}$-Lipschitz along the actor iterates.
The pre-update history $\mathcal F_t$ includes the stored batch and current parameters. Define
\[
 e_t^c:=\mathbb E[h_t\mid\mathcal F_t]-\nabla f_t(w_t),\qquad
 \xi_t:=h_t-\mathbb E[h_t\mid\mathcal F_t],
 \quad \mathbb E[\|\xi_t\|^2\mid\mathcal F_t]\le\sigma_c^2.
\]
Thus $e_t^c$ is a predictable bias, while $\xi_t$ is conditionally centered. No independence from the actor direction is required.
Set $\Delta_t=\|w_t-w_t^\star\|^2$, $\mathcal E_c=T^{-1}\sum_{t<T}\mathbb E\|e_t^c\|^2$, and suppose
\[
 0<\beta_c\le\min\left\{\frac{\mu_q}{2L_c^2},\frac1{\mu_q}\right\}.
\]
Then
\begin{align}
\frac1T\sum_{t<T}\mathbb E\Delta_t
\le{}&\underbrace{\frac{4\mathbb E\Delta_0}{\mu_q\beta_c T}}_{\text{initialization}}
+\underbrace{\frac{5\beta_c\sigma_c^2}{\mu_q}}_{\text{noise}}
+\underbrace{\frac{40\mathcal E_c}{\mu_q^2}}_{\text{sampling/target bias}}\notag\\
&+\underbrace{\frac{40L_{w^\star}^2\eta^2}{\mu_q^2\beta_c^2}
\frac1T\sum_{t<T}\mathbb E\|g_t\|^2}_{\text{actor-induced target drift}}.
\label{eq:critic-open-loop-NK}
\end{align}
For the finite-batch protocol, Lemma~\ref{lem:finite-batch-bridge} proves $\mathcal E_c\le E_c$.
\end{theorem}

\paragraph{Interpretation.}
The critic updates at each actual inner actor step and tracks the current-policy population minimizer.
A minibatch from the stored rollout batch is centered on its empirical Monte Carlo loss.
The discrepancy from the current-policy loss appears as the explicit bias floor $40\mathcal E_c/\mu_q^2$; it is not treated as martingale noise.
The target-drift term scales as $\eta^2/\beta_c^2$.

\begin{proof}[Proof of Theorem~\ref{thm:critic-open-loop-NK}]
Flatten $t=nK+k$ and write $e_t=w_t-w_t^\star$, $D_t=\mathbb E\|e_t\|^2$.
All subsequent expectations are taken with respect to the stated empirical-data filtration.

\paragraph{Step 1: one-step recursion for the tracking error.}
Define the history-measurable frozen-policy step $u_t=e_t-\beta_c\nabla f_t(w_t)$ and the target displacement $d_t=w_t^\star-w_{t+1}^\star$.
Then
\begin{equation}
 e_{t+1}=u_t-\beta_c\xi_t+(d_t-\beta_c e_t^c).
 \label{eq:e-rec}
\end{equation}
The displacement $d_t$ may depend on the current actor noise and be correlated with $\xi_t$.

\paragraph{Step 2: separate contraction, noise, and displacement.}
For $\alpha>0$, Young's inequality and conditional centering give
\begin{align}
\mathbb E[\|e_{t+1}\|^2\mid\mathcal F_t]
\le{}&(1+\alpha)(\|u_t\|^2+\beta_c^2\sigma_c^2)\notag\\
&+2(1+1/\alpha)\left(\mathbb E[\|d_t\|^2\mid\mathcal F_t]
+\beta_c^2\|e_t^c\|^2\right).
\label{eq:split}
\end{align}
Conditional centering cancels the cross term between the measurable $u_t$ and $\xi_t$, without requiring independence between the noise and the target displacement.

\paragraph{Step 3: contract the frozen-policy gradient step.}
Convexity, quadratic growth around the selected stationary minimizer, and smoothness imply
\[
 \langle\nabla f_t(w_t),e_t\rangle\ge\frac{\mu_q}{2}\|e_t\|^2,
 \qquad \|\nabla f_t(w_t)\|\le L_c\|e_t\|.
\]
Consequently
\begin{equation}
 \|u_t\|^2\le(1-\mu_q\beta_c+L_c^2\beta_c^2)\|e_t\|^2
 \le(1-\mu_q\beta_c/2)\|e_t\|^2.
 \label{eq:gd-contraction}
\end{equation}
The convex region must contain the current iterate and its selected stationary minimizer, as stated in the assumptions.

\paragraph{Step 4: control movement of the critic target.}
The minimizer Lipschitz property and the actor recursion give
\begin{equation}
 \|d_t\|^2\le L_{w^\star}^2\eta^2\|g_t\|^2.
 \label{eq:drift-bound}
\end{equation}
This holds pathwise even when the two update directions share a minibatch.

\paragraph{Step 5: choose the contraction slack.}
Take $\alpha=\mu_q\beta_c/4$. Since $0<\mu_q\beta_c\le1$,
\[
 (1+\alpha)(1-\mu_q\beta_c/2)\le1-\mu_q\beta_c/4,
 \quad 1+\alpha\le5/4,\quad 1+1/\alpha\le5/(\mu_q\beta_c).
\]
Substituting into \eqref{eq:split} and averaging yields
\begin{align}
D_{t+1}\le{}&(1-\mu_q\beta_c/4)D_t+\tfrac54\beta_c^2\sigma_c^2
+\frac{10\beta_c}{\mu_q}\mathbb E\|e_t^c\|^2\notag\\
&+\frac{10L_{w^\star}^2\eta^2}{\mu_q\beta_c}\mathbb E\|g_t\|^2.
\label{eq:delta-rec}
\end{align}

\paragraph{Step 6: telescope over actual updates.}
Sum \eqref{eq:delta-rec}, discard the nonnegative $D_T$, and divide by $\mu_q\beta_c T/4$.
This gives \eqref{eq:critic-open-loop-NK}. For reused Monte Carlo data, Lemma~\ref{lem:finite-batch-bridge} bounds the population discrepancy by $E_c$ directly from the sampling protocol.
\end{proof}

\section{Closed-Loop Convergence Results}
\label{appdx:closed-loop}

\begin{proof}[Proof of Theorem~\ref{thm:closed-loop-NK-explicit}]
We combine the actor and critic recursions, accounting explicitly for finite-batch error and trajectory mismatch.

\paragraph{Step 0: count actual updates and define averages.}
Let $T=NK$ and $t=nK+k$; both actor and critic execute every update counted here.
Define
\begin{align}
 \bar G&=T^{-1}\sum_{t<T}\mathbb E\|\nabla J(\theta_t)\|^2,
 &\bar\Delta&=T^{-1}\sum_{t<T}\mathbb E\Delta_t,\notag\\
 \bar B&=T^{-1}\sum_{t<T}\mathbb E\|b_t\|^2,
 &\bar S&=T^{-1}\sum_{t<T}\mathbb E\|g_t\|^2,
 \label{eq:closed-averages}
\end{align}
where $b_t=\mathbb E[g_t\mid\mathcal F_t]-\nabla J(\theta_t)$ and the history includes the stored batch.
The averages count only executed updates. Appendix~\ref{appdx:stopping} treats predictable stopping.

\paragraph{Step 1: open-loop actor progress.}
The smoothness argument in Theorem~\ref{thm:actor-open-loop-NK} gives
\begin{equation}
 \bar G\le A_0+4\bar B,
 \qquad A_0=\frac{4(J_{\max}-J(\theta_0))}{\eta T}+2L_J\eta\sigma_a^2.
 \label{eq:actor-open-loop-bar}
\end{equation}
Here the conditional bias uses the mean under resampling from the stored batch, which can differ from the population direction.

\paragraph{Step 2: control empirical, clipping, trajectory, and advantage errors.}
Lemma~\ref{lem:finite-batch-bridge} first compares the empirical clipped direction with its population counterpart. It then bounds clipping distortion with the random advantage sign inside the expectation, compares complete trajectory laws, and applies the on-policy GAE lemma.
It proves
\begin{equation}
 \bar B\le a\bar\Delta+F,
 \qquad a=4G_\pi^2L_V^2\widetilde C_{\lambda,H}^2.
 \label{eq:Bbar-closed}
\end{equation}
Here $F$ is exactly \eqref{eq:finite-batch-F}. The factor four comes from the four-term squared-norm bound.
In particular, the finite-batch contribution remains even with an exact critic, identical behavior/current policies, and inactive clipping.
The trajectory comparison accounts for the behavior-policy continuation law, since a single action ratio does not make multistep GAE on-policy.

\paragraph{Step 3: combine actor bias and critic tracking.}
Combining the preceding bounds yields
\begin{equation}
 \bar G\le A_0+4a\bar\Delta+4F.
 \label{eq:Gbar-in-Deltabar}
\end{equation}

\paragraph{Step 4: open-loop critic tracking with the derived target bias.}
Theorem~\ref{thm:critic-open-loop-NK} and the critic part of Lemma~\ref{lem:finite-batch-bridge} give
\begin{equation}
 \bar\Delta\le U+d\bar S,
 \quad U=\frac{4\mathbb E\Delta_0}{\mu_q\beta_c T}+\frac{5\beta_c\sigma_c^2}{\mu_q}+\frac{40E_c}{\mu_q^2},
 \quad d=\frac{40L_{w^\star}^2\eta^2}{\mu_q^2\beta_c^2}.
 \label{eq:critic-open-loop-bar}
\end{equation}
The critic uses stored Monte Carlo targets, whose batch and behavior-policy discrepancies are bounded by $E_c$.

\paragraph{Step 5: bound the update second moment.}
Conditional variance and $\|u+v\|^2\le2\|u\|^2+2\|v\|^2$ imply
\[
 \bar S\le\sigma_a^2+2\bar G+2\bar B
 \le\sigma_a^2+2A_0+10a\bar\Delta+10F.
\]
Neither population-unbiased updates nor independent actor and critic minibatches are required.

\paragraph{Step 6: solve the critic inequality.}
Substitute into \eqref{eq:critic-open-loop-bar}. With $\rho=10ad<1$,
\begin{equation}
 \bar\Delta\le\frac{U+d(\sigma_a^2+2A_0+10F)}{1-\rho}.
 \label{eq:closed-critic-solved}
\end{equation}
This is the critic bound in the theorem.

\paragraph{Step 7: conclude actor stationarity.}
Substitute \eqref{eq:closed-critic-solved} into \eqref{eq:Gbar-in-Deltabar} and collect terms:
\[
 \bar G\le\frac{(1-2ad)A_0+4F+4aU+4ad\sigma_a^2}{1-10ad}
 \le\frac{A_0+4F+4aU+4ad\sigma_a^2}{1-\rho}.
\]
The second inequality discards a nonpositive term because $A_0\ge0$.
This proves both bounds with the stated constants.
\end{proof}

\section{Rates and Rollout-Budget Consequences}
\label{appdx:rate-resources}

\begin{proof}[Proof of Corollary~\ref{cor:structured-main}]
Apply Corollary~\ref{cor:structured-batch-replacement} to the actual conditional
means. The resulting bounds are $\bar B\le a\bar\Delta+F_{\rm str}$ and
$T^{-1}\sum_{t<T}\mathbb E\|u_t\|^2\le E_{c,\rm str}$. The proof in Appendix~\ref{appdx:closed-loop} otherwise uses only these
two inequalities, bounded conditional variances, and the stated smoothness and
critic geometry. Substituting the structured residuals gives both claims.
Finite complete-path support is needed only for the empirical-law comparison
in the finite-support bridge. Bounded rewards, finite actions, and the tabular
critic suffice for the population identities and KL comparison in the structured bridge.
\end{proof}

\begin{proof}[Proof of Corollary~\ref{cor:rate-resources}]
For fixed problem constants and $1-\rho\ge1-\rho_0>0$, the two bounds are at most
constant multiples of
\begin{equation}\label{eq:rate-terms}
 \frac1{\eta T}+\eta+\frac1{\beta_c T}+\beta_c
 +\left(\frac\eta{\beta_c}\right)^2+v(B)+\kappa,
\end{equation}
where $v(B)=S_{\rm traj}/B$ for the finite-support bridge and
$v(B)=\log(eB)/B$ for a fixed structured problem.
Indeed $A_0=O((\eta T)^{-1}+\eta)$,
$U=O((\beta_c T)^{-1}+\beta_c+v(B)+\kappa)$, and
$d=O((\eta/\beta_c)^2)$; the remaining products are bounded by the same
expression when these quantities are bounded. Under the stated schedule their
orders are respectively
$T^{-2/5},T^{-3/5},T^{-3/5},T^{-2/5},T^{-2/5}$.
The prescribed batches give $v(B_T)=O(T^{-2/5})$ in either case.
Choosing $c_\beta$ below the critic restriction and then choosing $c_\eta$
sufficiently small enforces all step restrictions and $\rho\le\rho_0$ for every
$T\ge1$; if $a=0$ the coupling restriction is vacuous.
This establishes the two rates. A sufficiently large
$T=O(\varepsilon^{-5/2})$, rounded up to a multiple of fixed $K$, suffices.
Substitution in $Q=TB_T/K$ gives the claimed counts and the extra logarithm
in the structured case. These are averages of squared errors, not bounds on
$\mathbb E\|\nabla J\|$ at the same numerical tolerance.
With $m$ trajectories per minibatch, the actor and critic each process $mT$
trajectory samples, or $HmT$ time-indexed sample contributions; batch reuse
reduces fresh sampling without removing this optimization cost.
\end{proof}

\paragraph{A sufficient population-KL certificate.}
The rate corollary requires a vanishing population budget; a small empirical KL
alone does not supply it. One sufficient deterministic rule is available when
$\|\nabla\log\pi_\theta(a\mid s)\|\le G_\pi$ on every line segment joining
within-epoch actor parameters and $\|g_t\|\le G_g$ along every update.
Then $\|\theta_{n}^{(k)}-\bar\theta_n\|\le K\eta G_g$ and the fundamental theorem
of calculus gives, pointwise in state,
\[
 \mathrm{KL}(\pi_{\bar\theta_n}\Vert\pi_{\theta_n^{(k)}})
 \le G_\pi\|\theta_n^{(k)}-\bar\theta_n\|
 \le G_\pi G_g K\eta.
\]
Thus fixed $K$ and the schedule above give the sufficient budget
$\kappa_T=G_\pi G_gK\eta_T=O(T^{-3/5})$.
The score condition on these segments is an additional sufficient certificate,
not a consequence of a bound only at the iterates. The examples in
Appendices~\ref{appdx:finite-batch-witness} and~\ref{appdx:horizon-witness}
verify the stronger quadratic certificate $\kappa=D K^2\eta^2$ globally.

\begin{proposition}[A fixed-rollout-budget bound]
\label{prop:fixed-budget}
Suppose the synchronous proof supplies
$F\le f_B/B+f_\kappa\kappa$ and
$E_c\le e_B/B+e_\kappa\kappa$, with nonnegative constants, and an
actual-iterate certificate $\kappa\le D K^2\eta^2$.
Fix $\eta,\beta_c,m$ satisfying the theorem, with $\rho=10ad<1$,
and use common envelope, variance, and initialization constants over the candidate
batch sizes and epoch lengths.
For an exact fresh-rollout budget $Q=NB$ and $T=NK=KQ/B$, define
\begin{align*}
 P&=\frac{4(J_{\max}-J(\theta_0))}{\eta}
       +\frac{16a\mathbb E\Delta_0}{\mu_q\beta_c},\\
 V&=2L_J\eta\sigma_a^2+\frac{20a\beta_c\sigma_c^2}{\mu_q}
        +4ad\sigma_a^2,\\
 C_B&=4f_B+\frac{160a e_B}{\mu_q^2},\qquad
 C_\kappa=4f_\kappa+\frac{160a e_\kappa}{\mu_q^2}.
\end{align*}
Then
\begin{equation}\label{eq:fixed-budget-bound}
 \bar G\le\frac{P B/(KQ)+C_B/B+V+C_\kappa D K^2\eta^2}{1-\rho}.
\end{equation}
For positive coefficients, the continuous minimizers of the variable terms,
holding the other quantities fixed, are
\begin{equation}\label{eq:budget-relaxations}
 B_{\rm cont}(K)=\sqrt{\frac{C_BKQ}{P}},\qquad
 K_{\rm cont}(B)=\left(\frac{PB}{2C_\kappa D\eta^2Q}\right)^{1/3}.
\end{equation}
These optimize this upper bound only. Actual choices must satisfy the integer,
coverage, sampling, and stepsize requirements of the specified protocol.
\end{proposition}
\begin{proof}
Substitute the two residual bounds in
\eqref{eq:closed-loop-actor-repaired}, collect the $1/T$, constant, $1/B$, and
$\kappa$ terms, and use $T=KQ/B$ and the deterministic KL certificate.
This gives \eqref{eq:fixed-budget-bound} with the displayed coefficients.
For fixed $K$, differentiating $PB/(KQ)+C_B/B$ gives the first minimizer.
For fixed $B$, differentiating $PB/(KQ)+C_\kappa D\eta^2K^2$ gives the second.
Both functions have positive second derivatives on the positive axis.
On bounded continuous feasible intervals, truncate the minimizer to the interval;
on a discrete feasible set, evaluate the neighboring feasible values using
\eqref{eq:fixed-budget-bound}. Zero coefficients give the corresponding
monotone or constant boundary cases.
\end{proof}

For the finite-support bridge one may use
$f_B=4R_{\max}^2G_\pi^2A_{\max}^2S_{\rm traj}$ and
$e_B=2Q_c^2S_{\rm traj}$, with the drift coefficients read from
\eqref{eq:finite-batch-F}. For the structured bridge and $1\le B\le Q$,
replace $\ell_B$ by $\ell_Q$ in Corollary~\ref{cor:structured-main}. Then
\[
 f_B(Q)=\frac{16R_{\max}^2G_\pi^2A_{\max}^2M^2}{H^2}(\ell_Q+3),\qquad
 e_B=\frac{16(W^2+G_{\max}^2)}H
\]
give a valid budget-dependent majorant to which the same proposition applies.
The bound expresses a tradeoff: larger batches reduce uniform statistical
error but permit fewer updates at fixed $Q,K$; greater reuse supplies updates
but increases the certified drift. The stationary-point guarantee does not
assert that these relaxed choices maximize return or that the upper bound is tight.
If only the linear certificate $\kappa\le D_1K\eta$ is verified, replace the last
term by $C_\kappa D_1K\eta$; its fixed-$B$ continuous minimizer is
$\sqrt{PB/(C_\kappa D_1\eta Q)}$.

\section{The Classical Two-Time-Scale Limit}
\label{appdx:classical-limit}
We derive Proposition~\ref{prop:classical-connection-single-index} from the finite-batch bounds.

\begin{proof}[Proof of Proposition~\ref{prop:classical-connection-single-index}]
Fix the finite trajectory support, horizon, discount, GAE and clipping parameters, initial actor, and all regularity, coverage, and noise constants across the family of runs. The initial expected critic error is uniformly bounded, and each run executes $T=N(T)K(T)$ joint updates using batches of size $B_T\to\infty$.

\paragraph{Step 1: random-iterate expectations equal time averages.}
The output index $\hat t$ is uniform on $\{0,\ldots,T-1\}$ and independent of the run. Hence
\begin{align}
 \mathbb E\|\nabla J(\theta_{\hat t})\|^2
 &=T^{-1}\sum_{t<T}\mathbb E\|\nabla J(\theta_t)\|^2,\label{eq:rand-avg-grad}\\
 \mathbb E\Delta_{\hat t}&=T^{-1}\sum_{t<T}\mathbb E\Delta_t.\label{eq:rand-avg-delta}
\end{align}

\paragraph{Step 2: apply the finite-time bounds.}
Use $A_{0,T},U_T,d_T,F_T,E_{c,T}$ from Theorem~\ref{thm:closed-loop-NK-explicit}. For sufficiently large $T$, the diminishing positive stepsizes meet its actor and critic restrictions. The bounds are
\begin{align}
 \bar G_T&\le\frac{A_{0,T}+4F_T+4aU_T+4ad_T\sigma_a^2}{1-10ad_T},\label{eq:avg-grad-bound}\\
 \bar\Delta_T&\le\frac{U_T+d_T(\sigma_a^2+2A_{0,T}+10F_T)}{1-10ad_T}.\label{eq:avg-delta-bound}
\end{align}
Since $S_{\mathrm{traj}}$ is fixed and $B_T\to\infty$, $\epsilon_{B_T}\to0$.
Together with $\kappa_T\to0$, this proves $F_T\to0$ and $E_{c,T}\to0$. Thus both the finite-batch error and the error from stored behavior-policy targets vanish in the limit.

\paragraph{Step 3: control the coupling denominator.}
The ratio $\eta_T/\beta_{c,T}\to0$ gives $d_T\to0$ and $\rho_T=10ad_T\to0$.
Therefore $(1-\rho_T)^{-1}\le2$ eventually.

\paragraph{Step 4: every term in the critic bound vanishes.}
The initialization term tends to zero because $\beta_{c,T}T\ge\beta_{c,T}N(T)\to\infty$ and $\mathbb E\Delta_0$ is uniformly bounded.
The critic noise term vanishes with $\beta_{c,T}$; the target-bias term vanishes with $E_{c,T}$.
Thus $U_T\to0$. Also $d_T\sigma_a^2\to0$ and $d_TF_T\to0$.
For the remaining product,
\[
 d_TA_{0,T}=O\!\left(\frac{\eta_T}{\beta_{c,T}^2T}
 +\frac{\eta_T^3}{\beta_{c,T}^2}\right)\longrightarrow0
\]
under the stated schedules. Hence \eqref{eq:avg-delta-bound} tends to zero.

\paragraph{Step 5: every term in the actor bound vanishes.}
The conditions $\eta_T T\to\infty$ and $\eta_T\to0$ imply $A_{0,T}\to0$.
The preceding steps give $F_T,U_T,d_T\to0$, so \eqref{eq:avg-grad-bound} tends to zero as well.
Apply \eqref{eq:rand-avg-grad}--\eqref{eq:rand-avg-delta} to obtain the two claimed limits.
\paragraph{Conclusion.}
The result establishes stationarity and critic tracking in expectation; it does not imply policy optimality or monotonic improvement in return. Appendix~\ref{appdx:finite-batch-witness} gives a concrete nonempty family of admissible schedules.
\end{proof}

\section{Predictable Early Stopping and Population KL Control}
\label{appdx:stopping}

We extend the analysis to predictable early stopping using an implementable certificate for the population KL budget.
An empirical KL statistic alone does not provide this certificate.
There are $N$ epochs and at most $K$ update slots per epoch, so $T_{\max}=NK$.
Let $I_{n,k}\in\{0,1\}$ be measurable before the current minibatch draw, be zero after an epoch has stopped, and satisfy $I_{n,0}=1$.
Execute both updates when $I_{n,k}=1$ and leave both parameters unchanged otherwise.
Write $M=\sum_{n,k}I_{n,k}$, $R=\mathbb EM$, and $\chi=T_{\max}/R$, so $N\le R\le NK$.
At each epoch, the $B$ trajectories are conditionally i.i.d. given the pre-collection history. Subsequent minibatches are drawn uniformly with replacement from that stored batch, which is reused at most $K$ times.

\paragraph{A sufficient population certificate.}
Suppose the score bound holds on the segments from the behavior actor to the current actor and to every possible next actor produced by a minibatch update.
For every state, the fundamental theorem of calculus gives
\[
 \mathrm{KL}(\pi_{\bar\theta}(\cdot\mid s)\|\pi_\theta(\cdot\mid s))
 \le G_\pi\|\theta-\bar\theta\|.
\]
Every empirical actor direction has norm at most $G_g=R_{\max}G_\pi A_{\max}$.
Consequently the predictable gate
\begin{equation}
 G_\pi\big(\|\theta_n^{(k)}-\bar\theta_n\|+\eta G_g\big)\le\kappa
 \label{eq:kl-safe-gate}
\end{equation}
ensures that both the current actor and the actor after the next update satisfy the population budget, for every possible current minibatch.
Assume $\eta G_\pi G_g\le\kappa$, so the first update is allowed.
One may also require that the empirical KL of the current actor is below an empirical stopping threshold, using the stored batch, before executing the next step.
This additional gate is predictable and can only stop the epoch sooner.
If an executed step raises the empirical KL above the threshold, the epoch stops at the next slot; the population certificate still covers the executed step.
Rejecting a proposal after inspecting its fresh stochastic gradient changes the conditional update law and is not covered by this rule.

\begin{proposition}[Finite-time control under predictable stopping]
\label{prop:predictable-stopping}
Under the synchronous theorem's premises and the preceding stopping protocol, define the update-weighted quantities
\[
 \bar G_R=\frac{\mathbb E\sum_tI_t\|\nabla J(\theta_t)\|^2}{R},\qquad
 \bar\Delta_R=\frac{\mathbb E\sum_tI_t\Delta_t}{R}.
\]
Use the synchronous theorem's definitions with $T$ replaced by $R$ and $\epsilon_B$ replaced by $\chi\epsilon_B$ in $F$ and $E_c$ only.
The resulting $A_{0,R},F_R,E_{c,R},U_R$ satisfy
\begin{align}
 \bar G_R&\le\frac{A_{0,R}+4F_R+4aU_R+4ad\sigma_a^2}{1-10ad},\label{eq:stopped-actor}\\
 \bar\Delta_R&\le\frac{U_R+d(\sigma_a^2+2A_{0,R}+10F_R)}{1-10ad}.
 \label{eq:stopped-critic}
\end{align}
No extra factor is needed on the population drift/clipping or martingale-noise terms.
\end{proposition}

\begin{proof}
Let $Z_n=\mathrm{TV}(\widehat P_n,P_{\bar\theta_n})$.
The empirical-distribution bound holds conditional on the history before batch $n$ is drawn, so $\mathbb E Z_n^2\le\epsilon_B$.
Although stopping depends on that batch,
\[
 \mathbb E\sum_{n,k}I_{n,k}Z_n^2
 \le K\sum_n\mathbb E Z_n^2\le NK\epsilon_B.
\]
After dividing by $R$, the finite-batch term is at most $\chi\epsilon_B$.
All trajectory and clipping population bounds hold at each executed step under \eqref{eq:kl-safe-gate}; their weighted sums therefore divide by $R$ without a $\chi$ factor.
This proves the active-update versions of the batch bridge.

Since $I_t$ is predictable, multiplying the actor conditional progress inequality and critic conditional recursion by $I_t$ preserves every conditional noise cancellation.
On inactive slots, both parameters and the current-policy critic minimizer remain unchanged, so there is no progress or contraction to count.
Summing over all slots telescopes the actual endpoints and yields
\[
 \bar G_R\le A_{0,R}+4\bar B_R,\quad
 \bar B_R\le a\bar\Delta_R+F_R,\quad
 \bar S_R\le\sigma_a^2+2\bar G_R+2\bar B_R,\quad
 \bar\Delta_R\le U_R+d\bar S_R.
\]
The same scalar elimination as in Appendix~\ref{appdx:closed-loop} proves the result.
Predictability avoids conditioning the martingale law on a future acceptance event, while unchanged parameters on inactive slots contribute no critic contraction.
\end{proof}

\paragraph{Meaning of the output measure.}
The left sides are ratios of expectations, not expectations of a random ratio.
They describe the occupation distribution of executed updates across runs, which weights a run proportionally to its number of executed updates.
If instead a single run returns an index uniformly among its own $M$ executed steps, then $M\ge N$ gives the explicit guarantees
\[
 \mathbb E\!\left[\frac1M\sum_tI_t\|\nabla J(\theta_t)\|^2\right]
 \le\frac{R}{N}\bar G_R,\qquad
 \mathbb E\!\left[\frac1M\sum_tI_t\Delta_t\right]\le\frac{R}{N}\bar\Delta_R.
\]
Since $R/N\le K$, these also imply vanishing error under the classical schedules when $K$ is fixed.
If all $NK$ slots execute, then $M=T=NK$ deterministically, the occupation and ordinary uniform-output measures coincide, and $\chi=1$.

\section{An Explicit Finite-Batch Actor--Critic Example}
\label{appdx:finite-batch-witness}

We construct an analytic example with two-step trajectories, a learned tabular critic,
and repeated PPO-Clip updates on a finite rollout batch.
Population coverage and critic curvature hold along every realized iterate.
Explicit bounds account for both empirical error and old-policy sampling.
No projection, empirical KL acceptance test, or population-gradient oracle is used.

\paragraph{MDP, objective, and policy.}
Set $H=2$ and $\gamma=1$.
There is one initial state $s_0$ and two states $s_+,s_-$ at time one.
Both actions belong to $\{-1,+1\}$, the first action $a_0$ determines
$s_1=s_{a_0}$, and the second action terminates the episode.
The rewards are $r_0=0$ and
\[
 R(a_0,a_1)=r_1=\frac{1+a_0a_1}{2}\in\{0,1\}.
\]
The terminal value is zero. For a scalar parameter $\theta\in\mathbb R$, use the same policy at all
nonterminal states:
\begin{equation}\label{eq:witness-policy}
 q(\theta)=\frac12\tanh\theta,\qquad
 \pi_\theta(a\mid s)=\frac{1+a q(\theta)}2.
\end{equation}
This is a globally exploratory smooth policy: each action has probability in $[1/4,3/4]$.
Equivalently, it mixes the uniform policy and the binary softmax policy with equal weights.
The four possible trajectories form a support $\Omega$ with $|\Omega|=4$.
Under the normalized objective used in the analysis,
\begin{equation}\label{eq:witness-objective}
 J(\theta)=\frac12\mathbb E_\theta R=\frac{1+q(\theta)^2}{4},
 \qquad J'(\theta)=\frac{q(\theta)q'(\theta)}2.
\end{equation}
The objective is nonconstant. We take the fixed initial actor $\theta_0=1$.
In particular, $J'(\theta_0)>0$.

\paragraph{Critic, Monte Carlo targets, and updates.}
Use three parameters $w=(w_0,w_+,w_-)$, with $V_w(s_i)=w_i$, and initialize all three at $1/2$.
Both Monte Carlo returns in a trajectory equal $R$.
The current-policy population reference loss is
\begin{equation}\label{eq:witness-critic-loss}
 L_c(\theta,w)=\frac12\mathbb E_\theta
       \big[(w_0-R)^2+(w_{a_0}-R)^2\big].
\end{equation}
This is the time-average squared loss. Its unique minimizer is the true, unnormalized value function:
\begin{equation}\label{eq:witness-minimizer}
 w^\star(\theta)=
 \left(\frac{1+q(\theta)^2}{2},
       \frac{1+q(\theta)}2,\frac{1-q(\theta)}2\right).
\end{equation}
The actor objective is normalized by the horizon, whereas the value targets are unnormalized returns.

At the start of epoch $n$, let $\bar\theta=\theta_n^{(0)}$ and draw $B$ independent trajectories
from $\pi_{\bar\theta}$.
Keep their observed rewards and Monte Carlo targets throughout the epoch.
Execute exactly $K\ge2$ joint updates, with $N,K$ deterministic.
At each update, sample $m\ge1$ trajectory indices independently with replacement from this batch.
Use the current pre-update critic $w$ to recompute raw GAE, with $\lambda=1/2$:
\begin{equation}\label{eq:witness-gae}
 \widehat A_0=-w_0+(1-\lambda)w_{a_0}+\lambda R,
 \qquad
 \widehat A_1=R-w_{a_0}.
\end{equation}
The labels are neither normalized nor truncated and are detached when differentiating the actor surrogate.
The clipping threshold is $\delta=1/5$.
For one sampled trajectory, the actor direction is
\[
 g(\theta,w;\tau)=\frac12\sum_{h=0}^1
 \alpha_h\,\frac{\pi_\theta(a_h\mid s_h)}{\pi_{\bar\theta}(a_h\mid s_h)}
       \partial_\theta\log\pi_\theta(a_h\mid s_h)\,\widehat A_h,
\]
where $\alpha_h$ is the derivative selector of the clipped sample objective:
it is one on the active unclipped branch, zero on the inactive branch, and belongs to $[0,1]$
at the corresponding boundary, using a fixed samplewise rule (we select the unclipped branch at ties).
The actor update is $\theta^+=\theta+\eta g_t$, where $g_t$ is the average of the $m$ sampled directions.
The critic uses the average of the Monte Carlo gradients
\begin{equation}\label{eq:witness-mc-gradient}
 \psi(w;\tau)=(w_0-R)e_0+(w_{a_0}-R)e_{a_0},
 \qquad
 w^+=w-\beta_c\,\frac1m\sum_{i=1}^m\psi(w;\tau_i).
\end{equation}
Here $e_i$ denotes a coordinate vector.
Both directions use the same pre-update pair $(\theta,w)$.
Their minibatch indices may be shared, so their noises need not be independent.
After $K$ actual updates, collect the next batch under the new actor.

\paragraph{Boundedness and regularity along all iterates.}
For $0<\beta_c\le1$, each updated critic coordinate is a convex combination of its old value and
sampled targets in $[0,1]$.
For example, if $f_a$ is the fraction of minibatch trajectories with first action $a$, then
\[
 w_a^+=(1-\beta_c f_a)w_a+
          \frac{\beta_c}{m}\sum_{i:a_{0,i}=a}R_i.
\]
The same formula with $f_0=1$ applies to $w_0$.
Consequently $w_t\in[0,1]^3$ for every realized batch and minibatch sequence.
\Eqref{eq:witness-gae} then implies $|\widehat A_h|\le1$:
the target $(1-\lambda)w_{a_0}+\lambda R$ is also in $[0,1]$.

The policy bounds hold for every $\theta\in\mathbb R$:
\[
 \frac{\pi_\theta(a\mid s)}{\pi_{\bar\theta}(a\mid s)}\le3,
 \qquad
 |\partial_\theta\log\pi_\theta(a\mid s)|
 =\left|\frac{a q'(\theta)}{1+a q(\theta)}\right|\le1,
 \qquad |q'(\theta)|\le\frac12,\quad |q''(\theta)|\le1.
\]
The product admits a tighter bound than the separate ratio and score envelopes:
\[
 \left|\frac{\pi_\theta(a\mid s)}{\pi_{\bar\theta}(a\mid s)}
             \partial_\theta\log\pi_\theta(a\mid s)\right|
 =\frac{|q'(\theta)|}{1+a q(\bar\theta)}\le1.
\]
Thus $|g_t|\le1$ and $\|\psi(w_t;\tau)\|^2\le2$.
Conditioning on the full reused batch and the past, independent sampling with replacement gives
the valid conditional variance bounds
\begin{equation}\label{eq:witness-variances}
 \sigma_a^2=\frac1m,\qquad \sigma_c^2=\frac2m.
\end{equation}
These are variances about the empirical conditional means, not assertions of population unbiasedness.
Also $J\le J_{\max}=5/16$ and
$|J''|\le(|q'|^2+|q q''|)/2\le3/8$, so $L_J=1$ is valid.

Writing $p_\pm=(1\pm q)/2$, the critic Hessian is
\[
 \nabla_w^2L_c(\theta,w)=\operatorname{diag}(1,p_+,p_-).
\]
Therefore the loss is globally convex and $L_c=1$ smooth, with
quadratic-growth modulus $\mu_q=1/4$ around the unique stationary minimizer
\eqref{eq:witness-minimizer}.
The PL condition also holds with modulus $1/4$.
The time-averaged state probabilities are $(1/2,p_+/2,p_-/2)$, whose minimum is at least $1/8$.
Population coverage therefore holds uniformly over all actor iterates,
although a finite empirical batch may omit some states.
Since
\[
 \partial_\theta w^\star=(q q',q'/2,-q'/2),\qquad
 \|\partial_\theta w^\star\|^2\le\frac3{16},
\]
we may take $L_{w^\star}=\sqrt3/4$.
The value parameterization has $L_V=G_V=1$.
With the stated initialization, $\Delta_0\le9/64$.

\paragraph{A deterministic implementable KL budget.}
Let $p=\pi_\theta(+1\mid s)$ and $\bar p=\pi_{\bar\theta}(+1\mid s)$.
Using $\mathrm{KL}\le\chi^2$, $p(1-p)\ge3/16$, and $|p-\bar p|\le|\theta-\bar\theta|/4$ gives
\begin{equation}\label{eq:witness-kl}
 \mathrm{KL}(\pi_{\bar\theta}(\cdot\mid s)\|\pi_\theta(\cdot\mid s))
 \le\frac{(p-\bar p)^2}{p(1-p)}
 \le\frac{|\theta-\bar\theta|^2}{3}.
\end{equation}
The policy is state independent, so this is also the required state-averaged KL bound.
Because $|g_t|\le1$, every pre-update and post-update actor within an epoch satisfies
$|\theta-\bar\theta|\le K\eta$.
It is therefore sufficient to choose
\begin{equation}\label{eq:witness-kl-stepsize}
 K^2\eta^2/3\le\kappa.
\end{equation}
This rule uses known constants and the prescribed epoch length, and certifies population KL
for every realized run.
There is no empirical accept/reject conditioning or random stopping in this example.

\paragraph{Finite-batch and old-policy residuals.}
Let $\widehat P_n$ be the empirical distribution on $\Omega$ and let $P_{\bar\theta}$ be its
rollout law. Conditionally on the history before collecting the batch,
\[
 \mathbb E\big[\mathrm{TV}(\widehat P_n,P_{\bar\theta})^2\big]
 \le \min\left\{1,\frac{|\Omega|}{4B}\right\}
 =:\epsilon_B=\min\{1,1/B\}.
\]
Indeed, Cauchy--Schwarz bounds squared total variation by
$|\Omega|\|\widehat P_n-P_{\bar\theta}\|_2^2/4$, whose expectation is at most $|\Omega|/(4B)$.
The resulting finite-support bound is uniform over the data-dependent actor and critic iterates:
for every bounded trajectory function $f$, the inequality
$\|(\widehat P_n-P_{\bar\theta})f\|
\le2\|f\|_\infty\mathrm{TV}(\widehat P_n,P_{\bar\theta})$
holds pathwise.
Thus adaptive batch reuse is covered while retaining the empirical-to-population discrepancy.

The generic bridge permits $R_{\max}=3$, $G_\pi=L_V=G_V=A_{\max}=1$,
$V_{\max}=G_{\max}=1$ and $Q_c=4$. Direct calculation gives tighter bounds on
three components for this example.

First, the actor integrand has envelope one, so its uniform squared empirical
error is at most $4\epsilon_B=4/B$. The four-term decomposition of the actor bias
therefore contributes $16/B$. Its population drift and clipping contributions
remain $8(39+2\sqrt2)\kappa$.
Second, \eqref{eq:witness-feedback} below gives the exact current-policy GAE bias.
Since $|q'|\le1/2$ and $(e_+-e_-)^2\le2\|e\|^2$, its square is at most
$(1-\lambda)^2\Delta_t/32$. Multiplying by the four-term factor yields
$a=(1-\lambda)^2/8=1/32$ at $\lambda=1/2$.

Third, the empirical critic discrepancy admits a direct covariance bound. Write
$\delta p_{ab}=\widehat P(a_0=a,a_1=b)-P_{\bar\theta}(a_0=a,a_1=b)$ and
$\delta r=\delta p_{++}+\delta p_{--}$. Its initial-state component is $-\delta r$;
its component for first action $a$ is
$w_a\delta p_{a,-a}-(1-w_a)\delta p_{a,a}$. Uniformly over $w\in[0,1]^3$,
\[
 \|(\widehat P-P_{\bar\theta})\psi_w\|^2
 \le (\delta r)^2+\sum_{a,b}(\delta p_{ab})^2.
\]
Conditionally on the prebatch history, the expectation of this upper bound is
\[
 \frac{\operatorname{Var}(R)+1-\sum_{a,b}P_{\bar\theta}(a,b)^2}{B}
 \le\frac{1/4+3/4}{B}=\frac1B.
\]
This pathwise envelope also controls a critic selected using the batch.
For the old-policy component, regarding the population gradient as a function
of $q$, its coordinates are
\[
 (w_0-(1+q^2)/2,\ p_+w_+-p_+^2,\ p_-w_--p_-^2).
\]
Their derivatives in $q$ have absolute bounds $1/2,3/4,3/4$, respectively,
for all $w\in[0,1]^3$. Thus the squared population-gradient difference is at
most $(11/8)|q-\bar q|^2$. Pinsker's inequality for the common state-independent
Bernoulli policy gives $|q-\bar q|^2\le2\kappa$.
Combining the two critic components using a factor two yields
$E_c=2/B+11\kappa/2$.

Consequently the actual conditional biases obey
\begin{align}
 \frac1T\sum_{t<T}\mathbb E|b_t|^2
 &\le a\,\frac1T\sum_{t<T}\mathbb E\Delta_t+F,
 & a&=\frac1{32},\label{eq:witness-actor-bias}\\
 F&=\frac{16}{B}+8(39+2\sqrt2)\kappa,\notag\\
 \frac1T\sum_{t<T}\mathbb E\|e_t^c\|^2
 &\le E_c,
 &E_c&=\frac2B+\frac{11}2\kappa.\label{eq:witness-critic-bias}
\end{align}
These residuals retain the empirical and old-policy target discrepancies.
Substituting these constants into the argument in Appendix~\ref{appdx:closed-loop}
gives the stated closed-loop bounds without assuming population unbiasedness.

\paragraph{Nontrivial feedback.}
The actor objective is nonconstant, and the critic is inexact at initialization.
At an epoch start with $\bar\theta=\theta$, clipping is inactive and a population expectation of
the raw-GAE direction gives
\begin{equation}\label{eq:witness-feedback}
 \mathbb E_\theta g(\theta,w;\tau)-J'(\theta)
 =\frac{(1-\lambda)q'(\theta)}4
       \big[(w_+-w_+^\star)-(w_--w_-^\star)\big].
\end{equation}
At the stated initialization this equals $-(1-\lambda)J'(\theta_0)/2\ne0$.
Thus the learned critic changes the mean actor direction.
Conversely, the nonconstant map \eqref{eq:witness-minimizer} moves when the actor changes.
The formula is a population diagnostic of the coupling; the finite-batch discrepancies are separately
retained in \eqref{eq:witness-actor-bias}--\eqref{eq:witness-critic-bias}.
The example uses the clipped update but does not require clipping to activate at every admissible stepsize.

\paragraph{A nonempty stepsize regime and a vanishing-error family.}
The critic and actor restrictions are satisfied by
\[
 0<\beta_c\le\frac18,\qquad
 0<\eta\le\min\left\{\frac14,\frac{\beta_c}{10},
                              \frac{\sqrt{3\kappa}}K\right\}.
\]
The instance-specific bias coefficient gives
\[
 d=120\frac{\eta^2}{\beta_c^2},\qquad
 \rho=10ad=\frac{75}{2}\frac{\eta^2}{\beta_c^2}\le\frac38<1.
\]
All constants hold along actual iterates without constraining $\theta$ to a compact set.
For fixed $K\ge2$, $0<b\le1/8$, and $m\ge1$, take $T$ through multiples of $K$ and set
\begin{equation}\label{eq:witness-schedule}
 \begin{gathered}
 \beta_{c,T}=bT^{-2/5},\qquad
 \eta_T=\frac b{10}T^{-3/5},\qquad
 \kappa_T=K^2\eta_T^2/3,\\
 B_T=\lceil T^{2/5}\rceil,\qquad N(T)=T/K.
 \end{gathered}
\end{equation}
Then $\rho_T=(3/8)T^{-2/5}$ and both $F,E_c=O(T^{-2/5})$.
Corollary~\ref{cor:rate-resources} gives
\[
 \frac1T\sum_{t<T}\mathbb E|J'(\theta_t)|^2=O(T^{-2/5}),\qquad
 \frac1T\sum_{t<T}\mathbb E\Delta_t=O(T^{-2/5}),\qquad
 Q=NB=O(T^{7/5}/K).
\]
The constants are uniform in $T$ because the exploration floor is global and
the critic cube is invariant. This is a synchronous example; it does not by
itself establish the delayed or randomly stopped variants.

\paragraph{Numerical calibration of the complete bound.}
The elementary envelope for this objective is
\[
 \sup_\theta |J'(\theta)|^2
 =\max_{x\in[-1,1]}\frac{x^2(1-x^2)^2}{64}=\frac1{432}.
\]
Use $T=10^6$, $K=2$, $B=10^6$, $m=128$, $\eta=10^{-4}$, and
$\beta_c=0.02$. The actual-iterate certificate is
$\kappa=4\cdot10^{-8}/3$, and $\rho=0.0009375$.
Writing $q_0=\tfrac12\tanh 1$, the exact initial quantities are
$J_{\max}-J(1)=1/16-q_0^2/4$ and
$\Delta_0=q_0^4/4+q_0^2/2$.
Substitution in \eqref{eq:closed-loop-actor-repaired} with
\eqref{eq:witness-actor-bias}--\eqref{eq:witness-critic-bias} gives
\[
 \bar G<0.002094<1/432\approx0.002315.
\]
This analytic certificate improves on the elementary stationarity envelope; it is not a training result. It requires
$Q=5\cdot10^{11}$ fresh trajectories and $HQ=10^{12}$ transitions; the
bound is therefore not a practical sample requirement and does not establish a tight quantitative prediction.

\section{A Growing-Horizon Example with Polynomial Resources}
\label{appdx:horizon-witness}

The two-step example extends to every horizon $H\ge2$ without
introducing a separate state for each action prefix. The construction satisfies the
assumptions along every realized update sequence. Applying the structured finite-batch bound
gives polynomial resource bounds without an exponential trajectory-support factor.
These upper bounds are conservative and do not provide practical tuning recommendations.

\paragraph{An $O(H)$-state chain and a shared policy.}
There is one initial state $s_0$ and two states $s_{h,+},s_{h,-}$ at each
time $h\in\{1,\ldots,H-1\}$. Both actions lie in $\{-1,+1\}$; taking action
$a_h$ moves to $s_{h+1,a_h}$ when $h<H-1$, and the final action terminates.
Thus the state records only time and the immediately preceding action.
Set $\gamma=1$, terminal value zero, and
\[
 r_0=0,\qquad r_h=\mathbf1\{a_h=a_{h-1}\}\quad(1\le h<H).
\]
Use the same scalar policy at every state,
\[
 q(\theta)=\tfrac12\tanh\theta,\qquad
 \pi_\theta(a\mid s)=\frac{1+a q(\theta)}2,
 \qquad \theta_0=1.
\]
Under a fixed policy, the actions are independent with mean $q$.
Writing $c(\theta)=(1+q(\theta)^2)/2$, the normalized objective is
\begin{equation}
 J_H(\theta)=\frac{H-1}{H}c(\theta),\qquad
 J_H'(\theta)=\frac{H-1}{H}q(\theta)q'(\theta).
 \label{eq:horizon-objective}
\end{equation}
This objective is nonconstant uniformly in the horizon:
$\sup_\theta J_H(\theta)-J_H(0)=(H-1)/(8H)\ge1/16$ and
$J_H'(1)\ge q(1)q'(1)/2>0$.

\paragraph{A jointly learned tabular critic.}
There are $d_H=2H-1$ critic coordinates, with $V_w(s)=w_s$ and fixed
terminal value zero. Initialize
\[
 w_{0,0}=\frac{H-1}{2},\qquad
 w_{0,h,+}=w_{0,h,-}=\frac{H-h}{2}\quad(1\le h<H).
\]
For the observed Monte Carlo return $G_h=\sum_{j=h}^{H-1}r_j$, the reference
loss and its unique minimizer are
\begin{align}
 \mathcal L_c(\theta,w)
 &=\frac1H\mathbb E_\theta\sum_{h=0}^{H-1}(w_{s_h}-G_h)^2,
 \label{eq:horizon-critic-loss}\\
 w_0^\star(\theta)&=(H-1)c(\theta),\notag\\
 w_{h,x}^\star(\theta)&=\frac{1+xq(\theta)}2
                  +(H-1-h)c(\theta),\qquad x\in\{-1,+1\}.
 \label{eq:horizon-minimizer}
\end{align}
In particular this critic realizes the unnormalized true values.

At each of $N$ deterministic epochs, collect $B$ independent trajectories under
the epoch-start actor $\bar\theta$, retain their rewards and returns, and execute
exactly $K\ge2$ joint updates. Each update samples $m\ge1$ complete trajectories
independently with replacement from that batch. Recompute raw GAE with
$\lambda=1/2$ and the current pre-update critic, detach its labels in the actor
derivative, and use PPO-Clip with $\delta=1/5$. For one sampled path the directions
are
\begin{align}
 g(\theta,w;\tau)
 &=\frac1H\sum_{h=0}^{H-1}\alpha_h
   \frac{\pi_\theta(a_h\mid s_h)}{\pi_{\bar\theta}(a_h\mid s_h)}
   \partial_\theta\log\pi_\theta(a_h\mid s_h)\,
   \widehat A_h(w,\tau),\notag\\
 \psi(w;\tau)&=\frac2H\sum_{h=0}^{H-1}(w_{s_h}-G_h)e_{s_h}.
 \label{eq:horizon-directions}
\end{align}
Here $\alpha_h$ is the clipped sample derivative selector, choosing the
unclipped branch at ties. The actor and critic update from the same pre-update
pair by $\theta^+=\theta+\eta g_t$ and $w^+=w-\beta_c h_t$, where $g_t,h_t$
are the minibatch averages. Their sampled indices may be shared.

\begin{proposition}[Polynomial resources for a growing-horizon chain]
\label{prop:horizon-polynomial}
For every integer $H\ge2$, fixed integers $K\ge2,m\ge1$, and $0<b\le1/8$,
let $T=NK$ and choose
\begin{align}
 \beta_{c,T}&=bT^{-2/5},&
 \eta_T&=\frac{b}{40H^3}T^{-3/5},&
 \kappa_T&=12H^2K^2\eta_T^2,\notag\\
 B_T&=\left\lceil H^3T^{2/5}\log(eHT)\right\rceil.
 \label{eq:horizon-schedule}
\end{align}
The process above satisfies all synchronous actor--critic premises along every
realized sequence of batches and minibatches. In particular the population KL
budget is certified by the deterministic step sizes. Define
\begin{align*}
 A_H&=20H^3/b+2b/m,\\
 F_H&=101376+6b^2K^2,\\
 U_H&=H^4/(4b)+40bH^2/m+5120+20b^2K^2,\\
 P_H&=2\big[A_H+4F_H+16U_H+8H/m\big].
\end{align*}
Then the actual $T$ updates obey
\begin{equation}
 \frac1T\sum_{t<T}\mathbb E|J_H'(\theta_t)|^2\le P_HT^{-2/5},
 \qquad
 \frac1T\sum_{t<T}\mathbb E\|w_t-w^\star(\theta_t)\|^2
       \le P_HT^{-2/5}.
 \label{eq:horizon-rates}
\end{equation}
For $0<\varepsilon\le1$, both errors are at most $\varepsilon$ when
$T=K\lceil(P_H/\varepsilon)^{5/2}/K\rceil$.
Counting newly collected trajectories as $Q=NB_T$, the resource upper bounds are
\begin{equation}
 Q=O\!\left(\frac{H^3}{K}(P_H/\varepsilon)^{7/2}
                       \log(eHP_H/\varepsilon)\right),\qquad
 HQ=O\!\left(\frac{H^4}{K}(P_H/\varepsilon)^{7/2}
                       \log(eHP_H/\varepsilon)\right).
 \label{eq:horizon-resources}
\end{equation}
For fixed $b,m,K$, $P_H=O(H^4)$; hence these are respectively
$\widetilde O(H^{17}\varepsilon^{-7/2})$ collected trajectories and
$\widetilde O(H^{18}\varepsilon^{-7/2})$ environment steps.
\end{proposition}

\begin{proof}
We establish the premises along the update sequence and substitute the resulting constants into the
synchronous bound.

\paragraph{Invariant critic region and bounded directions.}
All Monte Carlo targets lie in $[0,H]$. For a coordinate $i$, let $f_i$ be
the fraction of minibatch trajectories visiting its state. Its update is
\[
 w_i^+=\left(1-\frac{2\beta_c f_i}{H}\right)w_i+
           \frac{2\beta_c}{Hm}\sum_{j:s_i\in\tau_j}G_i(\tau_j).
\]
The nonnegative weights sum to one when $0<\beta_c\le H/2$. Therefore
$[0,H]^{d_H}$ is invariant for every batch and minibatch realization, including
batches missing some states. It contains the initialization and every
$w^\star(\theta)$.
Telescoping the matched-terminal raw GAE gives
\begin{equation}
 \widehat A_h=-w_{s_h}
 +\sum_{\ell=0}^{H-1-h}2^{-\ell}r_{h+\ell}
 +\frac12\sum_{\ell=1}^{H-1-h}2^{-(\ell-1)}w_{s_{h+\ell}}.
 \label{eq:horizon-gae-telescope}
\end{equation}
The two positive sums are at most $2$ and $H$, respectively. Hence
$-H\le\widehat A_h\le H+2$, so $A_{\max}=2H$ is valid for $H\ge2$.

For every real $\theta$ the action probabilities lie in $[1/4,3/4]$,
the likelihood ratios are at most three, and
\[
 |\partial_\theta\log\pi_\theta(a\mid s)|
 =\left|\frac{a q'(\theta)}{1+a q(\theta)}\right|\le1,
 \qquad |q'|\le\tfrac12,\qquad |q''|\le1.
\]
Thus $|g_t|\le6H$. Since a trajectory visits $H$ distinct time-state
coordinates and $|w_{s_h}-G_h|\le H$, the squared norm of its critic direction
is at most $4H$. The conditional variances about the empirical means therefore
satisfy
\begin{equation}
 \sigma_a^2=36H^2/m,\qquad \sigma_c^2=4H/m.
 \label{eq:horizon-variances}
\end{equation}
These conditional means retain finite-batch bias.

\paragraph{Population coverage, critic geometry, and target motion.}
Writing $p_x=(1+xq)/2$, the time-averaged state masses are $1/H$ at $s_0$
and $p_x/H$ at $s_{h,x}$. Their minimum is at least $1/(4H)$, uniformly
over all real actor parameters. The critic Hessian is diagonal with entries
$2/H$ at the initial state and $2p_x/H$ at the other states. Consequently
\begin{equation}
 \mu_q=\frac1{2H},\qquad L_c=\frac2H,\qquad L_V=G_V=1
 \label{eq:horizon-geometry}
\end{equation}
are valid global quadratic-growth, smoothness, and value regularity constants.
The minimizer in \eqref{eq:horizon-minimizer} is stationary and unique.
Every coordinate of its derivative has absolute value at most $H/4$, since
\[
 (w_0^\star)'=(H-1)qq',\qquad
 (w_{h,x}^\star)'=xq'/2+(H-1-h)qq'.
\]
We may therefore take $L_{w^\star}^2=H^3/8$.
At initialization the coordinate errors are
\[
 w_{0,0}-w_0^\star=-(H-1)q(1)^2/2,\qquad
 w_{0,h,x}-w_{h,x}^\star=-xq(1)/2-(H-1-h)q(1)^2/2.
\]
Their absolute values are at most $H/8$, giving
$\Delta_0\le H^3/32$. Finally,
$|J_H''|\le |q'|^2+|q q''|\le3/4$, so $L_J=1$ is valid, and
$J_{\max}-J_H(1)\le1/8$.

\paragraph{A pathwise certificate for the population KL budget.}
For $p=\pi_\theta(+1\mid s)$ and $\bar p=\pi_{\bar\theta}(+1\mid s)$,
the elementary Bernoulli comparison gives
\[
 \mathrm{KL}(\pi_{\bar\theta}\Vert\pi_\theta)
 \le\frac{(p-\bar p)^2}{p(1-p)}
 \le\frac{|\theta-\bar\theta|^2}{3}.
\]
The last inequality uses $p(1-p)\ge3/16$ and
$|p-\bar p|\le|\theta-\bar\theta|/4$.
Since $|g_t|\le6H$, every pre-update and post-update actor in the epoch
satisfies $|\theta-\bar\theta|\le6HK\eta_T$. This proves the budget
$\kappa_T=12H^2K^2\eta_T^2$ in \eqref{eq:horizon-schedule} along every
realized run. The policy is state independent, so the same inequality is the
required state-averaged population KL certificate.

\paragraph{Structured finite-batch errors.}
There are $M_H=2d_H=4H-2$ time-state-action cells. Raw GAE is
$C_{1/2,H}$-Lipschitz in the critic infinity norm, with
$C_{1/2,H}=2-2^{-(H-1)}\le2$.
Apply Lemma~\ref{lem:structured-batch-bridge} with
$W=G_{\max}=H$, $A_{\max}=2H$, $R_{\max}=3$, and $G_\pi=1$.
Its empirical actor and critic errors satisfy
\begin{align}
 e_a(B)&\le\frac{2304H^2}{B}
       \big[\log(24H)+2H\log(B+2)+3\big],\notag\\
 e_c(B)&\le\frac{16H}{B}.
 \label{eq:horizon-empirical-errors}
\end{align}
These are uniform bounds over the data-dependent iterates in the critic cube;
they do not enumerate the $2^H$ possible action sequences.
For the trajectory-level critic integrand, the bound above permits
$Q_c=2\sqrt H$. The score-baseline cancellation gives
$\widetilde C_{1/2,H}=1-2^{-(H-1)}\le1$, so $a\le4$.
Corollary~\ref{cor:structured-batch-replacement} then gives
\begin{align}
 F&\le4e_a(B)+32H^2\kappa\big[(1+\sqrt H)^2+36\big]
       \le4e_a(B)+672H^3\kappa,\notag\\
 E_c&\le2e_c(B)+16H^2\kappa.
 \label{eq:horizon-residuals}
\end{align}
The old-policy continuation and stored-target discrepancies are included in
the $\kappa$ terms.

Let $L=\log(eHT)$. The chosen batch size satisfies
$B_T\ge H^3T^{2/5}L$, $\log(B_T+2)\le4L$, and
$\log(24H)\le3L$ for $H\ge2,T\ge1$. Hence the bracket in
\eqref{eq:horizon-empirical-errors} is at most $11HL$, and
\[
 e_a(B_T)\le25344T^{-2/5},\qquad
 F\le F_HT^{-2/5}.
\]
The structured concentration bound removes the complete-trajectory support factor from this step.

\paragraph{Admissible steps and substitution.}
The critic restriction is
$\beta_c\le\min\{\mu_q/(2L_c^2),1/\mu_q\}=H/16$; it and cube
invariance follow from $b\le1/8$, $H\ge2$.
The actor step also satisfies $\eta_T\le1/(4L_J)$.
With the constants above,
\begin{equation}
 d=\frac{40L_{w^\star}^2\eta_T^2}{\mu_q^2\beta_{c,T}^2}
       \le\frac1{80H}T^{-2/5},\qquad
 \rho=10ad\le\frac1{2H}T^{-2/5}\le\frac14.
 \label{eq:horizon-coupling}
\end{equation}
The open-loop terms satisfy
\begin{align*}
 A_0&=\frac{4(J_{\max}-J_H(1))}{\eta_TT}
                 +2L_J\eta_T\sigma_a^2
       \le A_HT^{-2/5},\\
 U&=\frac{4\Delta_0}{\mu_q\beta_{c,T}T}
       +\frac{5\beta_{c,T}\sigma_c^2}{\mu_q}
       +\frac{40E_c}{\mu_q^2}
       \le U_HT^{-2/5}.
\end{align*}
For the last bound, the three terms are at most
$H^4T^{-3/5}/(4b)$, $40bH^2T^{-2/5}/m$, and
$(5120+20b^2K^2)T^{-2/5}$, respectively.
The synchronous closure of Theorem~\ref{thm:closed-loop-NK-explicit}, with
$a\le4$ and $(1-\rho)^{-1}\le4/3$, gives
\begin{align*}
 \bar G_T
 &\le\tfrac43\big[A_H+4F_H+16U_H+36H/(5m)\big]T^{-2/5},\\
 \bar\Delta_T
 &\le\tfrac43\big[U_H+9H/(20m)+(2A_H+10F_H)/(80H)\big]T^{-2/5}.
\end{align*}
Both right sides are at most $P_HT^{-2/5}$.
Finally $Q=(T/K)B_T$ and every collected path has $H$ steps.
Substituting the stated $T$ proves \eqref{eq:horizon-resources}; the ceiling
changes only numerical constants because $P_H\ge K^{2/5}$.
\end{proof}

\paragraph{Nonzero actor--critic feedback.}
The critic is neither exact initially nor independent of the actor.
For a fixed on-policy pair $(\theta,w)$ at an epoch start, put
$e_h(s)=V_w(s)-V^{\pi_\theta}(s)$ and $z_h=\partial_\theta\log\pi_\theta(a_h)$.
The GAE error relative to the true critic is
\[
 -e_h(s_h)+(1-\lambda)\sum_{j=h+1}^{H-1}
                         \lambda^{j-h-1}e_j(s_j).
\]
The score has conditional mean zero, so its product with $e_h(s_h)$ averages
to zero. For $j\ge h+2$, $s_j=s_{j,a_{j-1}}$ is independent of $a_h$
under this shared state-independent policy, so those terms also average to
zero. Only $j=h+1$ survives, and
\[
 \mathbb E_\theta[z_he_{h+1}(s_{h+1})]
       =\frac{q'(\theta)}2\big[e_{h+1}(s_{h+1,+})-e_{h+1}(s_{h+1,-})\big].
\]
At an epoch start clipping is inactive. Consequently
\begin{equation}
 \mathbb E_\theta g(\theta,w;\tau)-J_H'(\theta)
 =\frac{(1-\lambda)q'(\theta)}{2H}
       \sum_{j=1}^{H-1}\big[e_j(s_{j,+})-e_j(s_{j,-})\big].
 \label{eq:horizon-feedback}
\end{equation}
At the stated initialization each bracket equals $-q(1)$, so this population
actor bias is $-J_H'(1)/4\ne0$. Conversely the selected critic target
\eqref{eq:horizon-minimizer} changes with the actor. This diagnostic concerns
the population mean; the stored-batch residuals remain in
\eqref{eq:horizon-residuals}.

The guarantee concerns expected stationarity and critic tracking for this analytic
family and does not assert global policy optimality. In particular, fixing $K$ and $m$
still gives a polynomial sample bound while $H$ grows; uniform vanishing along
a joint sequence follows, for example, when $H=o(T^{1/10})$.

\newpage
\section{Asynchronous PPO-Clip}\label{appdx:asyn}

\subsection{Asynchronous PPO-Clip: Parameter-Server vs.\ Decoupled Rollout--Training}
\label{subsec:async-ppo}

Large-model PPO-Clip systems can use \emph{asynchronous} execution to improve throughput, mitigate stragglers,
and accommodate heterogeneous hardware, such as many rollout workers paired with fewer training GPUs.
Asynchrony introduces \emph{staleness}: rollouts or gradients may be computed using parameters that lag behind
the current model. We distinguish two asynchronous paradigms and explain their relation to the analysis.

\paragraph{Baseline: synchronous PPO-Clip.}
In the synchronous setting, outer iteration $n$ collects a fresh batch under
$\pi_{\theta_n}$ (and corresponding critic $V_{w_n}$), then performs $K$ inner updates, producing
$(\theta_{n+1},w_{n+1})$. In asynchronous settings, this ``iteration barrier'' is relaxed: the data used at some
update step may have been generated by an older policy $\pi_{\theta_{n-\tau}}$ and/or evaluated by an older critic
$V_{w_{n-\tau'}}$, where $\tau,\tau'\ge 0$ denote staleness delays.

\subsubsection{Model I: Parameter-Server Asynchronous PPO-Clip}
\label{subsubsec:async-ps}

\paragraph{System architecture.}
In the \emph{parameter-server (PS)} paradigm, a central server maintains the latest actor--critic parameters
$(\theta,w)$. A set of workers interacts with environments (or prompts) to generate trajectories and compute
per-minibatch gradients, while the server applies incoming updates asynchronously.

\paragraph{Asynchronous update rule.}
At (global) update index $t$, a worker pulls a \emph{stale} snapshot
$(\theta_{t-\tau_t},w_{t-\tau'_t})$ from the server, computes a stochastic (sub)gradient of the PPO-Clip objective
and a critic gradient using its assigned retained behavior-policy batch, and then pushes the gradients back to the
server. The server applies these updates immediately:
\begin{align}
\theta_{t+1} &= \theta_t + \eta\, g_t(\theta_{t-\tau_t}, w_{t-\tau'_t}; \xi_t), \label{eq:async-ps-actor}\\
w_{t+1} &= w_t - \beta_c\, h_t(\theta_{t-\tau_t}, w_{t-\tau'_t}; \xi_t), \label{eq:async-ps-critic}
\end{align}
where $\xi_t$ denotes the random minibatch/trajectory data used by the worker and the delays
$\tau_t,\tau'_t$ capture staleness in the actor and critic snapshots, respectively.
In PPO-Clip, the worker also maintains a behavior reference $\theta_{\mathrm{old}}$ for importance ratios
$r_{\theta,\theta_{\mathrm{old}}}$ and clipping, which is the snapshot that generated the retained batch. It can precede the actor snapshot at which the current job is evaluated.

\paragraph{Effects of delayed evaluation.}
Evaluating updates at stale parameters and applying them to current parameters can introduce
systematic \emph{staleness bias} in addition to the other PPO error terms:
\begin{itemize}[leftmargin=*]
    \item \textbf{Policy-gradient mismatch:} the conditional update direction can differ from $\nabla J(\theta_t)$,
    even without clipping, because it is evaluated at $\theta_{t-\tau_t}$ using a potentially older behavior-policy batch.
    \item \textbf{Critic target drift under delay:} critic updates track $w^\star(\theta_{t-\tau_t})$ rather than
    $w^\star(\theta_t)$, adding delay-dependent terms to the tracking bound.
    \item \textbf{PPO-specific distortion under delay:} clipping and KL control compare the evaluation policy with the retained batch's behavior reference, which can differ from the current server policy.
\end{itemize}
The analysis bounds these effects through the \emph{parameter drift}
over the delay window, e.g., $\|\theta_t-\theta_{t-\tau_t}\|$ and $\|w_t-w_{t-\tau'_t}\|$.

\subsubsection{Model II: Decoupled Rollout--Training Asynchronous PPO-Clip}
\label{subsubsec:async-decoupled}

\paragraph{System architecture.}
A second paradigm decouples \emph{trajectory generation} (rollout) from \emph{training} (optimization).
A large pool of rollout workers continuously generates trajectories using periodically refreshed policy snapshots,
and writes them into a replay-like buffer or queue. Separately, one or more trainers repeatedly sample batches from
this buffer and perform PPO-Clip updates. Unlike off-policy RL, the buffer is \emph{shallow} and is refreshed
frequently; nonetheless, the trainer typically operates on \emph{stale} data generated by older snapshots.

\paragraph{Data and parameter staleness.}
Let $\theta^{\mathrm{roll}}_t$ denote the actor snapshot used by a rollout worker to generate a trajectory that later
appears in the training batch at step $t$. The trainer updates the current parameters $(\theta_t,w_t)$ using samples
generated under $\pi_{\theta^{\mathrm{roll}}_t}$, where $\theta^{\mathrm{roll}}_t$ may lag behind $\theta_t$ by a delay
$\tau_t$ induced by buffering and queueing. The trainer update can be written as
\begin{align}
\theta_{t+1} &= \theta_t + \eta\, g_t(\theta_t; \theta^{\mathrm{roll}}_t, w_t, \xi_t), \label{eq:async-dec-actor}\\
w_{t+1} &= w_t - \beta_c\, h_t(\theta_t; \theta^{\mathrm{roll}}_t, w_t, \xi_t), \label{eq:async-dec-critic}
\end{align}
where the PPO surrogate uses importance ratios
$r_{\theta_t,\theta^{\mathrm{roll}}_t}$ and clipping relative to the behavior policy $\theta^{\mathrm{roll}}_t$.
In this paradigm, \emph{the computation uses current parameters but stale data}, whereas in the PS paradigm
the computation itself may use stale parameters.

\paragraph{Practical variants.}
Two implementation choices are common:
\begin{itemize}[leftmargin=*]
    \item \textbf{Snapshot-per-epoch rollouts:} rollouts are generated using a fixed snapshot for a window of time,
    and trainers perform multiple PPO epochs on that window (multi-epoch reuse on buffered data).
    \item \textbf{Streaming rollouts:} rollouts arrive continuously with mixed staleness; trainers enforce a small KL
    constraint and/or discard overly stale samples to approximate on-policy behavior.
\end{itemize}
Both variants can be modeled through an upper bound on the delay $\tau_t$ and a bound on the induced drift
$\|\theta_t-\theta^{\mathrm{roll}}_t\|$.

\subsubsection{Comparison and modeling implications}
\label{subsubsec:async-compare}

The PS and decoupled paradigms induce different mathematical effects:
\begin{itemize}[leftmargin=*]
    \item \textbf{PS (stale compute):} gradients are evaluated at $(\theta_{t-\tau},w_{t-\tau'})$ and applied at
    $(\theta_t,w_t)$, which can create a mismatch between the update direction and the current iterate.
    \item \textbf{Decoupled (stale data):} gradients are evaluated at current parameters but on data from an older
    behavior policy $\theta^{\mathrm{roll}}$, creating potential trajectory-policy mismatch and clipping distortion.
\end{itemize}
Recent parameter drift enters the staleness bounds and interacts with actor--critic coupling.
The sufficient asynchronous conditions jointly constrain actor and critic stepsizes according to the coupling
and delay bounds; the permitted critic stepsize decreases with its delay bound.
The behavior-policy KL premise separately controls data mismatch.

\paragraph{Scope of the theorem.}
The theorem analyzes the single-gradient parameter-server model under the stated scheduling
and sampling premises. Parameter drift is bounded through stepsizes and staleness caps, while the KL budget
separately controls policy-distribution mismatch. Applying the result to a decoupled rollout--training
system requires verification of the same premises.

\begin{algorithm}[!t]
\caption{Single-gradient parameter-server PPO-Clip with reused batches}
\label{alg:ps-async-ppo}
\begin{algorithmic}[1]
\STATE \textbf{Input:} $B$ trajectories per batch, $K$ jobs per batch, $N$ batches, $T=NK$.
\STATE \textbf{Input:} $H,\gamma,\lambda,\delta$, stepsizes $\eta,\beta_c$, delay limits $\tau_a,\tau_w$.
\STATE Fix the batch-to-job assignment and server application schedule independently of current job minibatch draws.
\STATE Initialize server parameters $(\theta_0,w_0)$.
\STATE At each scheduled batch collection, draw $B$ independent trajectories under its behavior snapshot $\bar\theta_n$ and store their Monte Carlo returns.
\FOR{server jobs $t=0,\ldots,T-1$ (workers may compute concurrently)}
  \STATE Use causal actor/critic snapshots $(\theta_{t-\tau_t},w_{t-\nu_t})$ with delays within the prescribed limits.
  \STATE Use the assigned retained batch $\mathcal D_n$ and its behavior reference $\bar\theta_n$.
  \STATE Maintain the stated population KL budget between $\bar\theta_n$ and $\theta_{t-\tau_t}$.
  \STATE Draw a fresh empirical minibatch for this job; actor/critic draws may coincide.
  \STATE Compute raw GAE using the pre-job critic $w_{t-\nu_t}$, with no normalization or advantage clipping.
  \STATE Compute one PPO-Clip actor gradient $\widetilde g_t$ at $\theta_{t-\tau_t}$, using weights $\gamma^h$ and reference $\bar\theta_n$.
  \STATE Compute one Monte Carlo critic gradient $h_t$ at $w_{t-\nu_t}$ from the stored returns.
  \STATE Apply the gradient pair at the server:
  \STATE \hspace{1em}$\theta_{t+1}\gets\theta_t+\eta\widetilde g_t$, $w_{t+1}\gets w_t-\beta_c h_t$.
\ENDFOR
\STATE \textbf{Return:} server iterates (or an independently uniform server update index).
\end{algorithmic}
\end{algorithm}

\paragraph{Remark (PS staleness and sampling).}
Algorithm~\ref{alg:ps-async-ppo} describes the single-gradient read-delay model used in the theorem.
A worker job may be computed before application, provided the application schedule and delays do not inspect that job's fresh minibatch noise; equivalently, its empirical mean remains the conditional mean given the pre-application history.
Each job uses a critic snapshot taken before its gradient pair is computed. The theorem does not cover multi-step local fitting followed by accumulated parameter updates.
The population KL premise can, for example, be guaranteed by a deterministic bound on behavior-snapshot age: if that age in actor updates is at most $D_b$, then $G_\pi\eta G_gD_b\le\kappa$ suffices, with $G_g=R_{\max}G_\pi A_{\max}$ and the score bound along connecting segments.

\subsection{Asynchronous convergence guarantee}\label{subsec:async}
\label{ssec:async-ppo}

We extend the closed-loop finite-time guarantee in
Theorem~\ref{thm:closed-loop-NK-explicit} to a specified \emph{asynchronous} PPO-Clip model.
Two forms of asynchrony arise in RL training:
(i) a \emph{training--generation decoupled} setting~\citep{fu2025areal}, and
(ii) a \emph{parameter-server} setting~\citep{nair2015massivelyparallelmethodsdeep,mnih2016asynchronousmethodsdeepreinforcement}. In the training--generation decoupled setting, old-data effects can be bounded by the finite-batch trajectory comparison when its behavior-policy KL and data-assignment premises hold.
The theorem concerns the single-gradient parameter-server model distinguished from decoupled rollout--training systems in Section~\ref{subsec:async-ppo}. Its additional challenge is \emph{staleness}: rollouts, advantage estimates, and policy gradients may be computed using delayed actor and critic parameters and applied at a later server iterate.

\paragraph{Asynchronous update model (parameter server).}
Let $t=0,1,\ldots,T-1$ index \emph{server-side} actor updates. This index $t$ corresponds to the flattened index
$t=nK+k$ in the synchronous outer--inner structure.
At server update $t$, the parameter server maintains $(\theta_t,w_t)$.
A worker evaluates a job at a possibly stale snapshot $(\theta_{t-\tau_t},w_{t-\nu_t})$ using its assigned retained batch. The batch was generated by a behavior policy $\bar\theta_n$, which can precede this evaluation snapshot. Advantages are computed using the stale critic $V_{w_{t-\nu_t}}$.
The server applies the delayed actor direction $\tilde{g}_t$,
where $\tilde g_t$ is the empirical clipped-gradient estimate at $(\theta_{t-\tau_t},w_{t-\nu_t})$ using the assigned retained batch. The critic gradient is evaluated at the same pair using stored Monte Carlo targets. Both gradients use the pre-job critic, with no intervening local optimization steps. The population KL bound and fixed job count are premises of the theorem; an empirical post-update stopping check alone does not establish them.

\begin{assumption}[Bounded delay]
\label{ass:bounded-delay}
Assume the staleness is uniformly bounded: there exist finite integers $(\tau_a,\tau_w)$ such that
$0 \le \tau_t \le \min(t,\tau_a)$ and $0 \le \nu_t \le \min(t,\tau_w)$ for all $t$.
Moreover, asynchronous actor update direction satisfies
$
\mathbb{E}\big[\|\tilde{g}_t\|^2\mid \mathcal F_t\big]\le \bar\sigma_a^2, \forall t
$
for some finite constant $\bar\sigma_a^2>0$.
\end{assumption}


\begin{theorem}[Closed-loop non-asymptotic convergence of asynchronous PPO-Clip]
\label{thm:async-closed-loop-explicit}
Consider the single-gradient parameter-server model described above, under the finite-batch protocol and Assumptions~\ref{ass:smooth-J0}--\ref{ass:kl0} and Assumption~\ref{ass:bounded-delay}.
Every batch contains $B$ conditionally independent behavior-policy trajectories and is used for exactly $K$ server jobs, with a deterministic assignment of batches to the $T=NK$ jobs. The application order and delays do not inspect a job's fresh minibatch noise; its empirical mean is the conditional mean given the pre-application history.
The evaluation indices $s_t=t-\tau_t$, $r_t=t-\nu_t$ are chosen before the current minibatch draw, with $0\le\tau_t\le\min(t,\tau_a)$ and $0\le\nu_t\le\min(t,\tau_w)$.
The behavior-policy KL budget is imposed relative to the actor $\theta_{s_t}$ at which the job is evaluated.
Suppose also that $\nabla_w\mathcal L_c(\theta,w)$ is $L_\theta$-Lipschitz in $\theta$. The $L_c$ and $L_\theta$ bounds hold on all mixed current/stale parameter pairs and connecting segments used in the delay comparison; a common Cartesian region containing the iterates suffices.
Set $m=\min\{\mu_q/2,L_c\}$ and require
{\small
\[
 0<\eta\le\frac1{4L_J},\qquad
 0<\beta_c\le\frac{m}{24L_c^2(\tau_w+1)}.
\]
}
Retain $a,F,E_c,A_0$ from Theorem~\ref{thm:closed-loop-NK-explicit}, and define
{\small
\[
 U_{\mathrm{async}}=\frac{4\mathbb E\Delta_0}{m\beta_cT}
 +\frac{(6+8\tau_w)\beta_c\sigma_c^2}{m}+\frac{49E_c}{m^2},
 \qquad d_{\mathrm{async}}=\frac{12L_{w^\star}^2\eta^2}{m^2\beta_c^2}
 +\frac{98L_\theta^2\eta^2\tau_a^2}{m^2}.
\]
}
Let $c_a=1+\mathbf1\{\tau_a>0\}$ and $c_w=1+\mathbf1\{\tau_a+\tau_w>0\}$, and put
{\small
\[
\begin{gathered}
 A_{\mathrm{async}}=c_ac_w(\tau_w+1)a,
 \quad F_{\mathrm{async}}=c_aF,
 \quad q_{\mathrm{async}}=2c_a aL_{w^\star}^2\eta^2(\tau_a+\tau_w)^2
                  +2L_J^2\eta^2\tau_a^2,\\
 h_{\mathrm{async}}=A_{\mathrm{async}}d_{\mathrm{async}}+q_{\mathrm{async}},
 \qquad \rho_{\mathrm{async}}=10h_{\mathrm{async}}.
\end{gathered}
\]
}
If $\rho_{\mathrm{async}}<1$, then
{\small
\begin{align}
 \frac1T\sum\nolimits_{t<T}\mathbb E\|\nabla J(\theta_t)\|^2
 &\le\frac{A_0+4F_{\mathrm{async}}+4A_{\mathrm{async}}U_{\mathrm{async}}
                +4h_{\mathrm{async}}\sigma_a^2}{1-\rho_{\mathrm{async}}},
 \label{eq:async-actor-bound-explicit}\\
 \frac1T\sum\nolimits_{t<T}\mathbb E\Delta_t
 &\le\frac{U_{\mathrm{async}}+d_{\mathrm{async}}(\sigma_a^2+2A_0+10F_{\mathrm{async}})}
              {1-\rho_{\mathrm{async}}}.
 \label{eq:async-critic-bound-explicit}
\end{align}
}
\end{theorem}

\paragraph{Relation to the synchronous bound.}
The asynchronous bound contains the same optimization, actor and critic noise, finite-batch, trajectory-shift, and clipping terms,
along with terms for delayed critic updates and actor evaluation.
At zero delay, $A_{\mathrm{async}}=a$, $F_{\mathrm{async}}=F$, and $q_{\mathrm{async}}=0$, yielding the synchronous form.
Theorem~\ref{thm:closed-loop-NK-explicit} gives tighter critic constants in this zero-delay case.

\paragraph{Asynchronous stability conditions.}
Two distinct sufficient restrictions are required: the delay-dependent critic stepsize bound and $\rho_{\mathrm{async}}<1$.
Even with a fixed actor, delayed critic gradients can oscillate if the critic stepsize is too large.
Thus increasing the critic stepsize is useful only within its delay-dependent admissible range.
The theorem concerns one gradient pair per server job and a nonanticipative sampling schedule; accumulated local parameter deltas and data-dependent arrival selection require additional analysis.

\subsection{Auxiliary bounds for asynchronous convergence}
\label{subsec:async-auxiliary}

\subsubsection{A key auxiliary bound: server drift under bounded staleness}
\label{subsec:ps-drift-bound}

We first bound the distance between a current iterate and a stale snapshot.
The averaged bound permits snapshot indices that depend on the observed history.

\begin{lemma}[Parameter-server staleness drift]
\label{lem:ps-staleness-drift}
Suppose the server actor recursion is
\begin{equation}\label{eq:ps-server-update}
\theta_{t+1}=\theta_t+\Delta\theta_t,\qquad \Delta\theta_t=\eta\tilde g_t.
\end{equation}
For deterministic $0\le s\le t$,
\begin{equation}\label{eq:ps-theta-drift}
\mathbb E\|\theta_t-\theta_s\|^2
\le(t-s)\sum_{j=s}^{t-1}\mathbb E\|\Delta\theta_j\|^2.
\end{equation}
If $s_t\in\{\max(0,t-D),\ldots,t\}$ is a possibly random snapshot index, then
\begin{equation}\label{eq:ps-drift-average}
\frac1T\sum_{t<T}\mathbb E\|\theta_t-\theta_{s_t}\|^2
\le\eta^2D^2\bar S,\qquad
\bar S:=\frac1T\sum_{t<T}\mathbb E\|\tilde g_t\|^2.
\end{equation}
The same averaged bound holds between any two indices in the preceding window of length $D$.
If $\mathbb E\|\Delta\theta_t\|^2\le\eta^2\bar\sigma_a^2$, each such snapshot distance has expectation
at most $\eta^2D^2\bar\sigma_a^2$.
\end{lemma}

\begin{proof}
Telescoping and Cauchy--Schwarz give
$\|\theta_t-\theta_s\|^2\le(t-s)\sum_{j=s}^{t-1}\|\Delta\theta_j\|^2$.
For a random snapshot, bound this pathwise by
$D\sum_{j=\max(0,t-D)}^{t-1}\|\Delta\theta_j\|^2$.
On summing over $t<T$, each increment occurs in at most $D$ windows.
The same argument applies to the subinterval joining two indices in a window.
Taking expectations proves the assertions without an independence assumption on the indices.
\end{proof}

\subsubsection{Asynchronous control of clip activation in the parameter-server model}
\label{subsec:ps-async-clip}

The population clipping bound compares the evaluated stale actor with the worker's data-generating policy.
The stored batch may have been collected before the actor snapshot was taken.

\begin{lemma}[PS-Async KL control of the clipping region]
\label{lem:ps-async-kl-clip}
Fix behavior policy $p$, evaluated policy $q$, and $\delta\in(0,1)$, with common support as in
Lemma~\ref{lem:kl-bounds-Et-main}. If
\begin{equation}\label{eq:ps-kl-budget}
\mathbb E_{s\sim d^p}\mathrm{KL}(p(\cdot\mid s)\|q(\cdot\mid s))\le\kappa,
\end{equation}
then, for $(s,a)\sim d^p(s)p(a\mid s)$ and $r=q(a\mid s)/p(a\mid s)$,
\begin{equation}\label{eq:ps-PrEt}
\Pr(|r-1|\ge\delta)\le\frac{\sqrt{2\kappa}}{\delta}.
\end{equation}
For bounded random advantage labels, the corresponding population squared gradient distortion is at most
$2G_\pi^2A_{\max}^2\kappa(1+1/\delta)^2$.
\end{lemma}

\begin{proof}
Markov's inequality gives $\Pr(|r-1|\ge\delta)\le\mathbb E|r-1|/\delta$.
The identity $\mathbb E_{a\sim p}|r-1|=2\mathrm{TV}(p,q)$, Pinsker's inequality, and Jensen's inequality
give $\mathbb E|r-1|\le\sqrt{2\kappa}$. The distortion claim follows from Corollary~\ref{lem:population-clip-bias}, \eqref{eq:population-clip-bound}.
The clipped derivative is suppressed only on the appropriate advantage-sign branch, which is contained in the displayed
ratio event. These population bounds are combined with a separate bound for empirical batch reuse below.
\end{proof}

\subsubsection{Open-loop critic tracking with delayed gradients}

\begin{lemma}[Delayed critic tracking with predictable bias]
\label{lem:async-critic-tracking}
Let $\theta_{t+1}=\theta_t+\eta\tilde g_t$ and
\[
w_{t+1}=w_t-\beta_c h_t,\qquad
h_t=\nabla_w\mathcal L_c(\theta_{t-\tau_t},w_{t-\nu_t})+u_t+\xi_t,
\]
where $0\le\tau_t\le\min(t,\tau_a)$ and $0\le\nu_t\le\min(t,\tau_w)$.
The iterates, delays, and $u_t$ are $\mathcal F_t$-measurable, and
\[
\mathbb E[\xi_t\mid\mathcal F_t]=0,\quad
\mathbb E[\|\xi_t\|^2\mid\mathcal F_t]\le\sigma_c^2,\quad
\frac1T\sum_{t<T}\mathbb E\|u_t\|^2\le E_c.
\]
Assume the convexity, quadratic-growth, stationary selected-minimizer, and $L_c$-smoothness conditions of
Assumption~\ref{ass:critic0} hold on the relevant visited region. In addition, suppose
\[
\|\nabla_w\mathcal L_c(\theta,w)-\nabla_w\mathcal L_c(\theta',w)\|
\le L_\theta\|\theta-\theta'\|
\]
for the mixed parameter pairs used in the proof. Both the $L_c$ and $L_\theta$ bounds hold on all such mixed current/stale pairs and connecting segments; a common Cartesian parameter region suffices. Put $m:=\min\{\mu_q/2,L_c\}$ and assume
$0<\beta_c\le m/[24L_c^2(\tau_w+1)]$. Then
\begin{equation}\label{eq:critic-openloop-async}
\bar\Delta\le U_{\mathrm{async}}+d_{\mathrm{async}}\bar S,
\end{equation}
where
\begin{align}
U_{\mathrm{async}}&:=
\frac{4\mathbb E\Delta_0}{m\beta_cT}
+\frac{(6+8\tau_w)\beta_c\sigma_c^2}{m}
+\frac{49E_c}{m^2},\label{eq:async-critic-U}\\
d_{\mathrm{async}}&:=
\frac{12L_{w^\star}^2\eta^2}{m^2\beta_c^2}
+\frac{98L_\theta^2\eta^2\tau_a^2}{m^2},\label{eq:async-critic-d}
\end{align}
and $\bar\Delta=T^{-1}\sum_{t<T}\mathbb E\|w_t-w^\star(\theta_t)\|^2$.
Actor and critic update noises may be correlated.
\end{lemma}

\begin{proof}
Write $e_t=w_t-w^\star(\theta_t)$, $g_t^c=\nabla_w\mathcal L_c(\theta_t,w_t)$, and
\[
q_t^c:=\nabla_w\mathcal L_c(\theta_{t-\tau_t},w_{t-\nu_t})-g_t^c.
\]
Convexity and quadratic growth imply $\langle e_t,g_t^c\rangle\ge m\Delta_t$;
stationarity and gradient Lipschitzness imply $\|g_t^c\|\le L_c\|e_t\|$. Moreover,
\begin{equation}\label{eq:async-critic-gradient-delay}
\|q_t^c\|^2\le
2L_c^2\|w_t-w_{t-\nu_t}\|^2+
2L_\theta^2\|\theta_t-\theta_{t-\tau_t}\|^2.
\end{equation}
Put $z_t=q_t^c+u_t$. Expanding around the current target, using conditional zero mean of $\xi_t$,
and applying Young's inequality gives
\begin{align*}
\mathbb E[\|e_t-\beta_ch_t\|^2\mid\mathcal F_t]
\le {}&\left(1-\tfrac32m\beta_c+2\beta_c^2L_c^2\right)\Delta_t\\
&+\left(\tfrac{2\beta_c}{m}+2\beta_c^2\right)\|z_t\|^2+\beta_c^2\sigma_c^2\\
\le {}&(1-m\beta_c)\Delta_t+\frac{4\beta_c}{m}\|z_t\|^2+\beta_c^2\sigma_c^2.
\end{align*}
The last line uses $\beta_c\le m/(4L_c^2)$ and $m\beta_c\le1$.
To move the target to $w^\star(\theta_{t+1})$, apply the pathwise inequality
\[
\|x-y\|^2\le(1+m\beta_c/2)\|x\|^2+(1+2/(m\beta_c))\|y\|^2.
\]
Since $\|w^\star(\theta_{t+1})-w^\star(\theta_t)\|^2
\le L_{w^\star}^2\eta^2\|\tilde g_t\|^2$, this yields
\begin{align}
\mathbb E\Delta_{t+1}\le {}&
(1-m\beta_c/2)\mathbb E\Delta_t
+\frac{12\beta_c}{m}\mathbb E\|q_t^c\|^2
+\frac{12\beta_c}{m}\mathbb E\|u_t\|^2\notag\\
&+\tfrac32\beta_c^2\sigma_c^2
+\frac{3L_{w^\star}^2\eta^2}{m\beta_c}\mathbb E\|\tilde g_t\|^2.
\label{eq:async-critic-one-step}
\end{align}
Here $\|q_t^c+u_t\|^2\le2\|q_t^c\|^2+2\|u_t\|^2$ was used.
The target-shift argument does not require independent actor and critic noise.

Let $\bar W:=T^{-1}\sum_{t<T}\mathbb E\|w_t-w_{t-\nu_t}\|^2$ and
$\bar\Theta:=T^{-1}\sum_{t<T}\mathbb E\|\theta_t-\theta_{t-\tau_t}\|^2$.
Summing \eqref{eq:async-critic-one-step} and using \eqref{eq:async-critic-gradient-delay} gives
\begin{align}
\bar\Delta\le {}&\frac{2\mathbb E\Delta_0}{m\beta_cT}
+\frac{48L_c^2}{m^2}\bar W+\frac{48L_\theta^2}{m^2}\bar\Theta
+\frac{24E_c}{m^2}\notag\\
&+\frac{3\beta_c\sigma_c^2}{m}
+\frac{6L_{w^\star}^2\eta^2}{m^2\beta_c^2}\bar S.
\label{eq:async-critic-pre-absorption}
\end{align}
Set $\bar H_c:=T^{-1}\sum_{t<T}\mathbb E\|h_t\|^2$.
The window argument in Lemma~\ref{lem:ps-staleness-drift}, applied to both recursions, gives
\[
\bar W\le\beta_c^2\tau_w^2\bar H_c,
\qquad \bar\Theta\le\eta^2\tau_a^2\bar S.
\]
Conditional variance decomposition and a three-term norm bound give
\[
\bar H_c\le3L_c^2\bar\Delta+6L_c^2\bar W+
6L_\theta^2\bar\Theta+3E_c+\sigma_c^2.
\]
Since $6L_c^2\beta_c^2\tau_w^2\le1/2$, absorption implies
\[
\bar W\le\beta_c^2\tau_w^2
\left(6L_c^2\bar\Delta+12L_\theta^2\bar\Theta+6E_c+2\sigma_c^2\right).
\]
Substituting into \eqref{eq:async-critic-pre-absorption} leaves the coefficient
$\chi_c:=288\beta_c^2\tau_w^2L_c^4/m^2\le1/2$ multiplying $\bar\Delta$ on its right.
Absorbing again gives
\begin{align*}
\bar\Delta\le {}&\frac{4\mathbb E\Delta_0}{m\beta_cT}
+\left(\frac{48}{m^2}+\frac{576\beta_c^2\tau_w^2L_c^2}{m^2}\right)E_c\\
&+\left(\frac{6\beta_c}{m}+\frac{192\beta_c^2\tau_w^2L_c^2}{m^2}\right)\sigma_c^2\\
&+\left(\frac{96L_\theta^2}{m^2}
+\frac{1152\beta_c^2\tau_w^2L_c^2L_\theta^2}{m^2}\right)\bar\Theta
+\frac{12L_{w^\star}^2\eta^2}{m^2\beta_c^2}\bar S.
\end{align*}
Finally, $\beta_c\tau_wL_c^2/m\le1/24$ and $m\le L_c$ bound the respective coefficients by
$49/m^2$, $(6+8\tau_w)\beta_c/m$, and $98L_\theta^2/m^2$.
Using $\bar\Theta\le\eta^2\tau_a^2\bar S$ proves the result.
\end{proof}

\subsection{Closed-Loop Convergence of Asynchronous PPO-Clip (Parameter-Server Model)}
\label{appdx:async-closed-loop-proof}

We combine the finite-batch bias bound, delayed-gradient critic tracking, and actor stationarity.
Each server job applies one actor gradient and one critic gradient. The result does not cover
accumulated parameter changes from multiple worker updates.

\begin{proof}[Proof of Theorem~\ref{thm:async-closed-loop-explicit}]
At job $t$, write $s_t=t-\tau_t$, $r_t=t-\nu_t$,
$\bar\theta_t=\theta_{s_t}$, and $\bar w_t=w_{r_t}$.
The job evaluates the actor and critic at these stale parameters; its stored batch has the behavior policy
specified in the theorem. The $N$ batches each receive exactly $K$ jobs under the deterministic
batch-to-job mapping, so $T=NK$. The filtration contains the batches already drawn, current and stale
parameters, and chosen job indices, before the fresh minibatch draw.

\paragraph{Step 0: averaged quantities.}
Define
\begin{align}
\bar G&:=T^{-1}\sum_{t<T}\mathbb E\|\nabla J(\theta_t)\|^2,\label{eq:async-Gbar}\\
\bar\Delta&:=T^{-1}\sum_{t<T}\mathbb E\Delta_t,\quad
\Delta_t:=\|w_t-w^\star(\theta_t)\|^2,\label{eq:async-Deltabar}\\
\bar S&:=T^{-1}\sum_{t<T}\mathbb E\|\tilde g_t\|^2.\label{eq:async-Sbar}
\end{align}
Let $b_t:=\mathbb E[\tilde g_t\mid\mathcal F_t]-\nabla J(\theta_t)$ and
$e_t:=\mathbb E[\tilde g_t\mid\mathcal F_t]-\nabla J(\theta_{s_t})$. Then
\begin{equation}\label{eq:bt-decomp}
b_t=e_t+\nabla J(\theta_{s_t})-\nabla J(\theta_t),
\qquad\bar B:=T^{-1}\sum_{t<T}\mathbb E\|b_t\|^2.
\end{equation}

\paragraph{Step 1: open-loop actor stationarity.}
Part I of the proof of Theorem~\ref{thm:actor-open-loop-NK} uses only the current conditional mean bias,
smoothness, and conditional variance, so it applies here. For $\eta\le1/(4L_J)$,
\begin{equation}\label{eq:async-openloop-actor}
\bar G\le A_0+4\bar B,\qquad
A_0:=\frac{4(J_{\max}-J(\theta_0))}{\eta T}+2L_J\eta\sigma_a^2.
\end{equation}

\paragraph{Step 2: clipping and finite-batch reuse at the stale actor.}
Apply Lemma~\ref{lem:finite-batch-bridge} at $(\theta_{s_t},w_{r_t})$, relative to the job's behavior policy.
Its path-distribution and clipping bounds use the population KL premise between that behavior policy and
$\theta_{s_t}$; they do not require the current server policy to equal either one.
The empirical total-variation bound is uniform in the parameters, so it permits adaptive choices of stale snapshots.
Exactly $K$ jobs per batch ensure that the averaged squared empirical distribution error is bounded by
the same $\varepsilon_B$ as in the synchronous bridge.
Thus, with $a,F,E_c$ defined in the theorem,
\begin{equation}\label{eq:ebar}
\frac1T\sum_{t<T}\mathbb E\|e_t\|^2
\le a\bar R+F,\qquad
\bar R:=\frac1T\sum_{t<T}\mathbb E\|w_{r_t}-w^\star(\theta_{s_t})\|^2.
\end{equation}
The critic conditional-mean bias relative to
$\nabla_w\mathcal L_c(\theta_{s_t},w_{r_t})$ has average squared norm at most $E_c$.
The argument retains the reused batch's empirical error after conditioning on the full history.

\paragraph{Step 3: relate stale critic errors to current tracking errors.}
Put $c_a:=1+\mathbf1\{\tau_a>0\}$ and $c_w:=1+\mathbf1\{\tau_a+\tau_w>0\}$.
Lipschitzness of the selected critic minimizer gives
\[
\|w_{r_t}-w^\star(\theta_{s_t})\|^2
\le c_w\Delta_{r_t}+2L_{w^\star}^2\|\theta_{r_t}-\theta_{s_t}\|^2.
\]
For each realized delay sequence, any index $j$ can occur as $r_t$ at most $\tau_w+1$ times.
Both $r_t$ and $s_t$ lie in the preceding window of length $\tau_a+\tau_w$.
Lemma~\ref{lem:ps-staleness-drift} therefore gives
\begin{equation}\label{eq:async-adv}
\bar R\le c_w(\tau_w+1)\bar\Delta
+2L_{w^\star}^2\eta^2(\tau_a+\tau_w)^2\bar S.
\end{equation}
Likewise, \eqref{eq:bt-decomp}, $L_J$-smoothness, and the window bound imply
\[
\bar B\le c_a\frac1T\sum_{t<T}\mathbb E\|e_t\|^2
+2L_J^2\eta^2\tau_a^2\bar S.
\]
The indicator factors equal one when the corresponding delays vanish.

\paragraph{Step 4: the asynchronous actor bias inequality.}
Combining the preceding bounds gives
\begin{equation}\label{eq:Bbar-async}
\bar B\le A_{\mathrm{async}}\bar\Delta+q_{\mathrm{async}}\bar S+F_{\mathrm{async}},
\end{equation}
where
\begin{align*}
A_{\mathrm{async}}&=c_ac_w(\tau_w+1)a,\qquad F_{\mathrm{async}}=c_aF,\\
q_{\mathrm{async}}&=2c_aaL_{w^\star}^2\eta^2(\tau_a+\tau_w)^2
+2L_J^2\eta^2\tau_a^2.
\end{align*}
In particular, \eqref{eq:async-openloop-actor} yields
\begin{equation}\label{eq:actor-in-Delta-async}
\bar G\le A_0+4A_{\mathrm{async}}\bar\Delta+4q_{\mathrm{async}}\bar S+4F_{\mathrm{async}}.
\end{equation}

\paragraph{Step 5: open-loop critic tracking with delayed gradients.}
The empirical critic bridge supplies $E_c$ for the predictable bias $u_t$ in
Lemma~\ref{lem:async-critic-tracking}. Its centered minibatch noise has conditional mean zero and variance
at most $\sigma_c^2$ by the theorem's model. Applying that lemma gives
\begin{equation}\label{eq:critic-openloop-async2}
\bar\Delta\le U_{\mathrm{async}}+d_{\mathrm{async}}\bar S.
\end{equation}
The additional restriction on $\beta_c$ controls stale critic gradients even when the actor is fixed.
The predictable critic bias contributes the explicit $49E_c/m^2$ floor.

\paragraph{Step 6: close the actor update second moment.}
Conditional variance decomposition gives
\begin{equation}\label{eq:S-bound-async}
\bar S\le\sigma_a^2+2\bar G+2\bar B
\le\sigma_a^2+2A_0+10\bar B.
\end{equation}
Let $\chi:=A_{\mathrm{async}}d_{\mathrm{async}}+q_{\mathrm{async}}$.
Equations~\eqref{eq:Bbar-async} and \eqref{eq:critic-openloop-async2} imply
\[
\bar B\le A_{\mathrm{async}}U_{\mathrm{async}}+F_{\mathrm{async}}+\chi\bar S.
\]
Consequently, with $\rho_{\mathrm{async}}=10\chi<1$,
\begin{equation}\label{eq:S-closed-async}
\bar S\le
\frac{\sigma_a^2+2A_0+10A_{\mathrm{async}}U_{\mathrm{async}}+10F_{\mathrm{async}}}
{1-\rho_{\mathrm{async}}}.
\end{equation}

\paragraph{Step 7: conclude critic tracking.}
Substitute \eqref{eq:S-closed-async} into \eqref{eq:critic-openloop-async2}:
\[
\bar\Delta\le
\frac{(1-10q_{\mathrm{async}})U_{\mathrm{async}}
+d_{\mathrm{async}}(\sigma_a^2+2A_0+10F_{\mathrm{async}})}
{1-\rho_{\mathrm{async}}}.
\]
Dropping the nonpositive term $-10q_{\mathrm{async}}U_{\mathrm{async}}$ proves
\eqref{eq:async-critic-bound-explicit}.

\paragraph{Step 8: conclude actor stationarity.}
Use $\bar G\le A_0+4A_{\mathrm{async}}U_{\mathrm{async}}+4F_{\mathrm{async}}+4\chi\bar S$
and substitute \eqref{eq:S-closed-async}. After cancellation,
\[
\bar G\le
\frac{(1-2\chi)A_0+4A_{\mathrm{async}}U_{\mathrm{async}}
+4F_{\mathrm{async}}+4\chi\sigma_a^2}{1-\rho_{\mathrm{async}}}.
\]
Dropping $-2\chi A_0$ proves \eqref{eq:async-actor-bound-explicit}.
At zero delay, $c_a=c_w=1$, $q_{\mathrm{async}}=0$, and $A_{\mathrm{async}}=a$.
The bound then has the synchronous form, with the constants supplied by the delayed-critic lemma.
\end{proof}

\newpage
\section{Experimental Details}\label{appdx:exp}

\subsection{Purpose, environment, and update protocol}
Figures~\ref{fig:joint-errors}--\ref{fig:adaptive-batch} examine three aspects of the synchronous analysis: joint actor
stationarity and critic parameter tracking, population GAE bias cancellation, and empirical
gradient discrepancy under adaptive batch reuse. The two-step runs use the $H=2$, $\gamma=1$
example of Appendix~\ref{appdx:finite-batch-witness}, with actions $a_0,a_1\in\{-1,+1\}$,
terminal reward $R=(1+a_0a_1)/2$, and zero terminal value.
The shared policy has $q(\theta)=\tfrac12\tanh\theta$ and
$\pi_\theta(a\mid s)=(1+aq(\theta))/2$.
The normalized actor objective and the unnormalized critic target are, respectively,
\[
 J(\theta)=\frac{1+q(\theta)^2}{4},\quad J'(\theta)=\frac{q(\theta)q'(\theta)}2,
 \qquad w^\star(\theta)=\left(\frac{1+q(\theta)^2}{2},\frac{1+q(\theta)}2,\frac{1-q(\theta)}2\right).
\]
Every training run starts from $\theta_0=1$ and $w_0=(1/2,1/2,1/2)$, with critic
coordinates ordered as the initial state, positive-action state, and negative-action state.

At each epoch we collect $B$ independent complete trajectories under the frozen behavior
policy and retain their realized Monte Carlo returns. At each of exactly $K$ joint updates,
we draw $m$ trajectories independently with replacement from that batch. Actor and critic
share these draws and use the same pre-update parameters. Raw GAE is recomputed from the
current critic using \eqref{eq:witness-gae}; the actor uses the sign-dependent clipping
selector, choosing the unclipped branch at ties, and the critic uses
\eqref{eq:witness-mc-gradient}. We use plain gradient updates without advantage normalization,
projection, or early stopping. Here $K$ counts actual joint updates, not full passes through
the batch. The critic cube and the population KL bound follow from the example's analytic
invariance and stepsize certificates.

The four trajectory types permit exact multinomial compression: sample their counts from
$\operatorname{Multinomial}(B,P_{\bar\theta})$, then sample each minibatch's counts from
$\operatorname{Multinomial}(m,\widehat P)$. Because paths of a given type have identical
rewards and gradients, this is exactly the distribution of $B$ independent rollouts followed
by replacement sampling of trajectory indices. This representation reduces computational
cost while preserving the sample counts: $T=NK$, $Q=NB$, and $HQ=2Q$.

\subsection{Joint training diagnostics: Figure~\ref{fig:joint-errors}}
We run 32 independent seeds, 10000--10031, with
\[
 T=8192,\quad B=256,\quad K=8,\quad m=32,\quad
 \eta=0.01,\quad\beta_c=0.1,\quad\lambda=0.5,\quad\delta=0.2.
\]
These settings satisfy the instance-specific restrictions:
$\beta_c\le1/8$, $\eta\le1/4$,
$\rho=37.5(\eta/\beta_c)^2=0.375$, and
$\kappa=K^2\eta^2/3\approx0.0021333$.
Each seed represents $Q=262{,}144$ fresh trajectories and $HQ=524{,}288$ transitions.
For each realized run we evaluate every pre-update iterate exactly and form
\[
 \overline G_t=\frac1t\sum_{u=0}^{t-1}|J'(\theta_u)|^2,
 \qquad
 \overline\Delta_t=\frac1t\sum_{u=0}^{t-1}\|w_u-w^\star(\theta_u)\|^2.
\]
The plot averages these quantities over seeds; the critic metric is Euclidean parameter
error, not a visitation-weighted value error. Its units differ from actor stationarity.
All iterates enter the averages even though the curve is drawn at a subset of update indices.
At $T=8192$, the mean actor statistic is $2.8413\times10^{-4}$
(95\% interval $[2.8345,2.8482]\times10^{-4}$), and the mean critic statistic is
$1.9790\times10^{-3}$ (interval $[1.9499,2.0063]\times10^{-3}$).
The mean final pre-update episode return is $0.61881$.
The largest observed within-epoch population KL, including post-update parameters,
is $4.08\times10^{-7}$, below the deterministic certificate.
No active clipping occurs in these training runs; the next diagnostic deliberately
introduces drift to examine the clipped population gradient.
Here active clipping means $\widehat A>0,r>1+\delta$ or $\widehat A<0,r<1-\delta$.
Ratios outside the interval with the opposite advantage sign retain the unclipped gradient.

Substituting these training settings into the complete bound with
\eqref{eq:witness-actor-bias}--\eqref{eq:witness-critic-bias} gives actor and critic
upper bounds of $7.5065$ and $35.1910$. These exceed the elementary envelopes $1/432$
and $3$, respectively, and are therefore vacuous here. The measured reductions do not
establish tightness of the bound or verify the separate vanishing-stepsize rate.
The nonvacuous analytic calibration in Appendix~\ref{appdx:finite-batch-witness}
uses a much larger budget than the experiments.

\subsection{Population cancellation and residual errors: Figure~\ref{fig:gae-cancellation}}
To isolate the effect of GAE, we hold
$\theta=1$, $w=(1/2,1/2,1/2)$, and $\delta=0.1$, and vary
$\lambda\in\{0,0.1,\ldots,1\}$. In the on-policy condition, $\bar\theta=\theta$.
We enumerate the four paths to evaluate the population raw-GAE actor direction exactly.
Writing $e=w-w^\star(\theta)$, its signed bias is
\[
 b_{\mathrm{GAE}}=\frac{(1-\lambda)q'(\theta)}4(e_+-e_-).
\]
The solid population curve displays its absolute value, which equals the RMS error of this
deterministic quantity. At $\lambda=1$ it is zero, despite the inaccurate critic.
For each $\lambda$, we also draw 1,024 independent batches of $B=64$ trajectories,
using seeds 20000--21023, and compute the empirical conditional mean actor direction
on each batch. Its plotted RMS error is
$\sqrt{1024^{-1}\sum_i|\widehat P_i g-J'(\theta)|^2}$.
Common batch seeds across $\lambda$ give paired comparisons; batches within each setting
are independent. At $\lambda=1$, the measured RMS is $0.008439$
(95\% interval $[0.008106,0.008786]$), while exact enumeration of
$b_{\mathrm{GAE}}^2+\operatorname{Var}_{P_\theta}(g)/B$ gives an RMS of $0.008704$.
Thus cancellation of the population critic contribution does not remove sampling error.

The dashed curve instead uses $\bar\theta=0.5$ at the same current actor and critic,
and evaluates the clipped population gradient under $P_{\bar\theta}$ by exact enumeration.
Clipping is active at this parameter setting. Its absolute bias remains
$0.08035$ at $\lambda=1$; the cancellation identity does not remove trajectory-distribution
or clipping effects. This comparison isolates those residuals at fixed parameters;
it does not use an iterate from the training runs above.

\subsection{Adaptive batch discrepancies: Figure~\ref{fig:adaptive-batch}}
For each $B\in\{16,32,64,128,256,512,1024\}$, we draw 256 independent batches
using seeds 30000--30255 and run a single epoch from the common initialization.
The update parameters match the joint training configuration, with $K=8$ in the main figure.
Let $\phi_{\bar\theta,\theta,w}$ denote the per-trajectory clipped actor direction and
$\psi_w$ the Monte Carlo critic gradient. At every pre-update iterate, we compute
\[
 z_k^a=\big|(\widehat P-P_{\bar\theta})\phi_{\bar\theta,\theta_k,w_k}\big|^2,
 \qquad z_k^c=\big\|(\widehat P-P_{\bar\theta})\psi_{w_k}\big\|^2.
\]
The plotted values are the means, over independent batches and minibatch draws, of
$\max_{0\le k<K}z_k^a$ and $\max_{0\le k<K}z_k^c$.
Later iterates depend on the same empirical batch used in the discrepancy, so this
diagnostic includes adaptive selection. The maxima concern visited iterates, not
suprema over the entire actor and critic classes. The example's pathwise envelopes
hold uniformly over those classes and yield expectation bounds $4/B$ for the actor
and $1/B$ for the critic, drawn as dotted lines.

For $B=16$, the measured means are $1.7915\times10^{-4}$ (actor) and $0.030325$
(critic); for $B=1024$, they are $3.1169\times10^{-6}$ and $0.00045487$.
Both decrease over the measured batch-size range. Every observed update sequence satisfies its
batch-dependent pathwise envelope. Additional comparisons use $K=1$ and $K=32$.
For a fixed seed and $B$, these runs have common prefixes;
the maximum over more iterates is necessarily nondecreasing in $K$ and is not evidence
that additional reuse harms learning. Clipping is inactive in all these adaptive runs.
The $H=2$ study illustrates adaptive empirical error but does not distinguish complete-path
counting from the structured complexity bound at growing horizons.

\subsection{Statistical analysis and reproducibility}
For the two-step study, all shaded bands and error bars are pointwise 95\% percentile bootstrap intervals
with 2,000 resamples and bootstrap seed 20260925. We resample independent training
seeds or independent batch replicates, never correlated inner iterates. For the RMS
statistic, each resample averages squared errors before taking a square root.
Population curves and analytic envelopes are deterministic and have no sampling intervals.
All experimental settings, including the $K=8$ condition shown in the figure, were specified
in advance. The analysis includes all 32 training seeds, 1,024 replicates per $\lambda$,
and 256 replicates per $(B,K)$ setting, with no exclusions or replacements.

At each pre-update iterate, we evaluate the objectives, gradients, population KL, and
clipping frequencies exactly. To distinguish the sources of actor bias, we decompose
the empirical conditional-mean error as
\begin{align*}
 \widehat P\phi-J'(\theta)
 ={}&(\widehat P-P_{\bar\theta})\phi
 +P_{\bar\theta}(\phi-\phi^{\mathrm{unclip}})\\
 &+\big(P_{\bar\theta}\phi^{\mathrm{unclip}}-P_\theta\phi^{\mathrm{on}}\big)
 +\big(P_\theta\phi^{\mathrm{on}}-J'(\theta)\big),
\end{align*}
where $\phi^{\mathrm{unclip}}$ retains behavior-policy ratios and
$\phi^{\mathrm{on}}$ uses current-policy scores and the same raw GAE.
The critic decomposition similarly separates empirical and old-policy errors.
Both decompositions concern conditional means, so minibatch sampling noise is treated
separately from these bias terms.

Code is available in the anonymous repository linked in the
reproducibility statement. The two-step experiments ran on an ARM64 macOS CPU using Python 3.12.14,
NumPy 2.5.3, and double-precision arithmetic; figures used Matplotlib 3.11.2.
The experiments took 52.3 seconds, excluding figure generation.
These controlled measurements do not test neural or LLM performance, asynchronous updates,
a necessary threshold at $\rho=1$, the $T^{-2/5}$ rate, or the polynomial horizon-resource exponents.

\subsection{Critic--clipping interaction during training}
\label{appdx:interaction-study}

We examine whether critic error changes clipping decisions and distortion during
training, while actor updates move the critic's population target. This complements
the two-step study, whose training updates have inactive clipping.

\paragraph{Environment and protocol.}
We use the previous-action chain in Appendix~\ref{appdx:horizon-witness} with
$H=8$, $\gamma=1$, and zero terminal value. There is one initial state and two
states per subsequent layer, recording the preceding binary action, for 15
nonterminal states. Rewards are $r_0=0$ and
$r_h=\mathbf1\{a_h=a_{h-1}\}$ for $h\ge1$.
Unlike the shared scalar actor in that appendix, the experimental actor has an
independent parameter at each state:
\[
 \pi_\theta(+1\mid s)=0.1+0.8\operatorname{sigmoid}(\theta_s).
\]
All actor logits and learned tabular critic coordinates start at zero.
The actor maximizes $J(\theta)=\mathbb E[G_0]/H$; reported returns are
$\mathbb E[G_0]$, without horizon normalization.

We cross three critic modes---slow ($\beta_c=0.02$), fast ($\beta_c=0.2$), and
an exact-current-value oracle---with $K\in\{16,64\}$ and clipping enabled or
disabled. Each of the 12 configurations has 16 replicates of $T=2048$ actual
actor updates. The actor stepsize is $\eta=0.25$, the clipping half-width is
$\delta=0.05$, and $\lambda=0.5$. Each epoch collects $B=256$ independent
trajectories under its frozen initial policy. Each update samples $m=64$
trajectories with replacement from that batch, sharing actor and critic indices.
Raw GAE is recomputed with the current pre-update critic and detached in the
actor derivative. The critic loss averages squared residuals over trajectories
and time positions, using stored, unnormalized Monte Carlo returns.
Both gradients are evaluated before either plain ascent/descent update is applied.
There is no advantage normalization, entropy bonus, gradient or value clipping,
adaptive optimizer, KL rejection, early stopping, or extra critic update.
Clipping ties select the unclipped derivative.

The oracle uses $V^{\pi_{\theta_t}}$ to recompute the same raw GAE on the stored
trajectories; it neither refreshes the data nor substitutes true advantages.
All $2^8=256$ action paths are enumerated for exact population diagnostics.
Training still uses sampled finite batches and minibatches.

\paragraph{Clipping and interaction diagnostics.}
Let $c\in\{0,1\}$ indicate whether clipping is enabled, and let
$g_{\rm clip}(w)$ and $g_{\rm unclip}(w)$ denote population directions with the
same current actor, behavior trajectory law, ratios, and raw-GAE convention.
The applied population direction is
\[
 g_c(w)=g_{\rm unclip}(w)+c\,[g_{\rm clip}(w)-g_{\rm unclip}(w)].
\]
For a trajectory label $\widehat A_h(w)$ and action ratio $r_h$, define
\[
 \mathcal C_h(w)=
 \{\widehat A_h(w)>0,\ r_h>1+\delta\}
 \cup\{\widehat A_h(w)<0,\ r_h<1-\delta\}.
\]
Applied clipping activity averages $c\mathbf1_{\mathcal C_h(w)}$ over behavior
trajectories and time positions; sampled activity averages it over the actual
minibatch. With $w_t^\star=V^{\pi_{\theta_t}}$, the interaction vector is
\[
 I_t=c\Big([g_{\rm clip}(w_t)-g_{\rm unclip}(w_t)]
          -[g_{\rm clip}(w_t^\star)-g_{\rm unclip}(w_t^\star)]\Big).
\]
Mask disagreement averages
$c\,|\mathbf1_{\mathcal C_h(w_t)}-\mathbf1_{\mathcal C_h(w_t^\star)}|$.
These counterfactuals change only the critic at the same checkpoint.
Disabled-clipping zeros follow from $c=0$; they do not imply that the hypothetical
event $\mathcal C_h$ is absent. Oracle interaction is zero by construction.

\paragraph{Target movement and error measures.}
We measure the reverse direction of coupling by
\[
 \begin{aligned}
 M_t&=\|V^{\pi_{\theta_{t+1}}}-V^{\pi_{\theta_t}}\|_2^2,\\
 E_t&=\|w_{t+1}-V^{\pi_{\theta_{t+1}}}\|_2^2
       -\|w_{t+1}-V^{\pi_{\theta_t}}\|_2^2.
 \end{aligned}
\]
Here $E_t$ isolates target movement at a fixed updated critic; it is not the total
one-step tracking-error change. Negative values mean target movement reduces
that error. For the oracle, resetting the critic to the new exact value gives
$E_t=-M_t$. Critic MSE divides squared Euclidean tracking error by 15 and is
unweighted over states. Stationarity is $\|\nabla J(\theta_t)\|_2^2$.

\paragraph{Results.}
Table~\ref{tab:interaction-results} reports all configurations. At $K=64$ with
clipping, the mean interaction norms are $0.00204463$ for slow and $0.000421091$
for fast critics. Their paired difference is $0.00162354$, with 95\% bootstrap
interval $[0.00159658,0.00165281]$. Mask disagreement is respectively
$1.59182\%$ and $0.136202\%$. Both learned critics have positive mean interaction
in every replicate. At $K=16$, clipping is inactive and paired clipped/unclipped
traces coincide.

\begin{table}[t]
\centering
\small
\caption{Configuration means over 16 replicates. Critic MSE and interaction norm
are averaged over all pre-update checkpoints; final return is measured after
2048 updates. Disabled-clipping interaction is zero by definition.}
\label{tab:interaction-results}
\begin{tabular}{@{}lrlrrr@{}}
\toprule
Critic & $K$ & Clipping & Final return & Critic MSE & $\operatorname{mean}\|I_t\|$\\
\midrule
Slow   & 16 & Enabled  & 6.03452 & 0.935176 & 0\\
Slow   & 16 & Disabled & 6.03452 & 0.935176 & 0\\
Slow   & 64 & Enabled  & 6.01705 & 0.923023 & 0.00204463\\
Slow   & 64 & Disabled & 6.03721 & 0.972968 & 0\\
Fast   & 16 & Enabled  & 6.02924 & 0.059509 & 0\\
Fast   & 16 & Disabled & 6.02924 & 0.059509 & 0\\
Fast   & 64 & Enabled  & 5.99900 & 0.066098 & 0.00042109\\
Fast   & 64 & Disabled & 6.02917 & 0.067160 & 0\\
Oracle & 16 & Enabled  & 6.02953 & 0 & 0\\
Oracle & 16 & Disabled & 6.02953 & 0 & 0\\
Oracle & 64 & Enabled  & 5.99662 & 0 & 0\\
Oracle & 64 & Disabled & 6.02991 & 0 & 0\\
\bottomrule
\end{tabular}
\end{table}

The performance ordering is mixed. At $K=64$, the slow critic tracks less
accurately but clips less ($3.56532\%$ versus $5.07150\%$ population activity)
and attains a slightly higher final return than the fast critic: a paired gain
of $0.0180456$ $[0.0148847,0.0212563]$, about $0.3\%$ of the fast return.
Clipping reduces final return relative to its disabled control by $0.0201616$,
$0.0301689$, and $0.0332968$ for slow, fast, and oracle modes, respectively.
The oracle has $5.11585\%$ clipping activity despite zero interaction.
These observations do not give a universal performance ranking of critic speeds.

Mean target movement is $2.01810\times10^{-5}$, $1.93706\times10^{-5}$, and
$1.94415\times10^{-5}$ for the three clipped $K=64$ modes. The learned critics'
mean signed effects are positive ($0.0303849$ and $0.00532343$), but negative
effects occur at $0.787354\%$ and $12.4725\%$ of checkpoints. The oracle's
negative effect follows from its exact reset.

\paragraph{Statistics, resources, and scope.}
Time averages first use all 2048 pre-update checkpoints within each replicate.
Intervals use 2000 percentile bootstrap resamples of the 16 replicate summaries
(seed 20260926, NumPy PCG64). Contrasts pair replicate IDs within a fixed $K$;
correlated updates are not independent bootstrap observations. Rollout and
minibatch seeds are $50000+2j$ and $50001+2j$, $j=0,\ldots,15$, with common
random streams across matched configurations. Intervals are marginal, and curve
bands are pointwise. Curves use nonoverlapping 64-update bins.

Each replicate processes $Tm=131072$ trajectories per learned component.
The $K=16$ and $K=64$ settings collect respectively 32768 and 8192 fresh
trajectories, or 262144 and 65536 transitions. Thus cross-$K$ comparisons match
updates but not fresh-data budgets. The oracle is a privileged diagnostic control.
Execution used CPU float64 with one PyTorch thread on an AMD EPYC 9B14 under
Windows 11, Python 3.12.14, and PyTorch 2.10.0; the main run took 183.79 seconds.
The modified actor and initialization require their own theoretical constants;
Appendix~\ref{appdx:horizon-witness}'s certificate is not transferred to these
settings. The study supports interacting mechanisms, without testing the
convergence rate, coupling condition, bound tightness, horizon scaling, or
generalization to neural PPO.

\section{Discussion}
\label{appdx:further_discussion}

This appendix discusses the roles and limitations of the assumptions in the finite-time analysis
and the interpretation of the coupling index
$\rho$.

\subsection{Discussion of assumptions}
\label{subsec:discussion-assumptions}

The assumptions retain several features of PPO-Clip training---mini-batch reuse, clipping, KL control, and a learned
critic---while permitting finite-time bounds with explicit dependence on hyperparameters.
The geometric assumptions concern the visited parameter region and the connecting segments used in the proofs.
KL control limits policy-distribution mismatch; it does not itself bound parameter distance or keep the iterates
near initialization. The statistical bridge requires finite trajectory support or the structured
tabular conditions of Corollary~\ref{cor:structured-main}; positive action coverage and the conditional
sampling protocol are separate premises. These restrictions make the result a finite-sample specialization, not a guarantee for
arbitrary neural PPO implementations.

\paragraph{Smoothness of the actor objective.}
Assumption~\ref{ass:smooth-J0} imposes $L_J$-smoothness on the underlying reinforcement learning objective $J(\theta)$.
This is the canonical regularity condition used in non-asymptotic analyses of stochastic gradient methods, enabling a
one-step descent/ascent inequality and a telescoping argument.
Smoothness is assumed for the \emph{true objective} $J(\theta)$ rather than for the PPO-Clip surrogate,
which is non-smooth by construction due to clipping.
In our analysis, this non-smoothness is handled explicitly via event localization and clipping-activity bounds
(Appendix~\ref{appdx:useful-lemmas}), allowing the smoothness of $J$ to be used without approximating or
smoothing the surrogate.

\paragraph{Variance control of stochastic actor updates.}
Assumption~\ref{ass:actor-variance0} bounds the conditional variance of the stochastic actor update direction.
In the smoothness argument, this variance bound is combined with a small actor stepsize to absorb the conditional mean
part of the update; in the closed-loop argument, it gives the variance--mean decomposition used to control the
averaged update magnitude.
Such bounds are standard in stochastic optimization and RL theory and formalize the requirement that minibatch
gradient estimators do not exhibit heavy-tailed behavior on the visited region.
In the analyzed protocol, bounded raw advantages, bounded scores, and the separate ratio envelope
$r\le R_{\max}$ bound each actor contribution by $R_{\max}G_\pi A_{\max}$.
The critic contribution is bounded by $Q_c=2(V_{\max}+G_{\max})G_V$.
Thus $R_{\max}^2G_\pi^2A_{\max}^2$ and $Q_c^2$ suffice as conditional variance bounds
$\sigma_a^2$ and $\sigma_c^2$ for single-sample draws.
Uniform sampling with replacement then gives the conditional variance bounds in
Lemma~\ref{lem:finite-batch-bridge}. Average KL control alone does not imply a bounded likelihood ratio.
Bounded neural-network logits alone also do not bound the parameter derivatives of those logits; the score bound
remains an explicit premise.

\paragraph{Bounded advantages.}
Assumption~\ref{ass:adv0} bounds the raw GAE labels used by the theorem. A finite horizon and bounded rewards and
critic values provide such an envelope. Practical advantage normalization and clipping are different transformations;
their effects on the actor mean are not included in the present raw-GAE result.
Analytically, bounded advantages serve two essential roles:
(i) they ensure that the PPO surrogate gradient has controlled magnitude, and
(ii) they allow clipping-induced distortion to be quantified explicitly through the clip-activation event and its
probability.
Without bounded advantages, the clipped surrogate could still produce arbitrarily large gradients even when
likelihood ratios are well controlled.

\paragraph{KL trust-region control.}
Assumption~\ref{ass:kl0} is a population KL premise motivated by the empirical KL checks used in PPO practice.
An empirical check or penalty alone does not establish this premise. Appendix~\ref{appdx:stopping} gives a
sufficient population certificate and a predictable stopping rule under which the finite-time bound remains valid.
In our analysis, the KL budget plays a central dual role:
(i) it controls within-epoch drift induced by multi-epoch reuse (placing PPO in a ``weakly off-policy'' regime), and
(ii) it provides quantitative control over clipping activity by bounding the probability that likelihood ratios exit
the clipping interval.
These two controls determine how the KL budget enters the non-asymptotic PPO-Clip bounds.

\paragraph{Bounded return / objective.}
When a uniform bound $J(\theta)\le J_{\max}$ is invoked, it is not implied by smoothness.
Rather, it is a standard RL assumption that follows from bounded rewards under a finite horizon, a condition that is
typically enforced in RLHF through reward clipping or normalization.
The corresponding return assumption requires a uniform bound on the region of parameter space visited by the
algorithm.

\paragraph{Value realization and coverage.}
\label{par:realization-coverage-discussion}
For the Monte Carlo loss, the conditional-variance decomposition shows that a realizable minimizer agrees with the
true value on every state of positive visitation probability. The coverage premise in
Section~\ref{par:batch-protocol} makes these all relevant states;
Appendix~\ref{appdx:finite-batch-witness} verifies this identification in an explicit example.

\paragraph{Why local critic convexity and quadratic growth are used.}
Our analysis relies on a local convexity and quadratic-growth condition on the critic loss
(Assumption~\ref{ass:critic0}).
This condition is used to obtain the parameter-space tracking contraction needed by the actor--critic feedback
argument.
In PPO-Clip, the critic is not optimized against a fixed objective: as the policy parameters evolve, the
population-optimal critic $w^\star(\theta)$ becomes a \emph{moving target}.
When $a>0$, controlling the tracking error limits its contribution to the
population actor-bias bound.

The quadratic-growth condition converts critic excess loss into distance to the selected critic optimum, while local
convexity implies the one-point inequality
\[
f_\theta(w)-f_\theta(w^\star(\theta))
\le
\langle \nabla_w f_\theta(w),\,w-w^\star(\theta)\rangle,
\qquad f_\theta(w):=\mathcal L_c(\theta,w).
\]
Together with smoothness and the stated stepsize restriction, these properties imply that one critic gradient step
contracts the expected tracking error up to conditional noise, the empirical-to-population critic bias, and
actor-induced target drift.
This is the precise stability property needed to close the finite-time actor--critic inequalities.

Without such a local contraction condition, a smooth nonconvex critic objective may not correct critic error quickly
enough relative to the movement of $w^\star(\theta)$ induced by policy updates under constant stepsizes.
Critic lag can then accumulate and propagate into the actor update through advantage estimation.
A different analysis would require additional structure or a different error measure; the present parameter-distance
proof does not establish tracking without its contraction premises.

\paragraph{Relaxing the critic-geometry assumption via loss-gap tracking.}
The local convexity and quadratic-growth assumption is not the only possible route to a closed-loop analysis.
One can relax the parameter-space contraction requirement by redefining the critic tracking error in terms of the
critic optimization loss rather than distance to the selected minimizer. Concretely, for
$f_\theta(w):=\mathcal L_c(\theta,w)$, define
\[
\Delta_t^{\mathrm{loss}}
:=
f_{\theta_t}(w_t)-f_{\theta_t}(w^\star(\theta_t)).
\]
For a fixed population objective, a Polyak--\L{}ojasiewicz-type condition and smoothness give a function-value
contraction when the critic step is decomposed into its population gradient, predictable bias, and centered noise.
Reused data retain a bias floor as well as the noise floor. A changing objective requires a further drift bound;
PL alone does not supply it.
Rather than a direct contraction of
$\|w_t-w^\star(\theta_t)\|^2$, this route seeks a recursion for
$\Delta_t^{\mathrm{loss}}$.

To close the actor--critic loop under this alternative metric, one then needs a calibration step connecting critic
loss gap to value-estimation error. For example, one may assume a local value-error bound of the form
\[
\|V_{w_t}-V^{\pi_{\theta_t}}\|_\infty^2
\le
C_{\mathrm{val}}\,
\Delta_t^{\mathrm{loss}},
\]
or an analogous $L_2(d^{\pi_{\theta_t}})$ bound together with a concentrability condition.
Such a condition replaces the parameter-space Lipschitz bound
$\|V_w-V_{w^\star(\theta)}\|_\infty\le L_V\|w-w^\star(\theta)\|$.
The resulting expected actor-bias bound replaces $a\mathbb E\Delta_t$ by
$4G_\pi^2\widetilde C_{\lambda,H}^2C_{\mathrm{val}}\mathbb E\Delta_t^{\mathrm{loss}}$, while retaining the finite-batch,
trajectory-mismatch, and clipping terms in $F$.

The remaining requirement is control of the moving-objective drift. In parameter-distance tracking, policy drift
enters through
\[
\|w^\star(\theta_{t+1})-w^\star(\theta_t)\|
\le
L_{w^\star}\|\theta_{t+1}-\theta_t\|.
\]
For loss-gap tracking, one instead has to compare
$f_{\theta_{t+1}}(w_{t+1})-f_{\theta_{t+1}}(w^\star(\theta_{t+1}))$ with
$f_{\theta_t}(w_{t+1})-f_{\theta_t}(w^\star(\theta_t))$.
This requires additional Lipschitz or smoothness assumptions on the dependence of the critic loss and its optimal
value on $\theta$, and it produces several extra drift terms proportional to policy movement, such as
$\eta\|g_t\|$ or $\eta^2\|g_t\|^2$ after standard Young-type inequalities.
Bounds on these drift terms would permit a corresponding closed-loop argument with a coupling condition adapted
to the loss-gap metric. The calibration and drift estimates are substantive additional assumptions, not consequences
of PL or of the main theorem. Appendix~\ref{appdx:critic-extensions} develops this conditional route.
The main result uses parameter-distance tracking and its explicit contraction premises.

\paragraph{Interpreting the condition in LLM PPO.}
Local convexity and quadratic growth are substantive restrictions on a neural critic. The following considerations
motivate examining local geometry, but do not establish these assumptions for an LLM:

\begin{itemize}
    \item \textbf{Policy-distribution locality.}
    The population KL budget controls mismatch from each batch's behavior policy. It does not prevent cumulative
    parameter drift across batches or imply critic convexity. The required critic region and connecting segments
    must be checked separately.

    \item \textbf{Post-training initialization.}
    PPO is typically applied during post-training, starting from a supervised fine-tuning (SFT) checkpoint.
    This motivates studying the region reached from that checkpoint, but an SFT initialization does not guarantee
    high reward, local convexity, or quadratic growth of the subsequent critic loss.

    \item \textbf{Over-parameterization and regression structure.}
    Modern LLM critics operate in a highly over-parameterized regime and are trained via squared-error regression.
    Linearized regression models can motivate restricted geometric assumptions, but width and squared-error loss
    alone do not imply the selected-minimizer quadratic-growth condition. Parameter redundancies and insufficient
    state coverage can violate it.
\end{itemize}

\paragraph{Finite-batch scope and structured specialization.}
\label{par:batch-scope-discussion}
In the autoregressive setting, a finite prompt distribution, finite actions and horizon, and a deterministic reward
model give finite trajectory support. The baseline cardinality-dependent bound can nevertheless scale exponentially
with $H$. The structured specialization in Corollary~\ref{cor:structured-main} instead controls the clipped-gradient
class through the numbers of critic coordinates and time--state--action cells. Its envelope is required on the
entire critic cube and its clipping-tie rule selects the unclipped derivative.
Appendices~\ref{appdx:finite-batch-witness} and~\ref{appdx:horizon-witness} verify the premises in explicit examples,
including a growing-horizon family with polynomial constants. This remains a structured guarantee; validity for a
particular neural PPO implementation requires separate verification.

\subsection{Stability lens: the coupling index \texorpdfstring{$\rho$}{rho}}
\label{subsec:discussion-rho}

When $aL_{w^\star}^2>0$, the sufficient synchronous condition $\rho<1$ limits the actor-to-critic stepsize ratio, with a threshold determined
by the problem constants. When $a=0$, including $\lambda=1$ or $H=1$ with matched terminal values, $\rho=0$ imposes no
time-scale restriction; critic tracking still depends on $U$ and $d$.

\paragraph{Where $\rho$ comes from.}
At a high level, the open-loop actor bound takes the form
\[
\bar G \le A_0+4F+4a\bar\Delta,
\]
while the open-loop critic tracking bound takes the form
\[
\bar\Delta\le U+d\bar S,\qquad
d=\frac{40L_{w^\star}^2\eta^2}{\mu_q^2\beta_c^2}.
\]
Closing the loop requires controlling $\bar S$ (the average squared actor update magnitude) in terms of
$\bar G$ and the actor bias, and the actor bias itself depends on $\bar\Delta$.
Here $U$ includes the critic finite-batch and trajectory-mismatch floor $40E_c/\mu_q^2$.
After substitution, $\bar\Delta$ appears on both sides with coefficient $\rho=10ad$.
Thus $\rho<1$ ensures the coupled inequalities can be solved with a finite amplification factor.

\paragraph{Interpretation.}
When $\rho\ll1$, feedback produces a small amplification factor in these inequalities; the numerical guarantee
also depends on the sizes of the other constants and error floors.
As $\rho\to1^{-}$, the factor $(1-\rho)^{-1}$ becomes large. This deterioration of a sufficient upper bound does
not identify an actual stability transition in the algorithm.
The divergence of the bound as $\rho\to 1^{-}$ should be interpreted as a loss of \emph{uniform} finite-time control,
rather than immediate divergence of the iterates.
When $\rho\ge1$, the scalar closure does not provide the stated guarantee. The theorem makes no conclusion about
convergence or divergence in that regime, or about which implementation changes would restore stability.

\paragraph{Practical implication.}
For fixed regularity constants, reducing $\eta/\beta_c$ reduces the synchronous coupling coefficient.
Increasing $\beta_c$ is subject to its contraction restriction and raises the critic noise floor; in the
asynchronous theorem its permitted size also decreases with critic delay. Tightening $\kappa$ reduces the explicit
population mismatch floors, while increasing $B$ reduces the finite-batch floor. These are implications for the
proved bound, not necessary tuning rules or guarantees for alternative update schedules.

\section{Extensions and Alternative Critic Error Measures}
\label{appdx:critic-extensions}

This appendix examines the critic assumptions underlying the parameter-distance tracking error
\[
\Delta_t := \|w_t-w^\star(\theta_t)\|^2
\]
and considers two extensions: tracking the critic loss gap under additional calibration and drift assumptions,
and including critic approximation error in the closed-loop bounds.

\subsection{Why the strongest assumptions concern the critic landscape}
\label{subsec:critic-landscape-strength}

Among the assumptions used in the paper, the strongest ones are the local regularity assumptions imposed on the
critic landscape.  In particular, Assumption~\ref{ass:critic0} requires that, on the visited critic region and for each
fixed policy parameter $\theta$, the critic loss
\[
f_\theta(w):=\mathcal L_c(\theta,w)
\]
is locally convex, has Lipschitz-continuous gradients in $w$, satisfies quadratic growth around a selected minimizer
$w^\star(\theta)$, and has a selected minimizer map $w^\star(\theta)$ that is Lipschitz in $\theta$.  These conditions
are stronger than the actor smoothness, bounded-variance, bounded-advantage, and KL-control assumptions, because they
impose geometric structure on a potentially nonconvex learned value-function objective.

This critic geometry yields a moving-target contraction for the parameter-distance
tracking error $\Delta_t=\|w_t-w^\star(\theta_t)\|^2$:
\[
\begin{aligned}
\mathbb E\Delta_{t+1}
\le{}&(1-\mu_q\beta_c/4)\mathbb E\Delta_t
+\tfrac54\beta_c^2\sigma_c^2
+\frac{10\beta_c}{\mu_q}\mathbb E\|e_t^c\|^2\\
&+\frac{10L_{w^\star}^2\eta^2}{\mu_q\beta_c}\mathbb E\|g_t\|^2,
\end{aligned}
\]
where $e_t^c=\mathbb E[h_t\mid\mathcal F_t]-\nabla_w\mathcal L_c(\theta_t,w_t)$ is the predictable critic bias.
Averaging this recursion and applying the finite-batch bridge gives
\[
\bar\Delta\le
\frac{4\mathbb E\Delta_0}{\mu_q\beta_c T}
+\frac{5\beta_c\sigma_c^2}{\mu_q}
+\frac{40E_c}{\mu_q^2}
+\frac{40L_{w^\star}^2\eta^2}{\mu_q^2\beta_c^2}\bar S
=U+d\bar S.
\]
When combined with the actor bias bound
\[
\mathbb E\|b_t\|^2\le a\mathbb E\Delta_t+F,
\]
this produces the closed-loop amplification factor
\[
\rho=10ad,\qquad
a=4G_\pi^2L_V^2\widetilde C_{\lambda,H}^2,\qquad
d=\frac{40L_{w^\star}^2\eta^2}{\mu_q^2\beta_c^2}.
\]
These parameter-space assumptions provide a sufficient route to controlling actor--critic coupling, the GAE/horizon factor
$\widetilde C_{\lambda,H}$, KL-controlled off-policy drift, and clipping distortion.

An alternative is critic-loss tracking under a Polyak--\L{}ojasiewicz (PL) condition.
PL and smoothness control a frozen objective; extending this control to the coupled process also requires
a moving-objective drift bound and calibration from loss gap to value error.
These are additional premises beyond the main assumptions.

\subsection{Tracking the critic loss gap instead of parameter distance}
\label{subsec:loss-gap-tracking}

Define the critic optimal value and loss-gap tracking error by
\[
f_\theta^\star := \inf_w f_\theta(w),
\qquad
\Delta_t^{\mathrm{loss}}
:=
f_{\theta_t}(w_t)-f_{\theta_t}^\star .
\]
This definition avoids requiring a locally unique or Lipschitz-selected minimizer in parameter space.  Instead, the
critic part of the analysis can be based on local smoothness and a PL inequality:
\[
\frac12\|\nabla_w f_\theta(w)\|^2
\ge
\mu_{\mathrm{PL}}\big(f_\theta(w)-f_\theta^\star\big)
\]
on the visited region. Let the critic update have conditional mean
$\nabla_w f_{\theta_t}(w_t)+e_t^c$ and centered conditional variance at most $\sigma_c^2$.
For $\beta_c\le\min\{1/(4L_c),1/\mu_{\mathrm{PL}}\}$, smoothness and Young's inequality give the
frozen-objective bound
\[
\mathbb E\!\left[
f_{\theta_t}(w_{t+1})-f_{\theta_t}^\star
\mid \mathcal F_t
\right]
\le
(1-\mu_{\mathrm{PL}}\beta_c)\Delta_t^{\mathrm{loss}}
+\tfrac54\beta_c\|e_t^c\|^2
+\tfrac12L_c\beta_c^2\sigma_c^2 ,
\]
provided smoothness holds along the critic step. Indeed the descent calculation leaves at most
$-\tfrac12\beta_c\|\nabla f_{\theta_t}(w_t)\|^2+\tfrac54\beta_c\|e_t^c\|^2$
before applying PL. The empirical-to-population bias does not vanish merely because the critic loss is PL.

The main difference is the moving-objective drift. In parameter-distance tracking, actor movement enters through
\[
\|w^\star(\theta_{t+1})-w^\star(\theta_t)\|
\le
L_{w^\star}\|\theta_{t+1}-\theta_t\|
=
L_{w^\star}\eta\|g_t\|.
\]
With loss-gap tracking, there is no need to compare minimizers directly.  Instead one must compare two different
critic objectives:
\[
\Delta_{t+1}^{\mathrm{loss}}
=
f_{\theta_{t+1}}(w_{t+1})-f_{\theta_{t+1}}^\star .
\]
Insert and subtract the old objective:
\[
\begin{aligned}
\Delta_{t+1}^{\mathrm{loss}}
&=
\big(f_{\theta_t}(w_{t+1})-f_{\theta_t}^\star\big)\\
&\quad+
\big(f_{\theta_{t+1}}(w_{t+1})-f_{\theta_t}(w_{t+1})\big)
+
\big(f_{\theta_t}^\star-f_{\theta_{t+1}}^\star\big).
\end{aligned}
\]
The conditional expectation of the first term satisfies the biased PL bound above. The second and third terms are drift terms caused by policy
movement.  To control them, one needs regularity of the critic objective with respect to $\theta$, for example local
Lipschitz or smooth dependence of $f_\theta(w)$ and $f_\theta^\star$ on $\theta$ along the visited region.  A typical
bound would have the form
\[
\begin{aligned}
\mathbb E[\Delta_{t+1}^{\mathrm{loss}}\mid\mathcal F_t]
\le{}&(1-\mu_{\mathrm{PL}}\beta_c)\Delta_t^{\mathrm{loss}}
+\tfrac54\beta_c\|e_t^c\|^2\\
&+\tfrac12L_c\beta_c^2\sigma_c^2+D_t^{\mathrm{move}},
\end{aligned}
\]
where $D_t^{\mathrm{move}}$ is a history-measurable upper bound on the conditional expectation of the sum of
the two objective-change terms. A uniform Lipschitz bound gives a term proportional to
$\eta\mathbb E[\|g_t\|\mid\mathcal F_t]$. A quadratic bound requires a separate estimate; it does not follow from
PL alone. If such linear or quadratic movement control is established, averaging produces respectively
a linear movement term of order
\[
\frac{\eta}{\mu_{\mathrm{PL}}\beta_c}
\frac1T\sum_{t=0}^{T-1}\mathbb E\|g_t\|
\]
or, in the quadratic-drift case, a term comparable to
\[
\frac{\eta^2}{\mu_{\mathrm{PL}}\beta_c}
\frac1T\sum_{t=0}^{T-1}\mathbb E\|g_t\|^2 .
\]
The predictable critic bias additionally contributes at most $5E_c/(4\mu_{\mathrm{PL}})$ after averaging.
Young's inequality can convert a linear movement term into squared-update terms plus a residual, but that residual
and the resulting coupling condition must be retained.

On the actor side, parameter-distance tracking uses
\[
\|V_w-V_{w^\star(\theta)}\|_\infty
\le
L_V\|w-w^\star(\theta)\|
\]
to turn parameter tracking into an advantage-estimation error.  If the critic is tracked by loss gap, this step must
be replaced by a calibration assumption such as
\[
\|V_{w_t}-V^{\pi_{\theta_t}}\|_\infty^2
\le
C_{\mathrm{val}}\Delta_t^{\mathrm{loss}}
\]
or an $L_2(d^{\pi_{\theta_t}})$ version plus a concentrability condition.  The actor bias bound then becomes
\[
\mathbb E\|b_t\|^2
\le a_{\mathrm{loss}}\mathbb E\Delta_t^{\mathrm{loss}}+F,
\qquad a_{\mathrm{loss}}:=4G_\pi^2\widetilde C_{\lambda,H}^2C_{\mathrm{val}}.
\]
If a separately proved drift estimate gives
$\bar\Delta^{\mathrm{loss}}\le U_{\mathrm{loss}}+d_{\mathrm{loss}}\bar S$,
the same scalar elimination as in the main proof uses
\[
\rho_{\mathrm{loss}}=10a_{\mathrm{loss}}d_{\mathrm{loss}}<1.
\]
For example, the additional premise
$D_t^{\mathrm{move}}\le C_{\mathrm{move}}\eta^2\mathbb E[\|g_t\|^2\mid\mathcal F_t]$
would give $d_{\mathrm{loss}}=C_{\mathrm{move}}\eta^2/(\mu_{\mathrm{PL}}\beta_c)$.
This quadratic drift estimate remains an additional premise. In particular, the $\beta_c^{-2}$ dependence of the
parameter-distance bound does not follow from PL alone. Linear movement control instead introduces the additional
residuals discussed above.

Critic-loss tracking therefore offers a possible relaxation of local parameter contraction, conditional on the
additional calibration and drift estimates. A closed-loop bound based on this metric would depend on calibration
constants, objective-drift constants, and possibly additional movement floors. The parameter-distance result follows
from the explicit contraction premises of the main theorem; the loss-gap route outlined here requires these further
estimates.

\subsection{Including critic approximation error}
\label{subsec:critic-approximation-error}

The main theorem assumes a realizable critic: the selected population critic satisfies
\[
V_{w^\star(\theta)} = V^{\pi_\theta}
\]
on the visited region. This removes irreducible value-approximation error, leaving tracking lag and the finite-batch,
clipping, and trajectory-mismatch floors. The following extension allows an imperfect critic class while retaining
the other geometry, coverage, sampling, terminal-boundary, and boundedness premises.

Let
\[
\varepsilon_{\mathrm{app}}^2(\theta)
:=
\|V_{w^\star(\theta)}-V^{\pi_\theta}\|_\infty^2,
\qquad
\bar\varepsilon_{\mathrm{app}}^2
:=
\sup_{\theta\in\mathcal R_\theta}\varepsilon_{\mathrm{app}}^2(\theta).
\]
Then the value-estimation error decomposes as
\[
\begin{aligned}
\|V_{w_t}-V^{\pi_{\theta_t}}\|_\infty^2
&\le
2\|V_{w_t}-V_{w^\star(\theta_t)}\|_\infty^2
+2\|V_{w^\star(\theta_t)}-V^{\pi_{\theta_t}}\|_\infty^2\\
&\le
2L_V^2\Delta_t+2\bar\varepsilon_{\mathrm{app}}^2 .
\end{aligned}
\]
Substituting this into the GAE term of the four-term actor-bias decomposition gives
\[
\mathbb E\|b_t\|^2\le a_{\mathrm{app}}\mathbb E\Delta_t+F_{\mathrm{app}},
\qquad a_{\mathrm{app}}:=2a,
\]
\[
F_{\mathrm{app}}:=F+8G_\pi^2\widetilde C_{\lambda,H}^2\bar\varepsilon_{\mathrm{app}}^2.
\]
The factor eight combines the factor two in the value-error decomposition and the factor four in the actor-bias
decomposition. The critic recursion still tracks the selected population minimizer, so $U$, $d$, and $E_c$
are unchanged. The increased tracking coefficient requires the updated condition
$\rho_{\mathrm{app}}:=10a_{\mathrm{app}}d<1$. Scalar elimination then gives
\[
\begin{aligned}
\bar G&\le\frac{A_0+4F_{\mathrm{app}}+4a_{\mathrm{app}}U
+4a_{\mathrm{app}}d\sigma_a^2}{1-\rho_{\mathrm{app}}},\\
\bar\Delta&\le\frac{U+d(\sigma_a^2+2A_0+10F_{\mathrm{app}})}{1-\rho_{\mathrm{app}}}.
\end{aligned}
\]
Thus the approximation floor also enters critic tracking through the actor update magnitude.
This extension retains the learned-critic dynamics while adding irreducible value-approximation error to the
closed-loop bounds.

An open-loop alternative assumes directly that
the critic used by the actor satisfies
\[
\|V_{w_t}-V^{\pi_{\theta_t}}\|_\infty^2
\le
\varepsilon_{\mathrm{crit}}^2
\qquad \text{for all } t .
\]
Then the finite-batch actor proof gives
$\bar G\le A_0+4F+16G_\pi^2\widetilde C_{\lambda,H}^2\varepsilon_{\mathrm{crit}}^2$,
and requires no critic tracking theorem. This formulation treats critic quality as an exogenous oracle property.
The closed-loop analysis instead quantifies its dependence on critic optimization, actor drift, and the
actor--critic timescale ratio.

\end{document}